\documentclass[14pt]{extarticle} 
\usepackage{arxiv,times}
\usepackage{wrapfig}

\usepackage{amsmath,amsfonts,bm}

\def\eqref#1{equation~\ref{#1}}

\def\1{\bm{1}}

\DeclareMathAlphabet{\mathsfit}{\encodingdefault}{\sfdefault}{m}{sl}
\SetMathAlphabet{\mathsfit}{bold}{\encodingdefault}{\sfdefault}{bx}{n}

\DeclareMathOperator*{\argmin}{arg\,min}

\renewcommand{\eqref}[1]{(\ref{#1})}
\usepackage{hyperref}
\usepackage{url}
\usepackage{natbib}
\usepackage{graphicx}

\usepackage[utf8]{inputenc} 
\usepackage[T1]{fontenc}    
\usepackage{hyperref}       
\usepackage{url}            
\usepackage{booktabs}       
\usepackage{amsfonts}       
\usepackage{nicefrac}       
\usepackage{microtype}      
\usepackage{xcolor}         
\usepackage[algo2e]{algorithm2e}
\usepackage{float}
\usepackage{multirow}
\usepackage{nicematrix}

\usepackage{graphicx} 
\usepackage{wrapfig}
\usepackage{rotating}
\usepackage{lscape}
\usepackage{amsmath}
\usepackage[english]{babel}
\usepackage{amsthm}
\usepackage{amssymb}
\usepackage{algorithm}
\usepackage{algorithmic}
\usepackage[labelfont=bf, font=footnotesize]{caption}

\newtheorem{theorem}{Theorem}[section]
\newtheorem{proposition}[theorem]{Proposition}
\newtheorem{definition}[theorem]{Definition}
\newtheorem{corollary}[theorem]{Corollary}

\newtheorem{remark}[theorem]{Remark}

\newtheorem{lemma}[theorem]{Lemma}

\def\thickhline{\noalign{\hrule height.8pt}}

\SetKwInput{KwInput}{Input}                
\SetKwInput{KwOutput}{Output}

\definecolor{akcolor}{rgb}{0.65, 0.15, 0.6}

\definecolor{darkblue}{rgb}{0.1,0.1,0.6}

\title{Linear Spherical Sliced Optimal Transport: A Fast Metric for Comparing Spherical Data}

\author{%
  Xinran Liu$^{1}$, Yikun Bai$^{1}$, Rocío Díaz Martín$^{2}$, Kaiwen Shi$^{1}$, Ashkan Shahbazi$^{1}$, \\ \textbf{Bennett A. Landman}$^{1,3}$, \textbf{Catie Chang}$^{1,3}$, \textbf{Soheil Kolouri}$^{1}$  \\ 
  \textbf{for the Alzheimer's Disease Neuroimaging Initiative}\thanks{
Data used in preparation of this article were partially obtained from the Alzheimer’s Disease Neuroimaging Initiative (ADNI) database (\url{adni.loni.usc.edu}). As such, the investigators within the ADNI contributed to the design and implementation of ADNI and/or provided data but did not participate in analysis or writing of this paper. A complete listing of ADNI investigators can be found at: \url{https://adni.loni.usc.edu/wp-content/uploads/how_to_apply/ADNI_Acknowledgement_List.pdf}.}\\
  \\
  $^{1}$Department of Computer Science, Vanderbilt University, Nashville, TN, 37240\\
  $^{2}$Department of Mathematics, Tufts University, Medford, MA 02155\\
  $^{3}$Department of Electrical and Computer Engineering, Vanderbilt University, Nashville, TN, 37240
}

\begin{document}

\date{}
\maketitle
\begin{abstract}
Efficient comparison of spherical probability distributions becomes important in fields such as computer vision, geosciences, and medicine. Sliced optimal transport distances, such as spherical and stereographic spherical sliced Wasserstein distances, have recently been developed to address this need. These methods reduce the computational burden of optimal transport by slicing hyperspheres into one-dimensional projections, i.e., lines or circles. Concurrently, linear optimal transport has been proposed to embed distributions into \( L^2 \) spaces, where the \( L^2 \) distance approximates the optimal transport distance, thereby simplifying comparisons across multiple distributions. In this work, we introduce the Linear Spherical Sliced Optimal Transport (LSSOT) framework, which utilizes slicing to embed spherical distributions into \( L^2 \) spaces while preserving their intrinsic geometry, offering a computationally efficient metric for spherical probability measures. We establish the metricity of LSSOT and demonstrate its superior computational efficiency in applications such as cortical surface registration, 3D point cloud interpolation via gradient flow, and shape embedding. Our results demonstrate the significant computational benefits and high accuracy of LSSOT in these applications.
\end{abstract}

\section{Introduction}

Spherical signal analysis has received increasing attention in many domains such as computer vision, astrophysics \citep{jarosik2011seven}, climatology \citep{garrett2024validating} and neuroscience \citep{Elad2002, Zhao2019SphericalUO, ZHAO202346}. In practice, certain types of data are intrinsically defined on a sphere, namely omnidirectional images, ambisonics, and earth data. In addition, some forms of data are studied canonically in hyperspheres by spherical mappings \citep{gotsman2003fundamentals, gu2004genus, shen2006spherical, chung2022embedding} from the original data spaces. Among them, a notable category of data is the genus-zero surfaces, which can be conformally mapped to unit spheres according to the Poincare uniformization theorem \citep{Abikoff1981TheUT}. Cortical surface-based analysis, for example, can be facilitated by the inherent spherical topology (0-genus) of the cerebral cortex. Compared to conventional 3D volumetric analysis, applying spherical representations of cortical surfaces offers a more streamlined and accurate geometric framework to examine the brain. This approach proves particularly advantageous in both cross-sectional and longitudinal studies of brain structure and function \citep{9389746,ZHAO202346}.

Analyzing spherical data through machine learning techniques necessitates metrics that can effectively compare signals while maintaining data structure and geometry. In this context, optimal transport (OT) has emerged as a versatile and widely adopted tool in the machine learning community. OT provides a framework for comparing probability distributions in various domains \citep{kolouri2017optimal, peyre2019computational, kolkin2019style, arjovsky2017wasserstein, seguy2017large, makkuva2020optimal, rostami2019deep, cang2023screening, lu2023characterizing}. OT has demonstrated remarkable effectiveness not only in Euclidean spaces, but its theoretical foundations have also been firmly established for various other manifolds \citep{fathi2010optimal, mccann2010ricci}, especially on spheres \citep{cui2019spherical,hamfeldt2022convergence}. However, the calculation of OT distances is prohibitively expensive with a cubic cost $\mathcal{O}(N^3)$. This computational burden poses a substantial challenge to applications where the input size $N$ is large. In light of this, researchers have enthusiastically pursued and developed faster alternatives, including entropic optimal transport \citep{cuturi2013sinkhorn}  and sliced optimal transport \citep{bonnotte2013unidimensional,bonneel2015sliced,kolouri2016sliced,kolouri2018sliced,kolouri2018sliced2,deshpande2018generative,kolouri2019generalized,nadjahi2020statistical,bai2023sliced,liu2024expected}.


Sliced Optimal Transport (SOT) accelerates computation by projecting high-dimensional data onto one-dimensional subspaces and applying efficient OT solvers for the resulting one-dimensional probability distributions. In this paper, we are interested in sliced optimal transport methods on hyperspheres. Related work includes the Spherical Sliced Wasserstein (SSW) discrepancy \citep{bonet2022spherical}, which proposes slicing the sphere by the spherical Radon transform and leveraging circular optimal transport; and the Stereographic Spherical Sliced Wasserstein (S3W) distance \citep{tran2024stereographic}, which utilizes stereographic projection to solve the SOT problem in Euclidean space with minimal distortion.


Many machine learning applications, from manifold learning to k-nearest neighbor classification and regression, rely on pairwise distance calculations among members of a set, often requiring distance evaluations on the order of the square of the number of objects.
To accelerate pairwise calculation of OT distances between probability measures, Linear Optimal Transport (LOT) \citep{wang2013linear,kolouri2016continuous,park2018cumulative,bai2023linear} was developed as a framework that embeds distributions into $L^2$ spaces where pairwise distance computations can be performed with significantly higher efficiency. For example, Wasserstein Embedding and Sliced Wasserstein Embedding have been applied in Euclidean spaces to enable efficient image analysis/classification \cite{basu2014detecting,ozolek2014accurate,kolouri2015radon,shifat2021radon}, faster graph and set learning \citep{kolouri2020wasserstein,mialon2021trainable,naderializadeh2021pooling}, supervised task transferability prediction \citep{liu2022wasserstein}, point cloud retrieval \citep{lu2024slosh}, etc. A more recent work, \citet{martin2023lcot} focused on $\mathbb{S}^1$ and established a computationally efficient Linear Circular Optimal Transport (LCOT) framework to embed circular distributions into a $L^2$ space.

In this paper, we build on the work of \cite{martin2023lcot} and extend the linearized methodology from circles to hyperspheres. We develop a novel embedding technique for spherical probability distributions, which significantly enhances the efficiency of group analysis in spherical domains. 

\textbf{Contributions.} In summary, our contributions are as follows:
\begin{enumerate}
    \item We propose Linear Spherical Sliced Optimal Transport (LSSOT) to embed spherical distributions into $L^2$ space while preserving their intrinsic spherical geometry. 
    \item We prove that LSSOT defines a metric, and demonstrate the superior computation efficiency over other baseline metrics.
    \item We conduct a comprehensive set of experiments to show the effectiveness and efficiency of LSSOT in diverse applications, from point cloud analysis to cortical surface registration.
\end{enumerate}


\section{Background}

\subsection{Circular Optimal Transport and Linear Circular Optimal Transport}
\label{subsec:COT}
Consider two circular probability measures $\mu, \nu\in\mathcal{P}(\mathbb{S}^1)$, where $\mathbb{S}^1$ denotes the unit circle in $\mathbb{R}^2$. Let us parametrize $\mathbb{S}^1$ with the angles in between $0$ and $1$ and consider the cost function $c(x, y):=h(|x-y|_{\mathbb{S}^1})$, where $h:\mathbb{R}\rightarrow \mathbb{R}_+$ is a convex increasing function and 
$|x-y|_{\mathbb{S}^1}:=\min\{|x-y|,1-|x-y|\}$ for $x,y\in[0,1)$. The Circular Optimal Transport (COT) problem between $\mu$ and $\nu$ is defined by the following two equivalent minimization problems:
\begin{equation}
    COT_h(\mu,\nu):=\inf_{\gamma\in \Gamma(\mu, \nu)}\int_{\mathbb{S}^1\times\mathbb{S}^1}c(x, y)d\gamma(x, y)= \inf_{\alpha\in\mathbb{R}}\int_0^1 h(|F^{-1}_\mu(x)-F^{-1}_\nu(x-\alpha)|)dx,\label{eq:cot2}
\end{equation}
where in the first expression $\Gamma(\mu, \nu)$ is the set of all couplings between $\mu$ and $\nu$, and in the second expression
$F_\mu$ (respectively, $F_\nu$) is the cumulative distribution function of $\mu$  on $\mathbb{S}^1$
(i.e., $F_{\mu}(y):=\mu([0,y))=\int_0^y d\mu$, $\forall y\in[0,1)$)
extended on $\mathbb{R}$ by $F_\mu(y+n)=F_\mu(y)+n, \forall y\in[0, 1), n\in\mathbb{Z}$, 
and its inverse (or quantile function) is defined by $F^{-1}_{\mu}(y)=\inf\{x:F_{\mu}(x)>y\}$.
When $h(x)=|x|^p$ for $1\leq p < \infty$, we denote $COT_h(\cdot, \cdot)$ as $COT_p(\cdot, \cdot)$, and $COT_p(\cdot, \cdot)^{1/p}$ defines a metric on $\mathcal{P}(\mathbb{S}^1)$.
Moreover, if $\mu=\text{Unif}(\mathbb{S}^1)$ and $h(x)=|x|^2$, the minimizer $\alpha_{\mu, \nu}$ of \eqref{eq:cot2} is the antipodal of $\mathbb{E}(\nu)$, i.e., $\alpha_{\mu, \nu}=\mathbb{E}(\nu)-1/2$.
We refer the reader to \citet{delon2010fast,rabin2011,bonet2022spherical} for more details.

Recently, \citet{martin2023lcot} expanded the linear optimal transport framework \citep{Wang2012ALO} to encompass circular probability measures, introducing the Linear Circular Optimal Transport (LCOT) embedding and metric. For the reader’s convenience, we summarize the key definitions from their work below.

The \textbf{LCOT embedding} corresponds to the optimal circular displacement that
comes from the problem of transporting a \textit{reference measure} to a \textit{target measure} with respect to $COT_h(\cdot,\cdot)$. Precisely, 
given a reference measure $\mu\in\mathcal{P}(\mathbb{S}^1)$ 
that is absolutely continuous with respect to the Lebesgue measure on $\mathbb{S}^1$, the LCOT embedding of a target measure $\nu\in\mathcal{P}(\mathbb{S}^1)$ is defined by
\begin{equation}
    \widehat{\nu}^{\mu, h}(x):=F^{-1}_\nu(F_\mu(x)-\alpha_{\mu, \nu})-x, \quad \forall x\in [0, 1),
\end{equation}
where $\alpha_{\mu, \nu}$ is a minimizer of \eqref{eq:cot2}.  When $h(x)=|x|^p$, for $1\leq p<\infty$, this embedding gives rise to the \textbf{LCOT distance}
between circular measures:
\begin{align}
\label{eq:lcot}
        LCOT_{\mu,p}({\nu_1},{\nu_2}) &:=\|\widehat{\nu_1}^{\mu,h}-\widehat{\nu_2}^{\mu,h}\|_{L^p(\mathbb{S}^1,d\mu)}, \qquad \forall \nu_1,\nu_2\in\mathcal{P}(\mathbb{S}^1),
    \end{align}
where
\begin{equation}\label{eq:lp in embedding space}
  L^p(\mathbb{S}^1,d\mu):=\{f: \mathbb{S}^1 \to \mathbb{R}: \, \|f\|_{L^p(\mathbb{S}^1,d\mu)}^p:=\int_{\mathbb{S}^1}|f(t)|_{\mathbb{S}^1}^p \, d\mu(t)<\infty\}.  
\end{equation}
If $\mu=Unif(\mathbb{S}^1)$, we use the notation $L^p(\mathbb{S}^1):=L^p(\mathbb{S}^1,d\mu)$. If also $h(x)=|x|^2$, the LCOT embedding becomes
\begin{equation}\label{eq: embedding mu=unif}
    \widehat{\nu}:=F^{-1}_\nu\left(x-\mathbb{E}(\nu)+\frac{1}{2}\right)-x, \qquad \forall x\in [0, 1),
\end{equation}
and we denote the LCOT distance simply as $LCOT_2(\cdot,\cdot)$.

\subsection{Spherical Slicing}
\citet{bonet2022spherical} recently presented the Spherical Sliced Wasserstein (SSW) discrepancy for probability measures supported on the sphere $\mathbb{S}^{d-1}\subset \mathbb{R}^d$, that uses projections onto great circles $C$:
\begin{equation}\label{e1: proj_great_circle_original}    
    P^C(x):=\argmin_{y\in C} \{\arccos(\langle x, y\rangle)\}, \qquad \forall x\in \mathbb{S}^{d-1},
\end{equation}
where $\arccos(\langle \cdot, \cdot\rangle)$ is the geodesic distance on $\mathbb{S}^{d-1}$. Since each great circle is obtained by intersecting the sphere with a 2-dimensional plane in $\mathbb{R}^d$, the authors proposed to average over all great circles by integrating over all planes in $\mathbb{R}^d$ (i.e., the Grassmann manifold 
 \citep{bendokat2020grassmann}) and then project the distributions onto the intersection with $\mathbb{S}^{d-1}$. However, due to the impracticality of the Grassmann manifold, the set of rank 2 projectors is used instead:
 \begin{equation*}
     {V}_{2}(\mathbb{R}^d):=\{U\in\mathbb{R}^{d\times 2}: \, U^TU=I_2\},
 \end{equation*}
 where ${V}_{2}(\mathbb{R}^d)$ is the Stiefel manifold \citep{bendokat2020grassmann}.
 Hence, the integration process can be performed over $V_2(\mathbb{R}^d)$ according to the uniform distribution $\sigma\in \mathcal{P}(V_2(\mathbb{R}^d))$. The projection onto $U\in V_2(\mathbb{R}^d)$ can be computed as in \citep[Lemma 1]{bonet2022spherical} by
 \begin{equation}
 \label{eq:projection}
    P^U(x)=\frac{U^Tx}{\|U^Tx\|_2}, \qquad \forall x\in \mathbb{S}^{d-1}\setminus\{x: U^Tx=0\}
\end{equation}
Finally, given $\mu, \nu\in \mathcal{P}(\mathbb{S}^{d-1})$, the Spherical Sliced Wasserstein discrepancy between them, for $1\leq p<\infty$, is defined in \citep{bonet2022spherical} by
\begin{equation}\label{eq: SSW}
    (SSW_{p}(\mu,\nu))^p:=\int_{V_2(\mathbb{R}^d)}\left(COT_p(P_{\#}^U\mu,P_{\#}^U\nu)\right)^p \, d\sigma(U).
\end{equation}


\section{Method}
In this section, we first introduce the Linear Spherical Sliced Optimal Transport (LSSOT) embedding and discrepancy, and then state the main theorem ---- LSSOT indeed defines a metric in $\mathcal{P}(\mathbb{S}^{d-1})$.

\begin{figure}
  \begin{center}
    \includegraphics[width=0.5\textwidth]{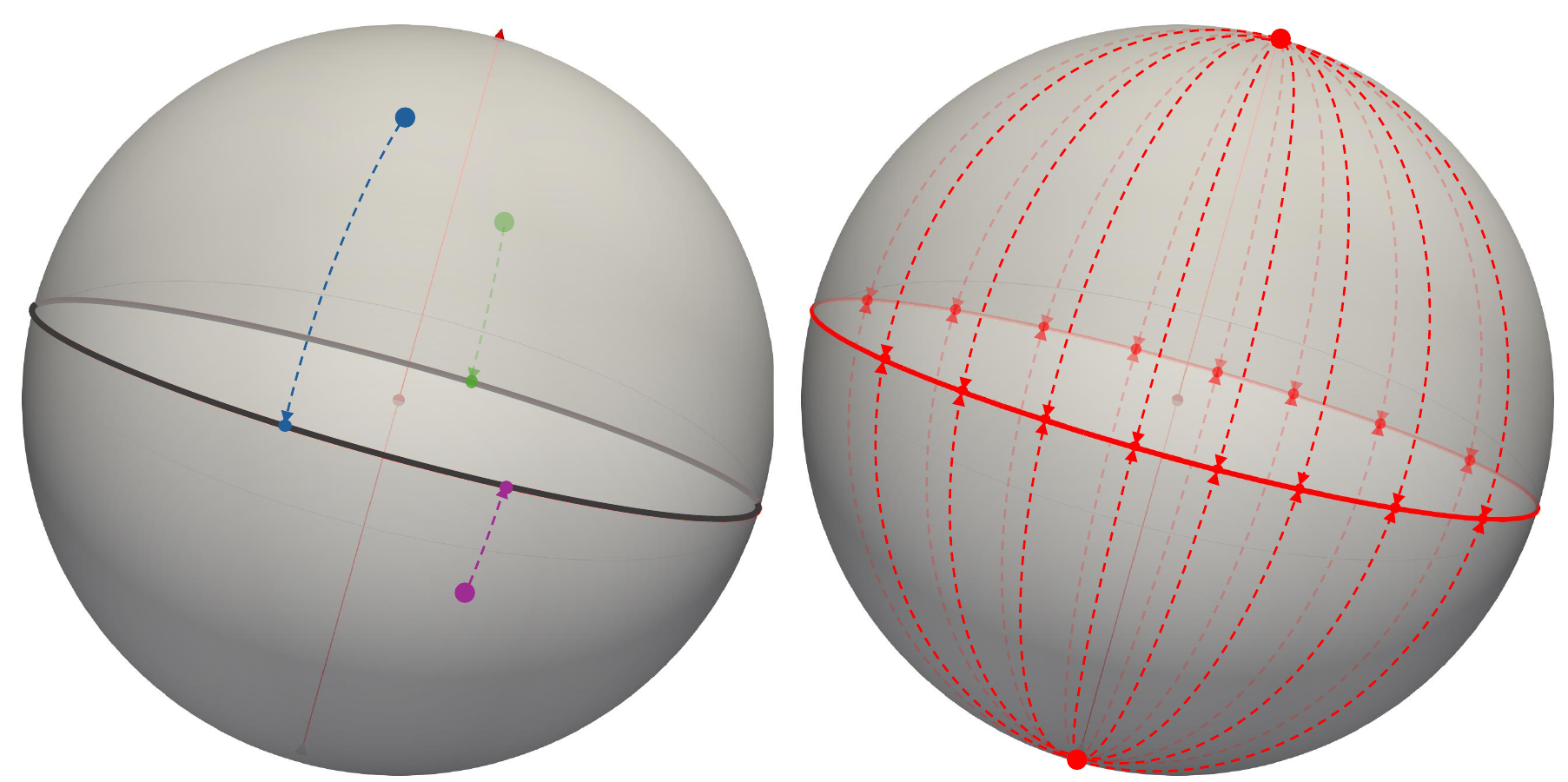}
  \end{center}
   \vspace{-0.1in}
  \caption{Semicircle Transform projects non-polar points onto a great circle, and the north/south poles to everywhere on the circle.}
  \label{fig:LSSOT_projection}
  \vspace{-.1in}
\end{figure}

\subsection{Linear Spherical Sliced Optimal Transport (LSSOT)}



\begin{definition}
    Consider spherical measures $\nu_1,\nu_2\in\mathcal{P}(\mathbb{S}^{d-1})$. We define the \textbf{Linear Spherical Sliced  Optimal Transport (LSSOT) embedding} of $\nu_i$,  $i=1,2$, with respect to a slice $U\in V_2(\mathbb{R}^d)$ by
\begin{equation}
    \widehat{\nu_i}^S(x,U):=F_{P_{\#}^U\nu_i}^{-1}\left(P^U(x)-\mathbb{E}(P_{\#}^U\nu_i)+\frac{1}{2}\right), \quad x\in\mathbb{S}^{d-1}, \, U\in V_2(\mathbb{R}^d), 
\end{equation}
where $P^U_\#\nu_i$ denotes the \text{pushforward} measure on $\mathbb{S}^1$
of $\nu_i$ by the map $P^U$.
We also define their \textbf{LSSOT discrepancy} by
\begin{align}\label{eq: LSSOT}
 \left(LSSOT_2(\nu_1,\nu_2)\right)^2&:=\int_{V_2(\mathbb{R}^d)}\left(LCOT_2(P_{\#}^U\nu_1,P_{\#}^U\nu_2)\right)^2 \, d\sigma(U)\\
 &=\int_{V_2(\mathbb{R}^d)}\|\widehat{\nu_1}^{S}(\cdot,U)-\widehat{\nu_2}^{S}(\cdot,U)\|_{L^2(\mathbb{S}^1)}^2 \, d\sigma(U),\notag
\end{align}
\vspace{.2in}
\end{definition}

\begin{remark}[LSSOT extends LCOT]
   When $d=2$, $P^U=I_2$ for every $U\in V_2(\mathbb{R}^2)$
   since $V_2(\mathbb{R}^2)$ coincides with the set of orthonormal basis on $\mathbb R^2$. 
   Thus, for any $\nu\in \mathcal{P}(\mathbb{ S}^1)$ we have that $\widehat{\nu}^S(x,U)=\widehat{\nu}^S(x,V)$ for all $U, V\in V_2(\mathbb{R}^2)$, so the LSSOT embedding and the LSSOT discrepancy coincide with the LCOT embedding and with the LCOT distance respectively, as the integration is normalized to 1 on $V_2(\mathbb{R}^2)$ (which is isomorphic to the orthogonal group $O(2)$ of dimension 1). 
\end{remark}

\begin{theorem}[Metric property of LSSOT]\label{thm: metric LSSOT}
    $LSSOT_2(\cdot,\cdot)$ defines a pseudo-metric in $\mathcal{P}(\mathbb{S}^{d-1})$, and a metric when restricting to probability measures with continuous density functions. We will refer to it as $LSSOT$ distance.
\end{theorem}
We refer the reader to the Appendix for detailed proof. The symmetry of $LSSOT_2(\cdot,\cdot)$ is easy to deduce from its definition. The triangle inequality holds as an application of Minkowski inequality by using the fact that $LCOT_2(\cdot,\cdot)$ is a metric (see Proposition \ref{thm: pseudo metric} in Appendix \ref{app: metric property}). For proving the identity of indiscernibles, the idea is to give an alternative expression of the $LSSOT_2(\cdot,\cdot)$ involving a Spherical Radon Transform $\mathcal{R}$: see \eqref{eq: LSSOT in terms of Radon}  in Appendix \ref{sec: LSSOT in terms of Radon}, which is inspired by \cite{bonet2022spherical}. Such formula reads for probability density functions $f_1,f_2\in L^1(\mathbb{S}^{d-1})$\footnote{(which is equivalent to considering two absolutely continuous probability measures on $\mathbb S^{d-1}$)} as
\begin{equation}\label{eq: LSSOT with R for functions}
    \left(LSSOT_2(\nu_1,\nu_2)\right)^2=\int_{V_2(\mathbb{R}^d)}\left(LCOT_2(\mathcal{R}f_1(U,\cdot),\mathcal{R}f_2(U,\cdot))\right)^2 \, d\sigma(U)
\end{equation}
where $\mathcal{R}:L^1(\mathbb{S}^{d-1})\to L^1(V_2(\mathbb R^d)\times \mathbb{S}^1)$ is the integral operator 
$$\mathcal{R}f(U,z)=\int_{\{x\in \mathbb S^{d-1}: \, z=P^U(x)\}} f(x)dx,$$
where the surface/line integral is taken over a $(d-2)$-dimensional surface, which coincides with a (1-dimensional) \textit{great semicircle} when $d=3$. We will call this transformation as \textit{Semicircle Transform} when $d=3$, and \textit{$(d-2)$-Hemispherical Transform on $\mathbb{S}^{d-1}$} in general. We describe $\mathcal{R}$ acting on $L^1(\mathbb{S}^2)$ by using different and detailed parameterizations in Appendix \ref{app: Radon S2} and generalize it for any dimension $d$ in Appendix \ref{app: Radon general d}.  In Appendix \ref{app: Radon for measures}, $\mathcal{R}$ is extended to the set of measures on the sphere, and the disintegration theorem from measure theory is used to obtain the general version of \eqref{eq: LSSOT with R for functions}. 

The \textbf{novelty} is that we will relate $\mathcal{R}$ with an integral transformation presented in \cite{groemer1998spherical}, instead of relating it with the Hemispherical Transform in \cite{rubin1999inversion} as done in \cite{bonet2022spherical} (see Remark \ref{remark: our difference with Bonet et al} for a discussion in this regard). By doing so, we can prove the injectivity of $\mathcal{R}$ on continuous functions ( Corollary \ref{eq: 1-1 for cont}), which is the key property to show that if $LSSOT_2(\nu_1,\nu_2)=0$ then $\nu_1=\nu_2$. As a byproduct, we also obtain that $SSW_p(\cdot,\cdot)$ is not only a pseudo-metric as proven in \cite{bonet2022spherical}, but also a metric when it is restricted to probability measures with continuous density functions (Theorem \ref{thm: SSW metric} in Appendix \ref{app: metric property}).       

\begin{figure}[t]
\vspace{-0.2in}
\begin{algorithm}[H]
    \caption{Linear Spherical Sliced Optimal Transport (LSSOT)}\label{alg:lssot}
    \begin{algorithmic}
        \REQUIRE Spherical distributions $\mu=\sum_{i=1}^{N_\mu} a_i \delta_{x_i}$ and $\nu=\sum_{j=1}^{N_\nu} b_j\delta_{y_j}$; number of slices $L$; \\
        reference size $M$, threshold $\epsilon$
        \FOR{$l=1$ to $L$}
        \STATE Construct a matrix $Z_l\in \mathbb{R}^{d\times 2}$ with entries randomly drawn from $\mathcal{N}(0,1)$
        \STATE Apply QR decomposition on $Z_l$ to get an orthogonal $Q$ matrix $U_l$
        \STATE Project $\{x_i\}_{i=1}^{N_\mu}$ and $\{y_j\}_{j=1}^{N_\nu}$ on $\mathbb{R}^2$ to get $\hat{x}^l_i=U_l^Tx_i$ and $\hat{y}^l_j=U_l^Ty_j$, $\forall i, j$
        \STATE Find $I^*_l=\{i^*\}, J^*_l=\{j^*\}$ such that $\|\hat{x}^l_{i^*}\|\leq \epsilon$ or $\|\hat{y}^l_{j^*}\|\leq\epsilon$ 

        \hspace*{31em}%
        \rlap{\smash{$\left.\begin{array}{@{}c@{}}\\{}\\{}\\{}\\{}\end{array}\right\}%
        \begin{tabular}{l} Apply $\epsilon$-cap \\around poles \end{tabular}$}}
        
        \STATE $a^*_l=\sum_{i^*\in I^*} a_{i^*}$, $b^*_l=\sum_{j^*\in J^*} b_{j^*}$; $\tilde{N}_\mu=N_\mu-|I^*_l|$, $\tilde{N}_\nu=N_\nu-|J^*_l|$
        \STATE $\tilde{\mu}_l=\sum_{\tilde{i}=1}^{\tilde{N}_\mu}(a_{\tilde{i}}+\frac{a^*}{\tilde{N}_\mu})\delta_{\hat{x}^l_{\tilde{i}}}$, $\tilde{\nu}_l=\sum_{\tilde{j}=1}^{\tilde{N}_\nu}(b_{\tilde{j}}+\frac{b^*}{\tilde{N}_\nu})\delta_{\hat{y}^l_{\tilde{j}}}$
        \STATE Project all $\hat{x}_{\tilde{i}}^l$ and $\hat{y}_{\tilde{j}}^l$ on $\mathbb{S}^1$: $P(\hat{x}_{\tilde{i}}^l) = \frac{\hat{x}_{\tilde{i}}^l}{\|\hat{x}_{\tilde{i}}^l\|}$, $P(\hat{y}_{\tilde{j}}^l) = \frac{\hat{y}_{\tilde{j}}^l}{\|\hat{y}_{\tilde{j}}^l\|}$
        \STATE Calculate $LCOT_{\bar{\mu}, 2}(\tilde{\mu}_l^{proj}, \tilde{\mu}_l^{proj})$ by \eqref{eq:lcot}, where $\tilde{\mu}_l^{proj}=\sum_{\tilde{i}=1}^{\tilde{N}_\mu}(a_{\tilde{i}}+\frac{a^*}{\tilde{N}_\mu})\delta_{P(\hat{x}^l_{\tilde{i}})}$, $\tilde{\nu}_l^{proj}=\sum_{\tilde{j}=1}^{\tilde{N}_\nu}(b_{\tilde{j}}+\frac{b^*}{\tilde{N}_\nu})\delta_{P(\hat{y}^l_{\tilde{j}})}$, $\bar{\mu}$ is the discrete uniform reference measure of size $M$
        \ENDFOR
        \RETURN{} $LSSOT_2 (\mu, \nu) \approx \left(\frac{1}{L}\sum_{l=1}^L LCOT_2^2(\tilde{\mu}_l^{proj}, \tilde{\mu}_l^{proj})\right)^{\frac{1}{2}}$
    \end{algorithmic}
\end{algorithm}
\vspace{-0.4in}
\end{figure}


\subsection{LSSOT Implementation}
Following the Monte-Carlo approach for computing the classical Sliced Wasserstein distance, we consider samples from the uniform distribution 
$\sigma\in\mathcal{P}(V_2(\mathbb{R}^d))$ for generating the slices. To do so, a standard procedure is to first construct matrices in $\mathbb{R}^{d\times 2}$ whose components are drawn from the standard normal distribution $\mathcal{N}(0,1)$, then use the Gram-Schmidt algorithm to orthonormalize the column vectors of the matrix, or apply the QR-decomposition and finally normalize the orthogonal matrix $Q$ (dividing each column vector of $Q$ by its Euclidean norm). In contrast to the Spherical Sliced Wasserstein (SSW) numerical implementation in which projections of poles are not well defined (when $U^Tx=0$ in \ref{eq:projection}), LSSOT maps the polar points everywhere on the equator by definition. To achieve this mapping numerically, we put an $\epsilon$-cap around the poles, and redistribute all the mass inside the $\epsilon$-cap evenly onto the discrete points outside the cap. See Figure \ref{fig:LSSOT_projection} for a visual depiction of the projection and Algorithm \ref{alg:lssot} for detailed implementation steps.
\vspace{-0.1in}
\subsection{Time Complexity and Runtime Analysis}
\vspace{-0.1in}
Let $L$ denote the number of slices and $d$ denote the dimension. For a discrete spherical distribution $\nu$ with $N$ samples, the slicing and projection require $\mathcal{O}(LdN)$. For a single slice, the time complexity of calculating the LCOT embedding (Equation \ref{eq: embedding mu=unif}) of the corresponding circular distributions with respect to the uniform distribution is $\mathcal{O}(N)$ after sorting. Therefore the time complexity of calculating LSSOT embedding of $\nu$ is $\mathcal{O}(LN(d+\log N+1))$, where $LN\log N$ comes from the sorting $\mathbb{S}^1$. 

For a set of $K$ distributions $\{\nu_k\}_{k=1}^K$, each with $N$ samples, the pairwise LSSOT distances can be boiled down to $K$ LSSOT embeddings and $\frac{LK(K-1)}{2}$ distances in $L^2(\mathbb{S}^1)$, thus cost $\mathcal{O}(KLN(d+\log N+1+\frac{K-1}{2}))$. However, other methods such as spherical Optimal Transport (OT), Sinkhorn divergence, Sliced Wasserstein Distance (SWD), SSW, and Stereographic Spherical Sliced Wasserstein distance (S3W) all require the time complexity $\mathcal{O}(K^2D)$, where $D$ is the complexity of each distance, e.g. $\mathcal{O}(K^2N^3\log N)$ for spherical OT and $\mathcal{O}(K^2LN\log N)$ for SSW, compared to $\mathcal{O}(K^2LN)$ for our method. Figure \ref{fig:computation_cpu} shows the pairwise distances computation time of LSSOT along with the methods mentioned above. All methods are run on the CPU for fair comparison. The distributions are randomly generated with sample sizes $N=1000, 5000, 10000, 12500, 15000$. For each fixed sample size, we observe that LSSOT stands out as the fastest in all methods under the increase in the number of distributions. In the Appendix \ref{app: runtime-slices}, we also provide the runtime comparisons for slice-based methods with varying numbers of slices.
\label{subsec: runtime}
\begin{figure}[H]
    \vspace{-0.1in}
    \centering
    \includegraphics[width=\linewidth]{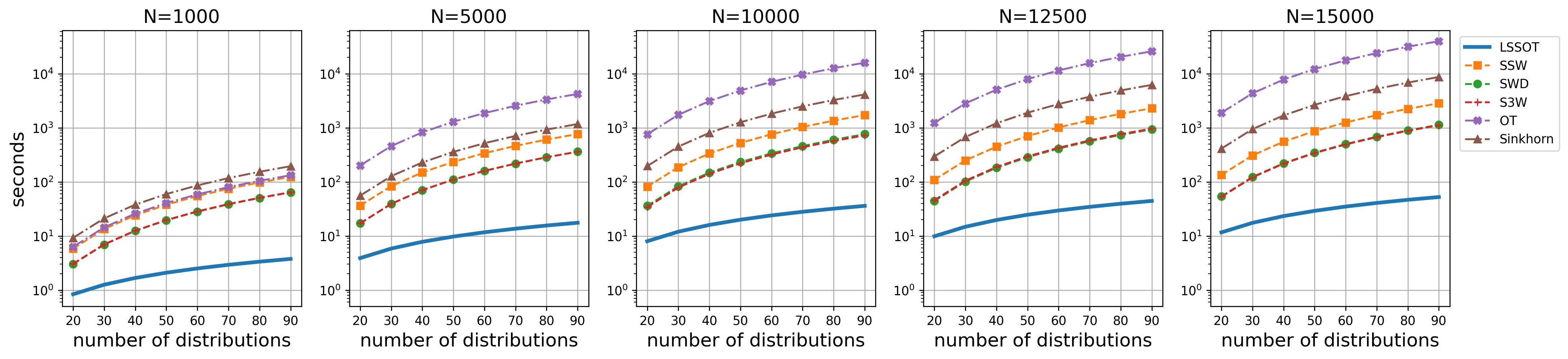}
     \vspace{-0.2in}
    \caption{Pairwise distances runtime (log scale) comparison w.r.t the number of distributions. $N$ denotes sample sizes in each distribution. The number of slices is 500 for all slice-based methods.}
    \label{fig:computation_cpu}
\end{figure}

\section{Experiments}
To demonstrate that the proposed LSSOT is an effective and efficient metric of spherical distributions while preserving the geometry, we design and implement the following experiments on toy examples, cortical surfaces, and point clouds.
\subsection{Rotation and Scaling in Spherical Geometry}
We test the LSSOT metric on spherical geometric transformations such as rotation and scaling.

For rotations, we generate a set of 20 Von Mises-Fisher (VMF) distributions $\{p^i(x;\mu_i, \kappa) = \frac{\kappa}{2\pi(e^\kappa-e^{-\kappa})}\exp(\kappa\mu_i^T x)\}_{i=1}^{20}$ with the same concentration parameter $\kappa$, but the mean directions $\mu_i$'s are 20 distinct points on the Fibonacci Sphere, i.e. spread evenly on $\mathbb{S}^2$. Since the 20 VMFs are of the same shape with varying mean directions, each can be interpreted as a rotated version of another. We calculate the pairwise LSSOT distances among the 20 distributions and visualize the embeddings in 3-dimensional space using multi-dimensional scaling (MDS). We use the Scikit Learn \citep{scikit-learn} sklearn.manifold.MDS function for MDS implementation. As shown in Figure \ref{fig:vmf} (left), the embeddings spread almost evenly on the sphere, implying the LSSOT captures the rotation geometry.

To study the scaling geometry, we generate 6 VMF distributions with fixed $\mu$ but varying $\kappa$, and then apply the same LSSOT pairwise distances calculation followed by the MDS embedding procedure. Figure \ref{fig:vmf} (right) illustrates that the LSSOT metric preserves the linear transitions in scaling geometry, with embeddings falling approximately onto a straight line. 
\begin{figure}[H]
    \centering
    \includegraphics[width=\textwidth]{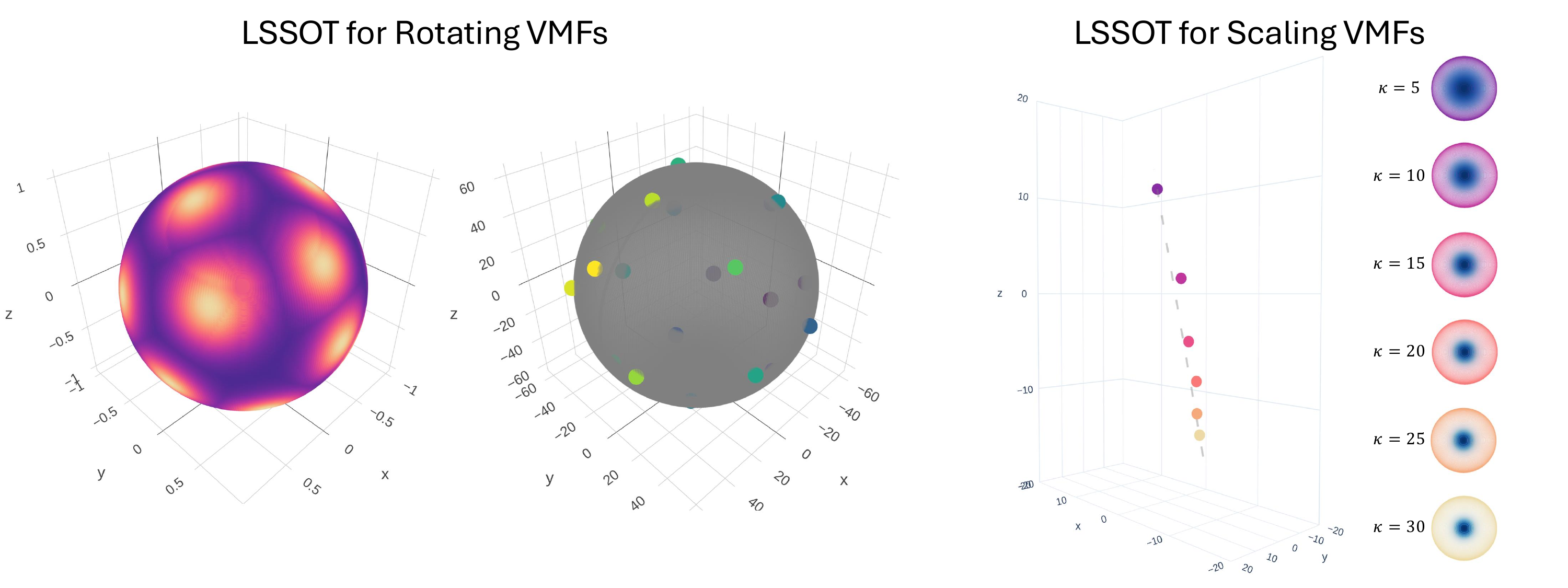}
    \caption{Generated VMF distributions and manifold learning results by LSSOT. Left: a group of 20 VMFs generated by rotating a source VMF, and the corresponding 3-dimensional visualizations by LSSOT and MDS. Right: 6 VMFs generated by scaling a source VMF with $\kappa=30$, and the same 3-dimensional visualizations. }
    \label{fig:vmf}
\end{figure}

\subsection{Cortical Surface Registration}
We further verify the validity and efficiency of the proposed LSSOT method in comparing spherical cerebral cortex data for registration tasks. Cortical surface registration seeks to establish meaningful anatomical correspondences across subjects or time points. This process is essential for surface-based analysis in both cross-sectional and longitudinal neuroimaging studies. Cortical surface scans are often acquired from diverse spatial frameworks, necessitating a preliminary alignment process to situate them within a common space. This alignment is crucial for enabling meaningful group analyses and comparisons. The unit sphere $\mathbb{S}^2$ is usually chosen to be this common space, owing to the inherent genus-0 spherical topology of the cortical surfaces \citep{wang2004intrinsic,gu2004genus,lui2007landmark}. Thus, such aligning involves mapping each hemisphere of the convoluted cerebral cortical surface onto the $\mathbb{S}^2$ space with minimal distortions. Then the goal is to register spherical cortical surfaces to a template, i.e. computing a smooth velocity field on the sphere which is used to deform the moving vertices into the fixed template. Through this deformation, correspondences are established across a group of cortical surfaces. 
\subsubsection{Experimental Setup}
We leverage the Superfast Spherical Surface Registration (S3Reg) \citep{9389746} neural network to perform atlas-based registration, which focuses on registering all surfaces to one atlas surface. S3Reg is an end-to-end unsupervised diffeomorphic registration network based on the Voxelmorph framework \citep{Balakrishnan2018VoxelMorphAL} and the Spherical Unet \citep{Zhao2019SphericalUO} backbone, offering flexibilities in choosing the similarity measure for spherical surfaces in the training objectives. Following the S3Reg experimental setup with resolution level at the 6-th icosahedron subdivisions (i.e. 40962 vertices on each surface), we replace the original mean squared error (MSE) similarity loss with our LSSOT distance and other baselines, and evaluate the registration performance achieved by each similarity measure. We use the Freesurfer fsaverage \citep{fischl2012freesurfer} surface as the fixed (atlas) surface. All models are implemented in PyTorch \citep{paszke2017automatic} and run on a Linux server with NVIDIA RTX A6000 GPU. More experimental details are documented in Appendix \ref{app: cortical-reg}.

\textbf{Dataset and Preprocessing.}
We perform registration on the NKI dataset \citep{nooner2012nki} and ADNI dataset \citep{jack2008alzheimer}. NKI contains T1-weigted brain MRI images from 515 subjects, which are processed to obtain reconstructed cortical surfaces, spherical mappings, parcellations, and anatomical features using the FreeSurfer pipeline \citep{fischl2012freesurfer}. A subset of the ADNI dataset that contains 412 subjects is used and processed in the same FreeSurfer pipeline. We choose the sulcal depth as the single-channel feature for registration and normalize features into probability vectors. For parcellations, we choose the DKT Atlas \citep{klein2012101,klein2017mindboggling}, which parcellates each hemisphere into 31 cortical regions of interest (ROIs). We split both datasets into 70\% scans for training, 10\% scans for validation, and 20\% scans for testing. 

\textbf{Training Losses.}
The training loss has three components: the similarity loss for aligning the moving surfaces to the fixed atlas surface, the Dice loss to impose biological validity, and the regularization loss to enforce the smoothness of the deformation field. Specifically, we employ the parcellation maps generated by FreeSurfer and the torchmetrics.Dice function from the torchmetrics library \citep{Nicki_Skafte_Detlefsen_and_Jiri_Borovec_and_Justus_Schock_and_Ananya_Harsh_and_Teddy_Koker_and_Luca_Di_Liello_and_Daniel_Stancl_and_Changsheng_Quan_and_Maxim_Grechkin_and_William_Falcon_TorchMetrics_-_Measuring_2022} for Dice loss. The training loss is a linear combination of the three components. In our experimental design, we keep the regularization and the Dice terms in all setups. Our primary focus is on investigating the impact of various similarity losses on the registration performance, including LSSOT distance and the following baselines.

\textbf{Baselines.}
Since the original S3Reg framework used MSE as part of the similarity loss, we include MSE as one of the baselines. Besides, Spherical Sliced-Wasserstein (SSW) \citep{bonet2022spherical} and Stereographic Spherical Sliced Wasserstein Distances (S3W) \citep{tran2024stereographic} are powerful spherical distances to test LSSOT against. We use the SSW implementation from the official repo \footnote{https://github.com/clbonet/Spherical\_Sliced-Wasserstein} and the Amortized Rotationally Invariant Extension of S3W (AI-S3W) from the official implementation \footnote{https://github.com/mint-vu/s3wd} with 10 rotations. We also add the Slice Wasserstein distance (SWD), which is defined in Euclidean space $\mathbb{R}^3$ instead of $\mathbb{S}^2$. The implementation of the SWD algorithm is from the Python Optimal Transport (POT) library \citep{flamary2021pot}.

\textbf{Metrics}
We quantitatively evaluate the similarities between the registered moving surfaces and the fixed atlas surface using mean absolute error (MAE) and Pearson correlation coefficient (CC) on features, as well as LSSOT, SWD on the spherical distributions (represented by vertex locations and feature probability vectors). We calculate the Dice score of the registered parcellations and the fixed parcellation. Moreover, we measure the area (resp. edge) distortions between the original moving surface meshes and the deformed surface meshes as relative changes in the areas of all faces (resp. the lengths of all edges) of meshes.

\subsubsection{Results}
Table \ref{tab: cortical-reg} summarizes the performance of registration evaluated by the metrics for four scenarios: left/right hemispheres of the NKI and ADNI datasets. As the two best methods overall, our proposed LSSOT metric and SSW are on par with each other, yet LSSOT costs less than 10\% of the training time. We emphasize that LSSOT is the leading method in the alignment of the parcellations (i.e. Dice Score), indicating that LSSOT yields more biologically meaningful registrations. Besides, with the least area/edge distortions, LSSOT finds the simplest solution of the deformation field to align the moving surface to the fixed surface. Figure \ref{fig:nki-A00065749} shows the registered sulcal depth map and parcellation map of a random subject. More visualizations can be found in the Appendix \ref{app: visual-cortical-reg}. It's worth noting that the LSSOT embedding $\hat{v}^S_{\text{fix}}$ of the fixed surface is calculated once and for all before training, and in each training iteration, only the LSSOT embedding of the deformed moving surface $\hat{v}^S_{\text{mov}}$ is calculated and then compared to $\hat{v}^S_{\text{fix}}$. Other competing methods, however, need to process both surfaces and calculate the distance during each iteration, which accounts for LSSOT taking less time than SWD. This makes LSSOT a perfect fit for atlas-based registrations, as the template is fixed. 

\begin{table}[t!]
    \centering
    \scalebox{0.9}{
    \begin{NiceTabular}{ccccccc}
        \thickhline
       && \multirow{2}{*}{MSE} & \multirow{2}{*}{S3W} & \multirow{2}{*}{SWD} & \multirow{2}{*}{SSW} & \multirow{2}{*}{LSSOT (Ours)} \\ 
       &&&&&& \\
       \hline
	\multirow{8}{*}{\centering \rotatebox{90}{NKI (Left Hemisphere)}}
         &\rule{0pt}{3ex}   LSSOT($\downarrow$) &0.2754$\pm$0.0611&0.2298$\pm$0.0346 & 0.2411$\pm$ 0.0366 &\textbf{0.2079$\pm$0.0369}&\textbf{0.1890$\pm$0.0361} \\ 
         &\rule{0pt}{3ex}   SWD($\downarrow$) &\textbf{0.0051$\pm$0.0017}&0.0053$\pm$0.0010 & \textbf{0.0027$\pm$0.0011}&0.0059$\pm$0.0011&0.0052$\pm$0.0011 \\ 
         &\rule{0pt}{3ex}   MAE($\downarrow$) &\textbf{0.1129$\pm$0.0471}&0.2278$\pm$0.0372& 0.2658$\pm$0.0410&\textbf{0.2145$\pm$0.0266}&0.2516$\pm$0.0397 \\ 
         &\rule{0pt}{3ex}   CC($\uparrow$) &0.8269$\pm$0.0425&\textbf{0.9216$\pm$0.0174}&\textbf{0.8722$\pm$0.0302}&0.8671$\pm$0.0314&0.8649$\pm$0.0295 \\ 
         &\rule{0pt}{3ex}   Dice ($\uparrow$) &0.7746$\pm$0.0861&\textbf{0.8498$\pm$0.0670}&0.7984$\pm$0.0539&0.8429$\pm$0.0692&\textbf{0.8462$\pm$0.0548} \\ 
         &\rule{0pt}{3ex}   Edge Dist.($\downarrow$) &0.3060$\pm$0.0404&0.3922$\pm$0.0647&0.3442$\pm$0.0346&\textbf{0.2476$\pm$0.0348}&\textbf{0.2365$\pm$0.0343} \\ 
         &\rule{0pt}{3ex}   Area Dist.($\downarrow$) &0.4305$\pm$0.0488&0.4048$\pm$0.0752&0.4073$\pm$0.0368&\textbf{0.2733$\pm$0.0402}&\textbf{0.2897$\pm$0.0396} \\ 
         &\rule{0pt}{3ex}   Time(seconds)($\downarrow$) &\textbf{73.07}&121.00& 118.96&1350.96&\textbf{101.01} \\ 
       \hline
       \multirow{8}{*}{\rotatebox{90}{NKI (Right Hemisphere)}}

         &\rule{0pt}{3ex}LSSOT($\downarrow$) &0.2847$\pm$0.0502&0.2062$\pm$0.0510&0.2454$\pm$0.0396&\textbf{0.1341$\pm$0.0209}&\textbf{0.1121$\pm$0.0253} \\ 
         &\rule{0pt}{3ex}SWD($\downarrow$) &0.0048$\pm$0.0015&0.0046$\pm$0.0009& \textbf{0.0019$\pm$0.0003}&\textbf{0.0031$\pm$0.0017}&0.0033$\pm$0.0015 \\ 
         &\rule{0pt}{3ex}MAE($\downarrow$) &\textbf{0.1411$\pm$0.0215}&\textbf{0.2224$\pm$0.0324}&0.2379$\pm$0.0348&0.2268$\pm$0.0551&0.2489$\pm$0.0568 \\ 
         &\rule{0pt}{3ex}CC($\uparrow$) &\textbf{0.9334$\pm$0.0111}&0.9145$\pm$0.0130&0.9191$\pm$0.0263&0.8979$\pm$0.0383&\textbf{0.9284$\pm$0.0599} \\ 
         &\rule{0pt}{3ex}Dice ($\uparrow$) &0.7318$\pm$0.0943&0.6961$\pm$0.0777&0.6972$\pm$0.0820&\textbf{0.7952$\pm$0.0751}&\textbf{0.7996$\pm$0.0706} \\ 
         &\rule{0pt}{3ex}Shape Dist.($\downarrow$) &0.2872$\pm$0.0421&\textbf{0.1395$\pm$0.0709}&0.2160$\pm$0.0406&0.1460$\pm$0.0417&\textbf{0.1428$\pm$0.0162} \\ 
         &\rule{0pt}{3ex}Area Dist.($\downarrow$) &0.3453$\pm$0.0422&0.3125$\pm$0.0876&0.2924$\pm$0.0412&\textbf{0.1709$\pm$0.0126}&\textbf{0.1910$\pm$0.0189} \\ 
         &\rule{0pt}{3ex}Time(seconds)($\downarrow$) &\textbf{72.14}&121.02& 116.91 &1362.16&\textbf{100.68} \\

	\hline 
	\multirow{8}{*}{\rotatebox{90}{ADNI (Left Hemisphere)}}

         &\rule{0pt}{3ex}   LSSOT($\downarrow$) &0.1476$\pm$0.0396&0.1340$\pm$0.0397 & 0.1412$\pm$0.0419 &\textbf{0.1262$\pm$0.0261}&\textbf{0.1041$\pm$0.0277} \\ 
         &\rule{0pt}{3ex}   SWD($\downarrow$) &0.0046$\pm$0.0011&\textbf{0.0041$\pm$0.0011} & \textbf{0.0027$\pm$0.0009}&0.0052$\pm$0.0014&0.0063$\pm$0.0016 \\ 
         &\rule{0pt}{3ex}   MAE($\downarrow$) &\textbf{0.2784$\pm$0.0367}&0.3295$\pm$0.0427& 0.3627$\pm$0.0477&\textbf{0.3153$\pm$0.0485}&0.3332$\pm$0.0873 \\ 
         &\rule{0pt}{3ex}   CC($\uparrow$) &\textbf{0.9141$\pm$0.0269}& 0.9077$\pm$0.0207&\textbf{0.9139$\pm$0.0406}&0.9032$\pm$0.0139&0.9036$\pm$0.1128 \\ 
         &\rule{0pt}{3ex}   Dice ($\uparrow$) &0.7859$\pm$0.0949&0.7733$\pm$0.0816&0.7457$\pm$0.0683&\textbf{0.7987$\pm$0.0775}&\textbf{0.8119$\pm$0.0940} \\ 
         &\rule{0pt}{3ex}   Edge Dist.($\downarrow$) &0.3241$\pm$0.0516&0.2836$\pm$0.0750&0.2379$\pm$0.0459&\textbf{0.1927$\pm$0.0582}&\textbf{0.0979$\pm$0.0129} \\ 
         &\rule{0pt}{3ex}   Area Dist.($\downarrow$) &0.2990$\pm$0.0580&0.1937$\pm$0.0849&0.1930$\pm$0.0450&\textbf{ 0.1008$\pm$0.0053}&\textbf{0.1303$\pm$0.0088} \\ 
         &\rule{0pt}{3ex}   Time(seconds)($\downarrow$) &\textbf{59.34}&100.78&96.69 &1094.06&\textbf{84.51} \\

	\hline
	\multirow{8}{*}{\rotatebox{90}{ADNI (Right Hemisphere)}}

        &\rule{0pt}{3ex}LSSOT($\downarrow$) &0.1807$\pm$0.0347&0.1367$\pm$0.0375&0.1497$\pm$0.0509&\textbf{0.1349$\pm$0.0541}&\textbf{0.1295$\pm$0.0281} \\ 
         &\rule{0pt}{3ex}SWD($\downarrow$) &\textbf{0.0034$\pm$0.0010}&0.0046$\pm$0.0011& \textbf{0.0020$\pm$0.0005}&0.0041$\pm$0.0016&0.0037$\pm$0.0020 \\ 
         &\rule{0pt}{3ex}MAE($\downarrow$) &\textbf{0.2433$\pm$0.0273}&0.3229$\pm$0.0386&\textbf{0.2618$\pm$0.0450}&0.3236$\pm$0.0519&0.3493$\pm$0.0732 \\ 
         &\rule{0pt}{3ex}CC($\uparrow$) &\textbf{0.9330$\pm$0.0163}&0.9139$\pm$0.0204&0.8985$\pm$0.0391&\textbf{0.9375$\pm$0.0437}&0.8013$\pm$0.0851 \\ 
         &\rule{0pt}{3ex}Dice ($\uparrow$) &0.8073$\pm$0.0983&0.7617$\pm$0.0806&0.8122$\pm$0.0648&\textbf{0.8622$\pm$0.0915}&\textbf{0.8518$\pm$0.0879} \\ 
         &\rule{0pt}{3ex}Shape Dist.($\downarrow$) &0.3371$\pm$0.0627&0.2868$\pm$0.0835&0.3690$\pm$0.0513&\textbf{0.1841$\pm$0.0576}&\textbf{0.1226$\pm$0.0213} \\ 
         &\rule{0pt}{3ex}Area Dist.($\downarrow$) &0.3015$\pm$0.0568&0.3109$\pm$0.0987&0.4123$\pm$0.0438&\textbf{0.1939$\pm$0.0392}&\textbf{0.1948$\pm$0.0325} \\ 
        &\rule{0pt}{3ex}   Time(seconds)($\downarrow$) &\textbf{59.69}&101.35&97.52 &1114.83&\textbf{85.37} \\
        \thickhline
        \end{NiceTabular}}
    \vspace{0.2in}
    \caption{Evaluation metrics on the test datasets (mean $\pm$ standard deviation) for sulcal depth registration of both hemispheres from NKI dataset and ADNI dataset. The best two methods are bolded. The average training time for each epoch is also included in the bottom row of each scenario. We would like to note that as the number of samples is 40,962, using spherical optimal transport or even Sinkhorn divergence is computationally intractable in both time and space.}
    \label{tab: cortical-reg}
\end{table}

\begin{figure}[t!]
    \centering
    \vspace{-0.2in}
    \includegraphics[width=\linewidth]{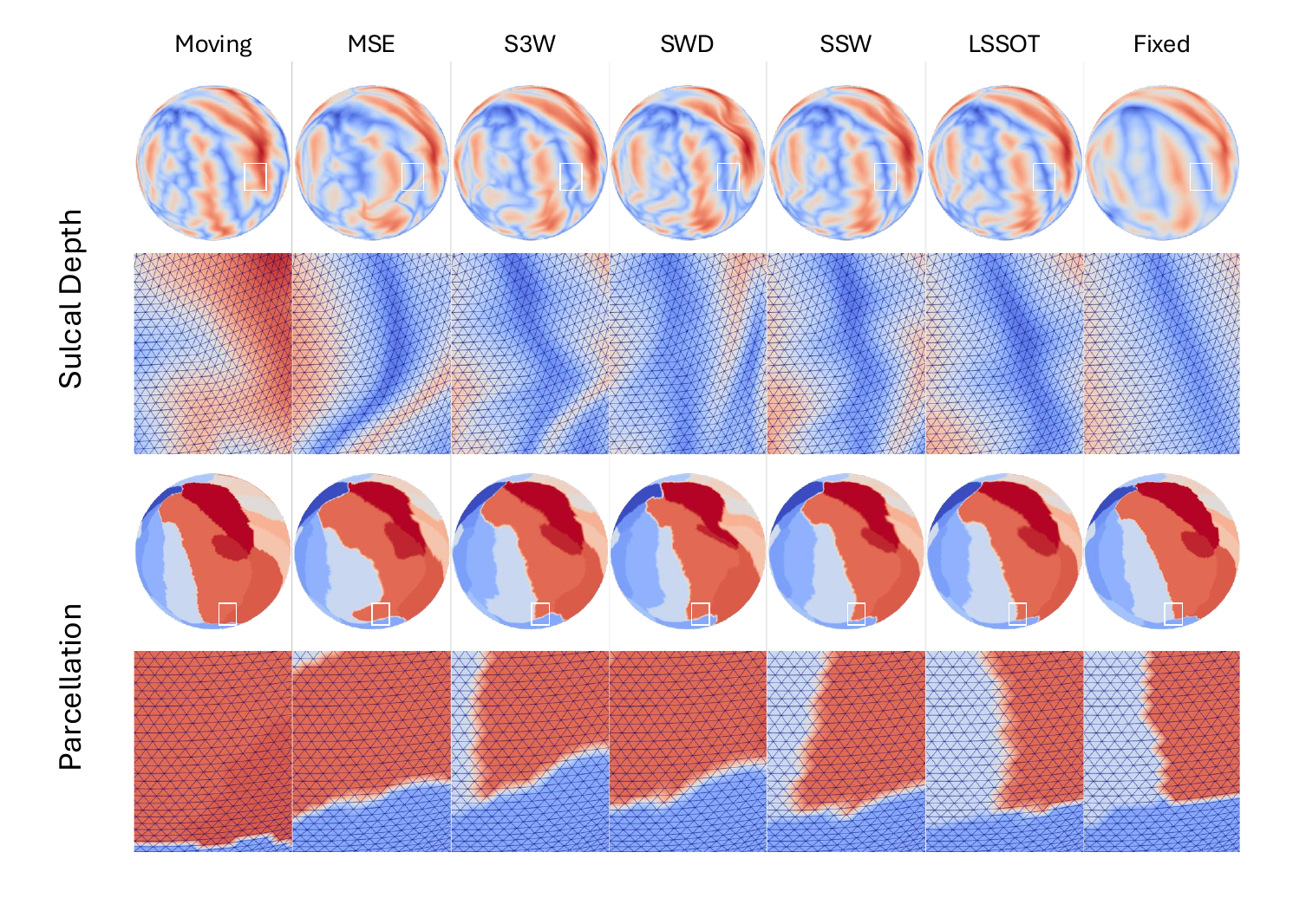}
    \vspace{-0.5in}
    \caption{Qualitative registration results (middle columns) from a moving surface (left column) to the fixed surface (right column). Both sulcal depth and parcellations are visualized with global and close-up views. This moving surface is from Subject A00065749 in the NKI dataset. Visualizations for more subjects can be found in the Appendix \ref{app: visual-cortical-reg}.}
    \label{fig:nki-A00065749}
\end{figure}

\subsection{Point Cloud Interpolation}
LSSOT can also be utilized in point cloud analysis, once each point cloud is endowed with a spherical representation. In this experiment, we explore the interpolations between point cloud pairs from the ModelNet dataset \citep{Wu_2015_CVPR}. Specifically, we train an autoencoder to project the original point clouds to a spherical latent space, to represent each of them as a spherical distribution. Then we apply gradient flow between the pairs of spherical distributions using the LSSOT metric along with SSW and Spherical OT. Finally, this transformation is reconstructed in the original space by the trained decoder, resulting in an interpolation between the original pairs. Figure \ref{fig:modelnet_interpolation} (top) illustrates this process.
\subsubsection{Mapping Point Clouds to Spherical Latent Space via LSSOT Autoencoder}
As point clouds are set-structured data, we leverage the Set Transformer \citep{lee2019set} architecture to build the autoencoder\footnote{https://github.com/juho-lee/set\_transformer}. The encoder $\Phi_{\text{enc}}$ is composed of three consecutive Set Attention Blocks, followed by a linear layer and a normalization layer to output a set of points on $\mathbb{S}^2$. The decoder $\Psi_{\text{dec}}$ consists of four Set Attention Blocks with a linear layer at the end. For a set of point clouds $\{X_i\}_{i=1}^I$ (each point cloud $X_i$ is a set of points), the autoencoder is trained to reconstruct each $X_i$ under a soft regularity constraint in the latent space. That is, the objective function is defined as 
\begin{align}
    \label{eq: autoencoder loss}
    \mathcal{L} = \frac{1}{I}\sum_{i=1}^I\Bigl(\|\Psi_{\text{dec}}(\Phi_{\text{enc}}(X_i))-X_i\|^2_2+\lambda\cdot LSSOT_2 (\Phi_{\text{enc}}(X_i), X_{\text{unif}})\Bigr)
\end{align}
where $X_{\text{unif}}$ is a set of samples from the uniform distribution on $\mathbb{S}^2$. 

After training, the encoder $\Phi_{\text{enc}}$ can serve as a spherical mapping from $\mathbb{R}^3$ to $\mathbb{S}^2$, and the decoder $\Psi_{\text{dec}}$ acts as an inverse mapping. Enforcing the latent representation to closely approximate the uniform distribution offers significant benefits. It minimizes the occurrence of singularities in both the spherical mapping and its inverse function, resulting in fewer distortions or discontinuities, hence the structural and geometrical characteristics of the original point clouds are better preserved. In implementation, we select $\lambda=1e-4$ through hyperparameter search. We set the number of slices to be $L=1000$, and the reference size as $M=1024$ for the $LSSOT_2$ function.

\subsubsection{Gradient Flow and Interpolations of Spherical Representations}
Given two spherical distributions as the source and the target, we can attach gradients on the source distribution and drive it toward the target by minimizing the distance between them. For the distance, we use LSSOT, SSW, and Spherical OT for comparisons. Here the Riemannian Gradient flow algorithm (see Appendix \ref{subsec: gradient-flow algo}) is implemented on $\mathbb{S}^2$. With a slight abuse of notation to mimic the constant speed geodesic, we sample 5 time points denoted as $T=0, 1/4, 2/4, 3/4, 1$, from the beginning ($T=0$) to the loss converging ($T=1$), such that the loss goes down by approximately the same amount in each interval. In Figure \ref{fig:modelnet_interpolation} (bottom), we visualize the interpolations for two example pairs of point clouds. 
\begin{figure}[t!]
    \vspace{-0.2in}
    \centering
    \includegraphics[width=\linewidth]{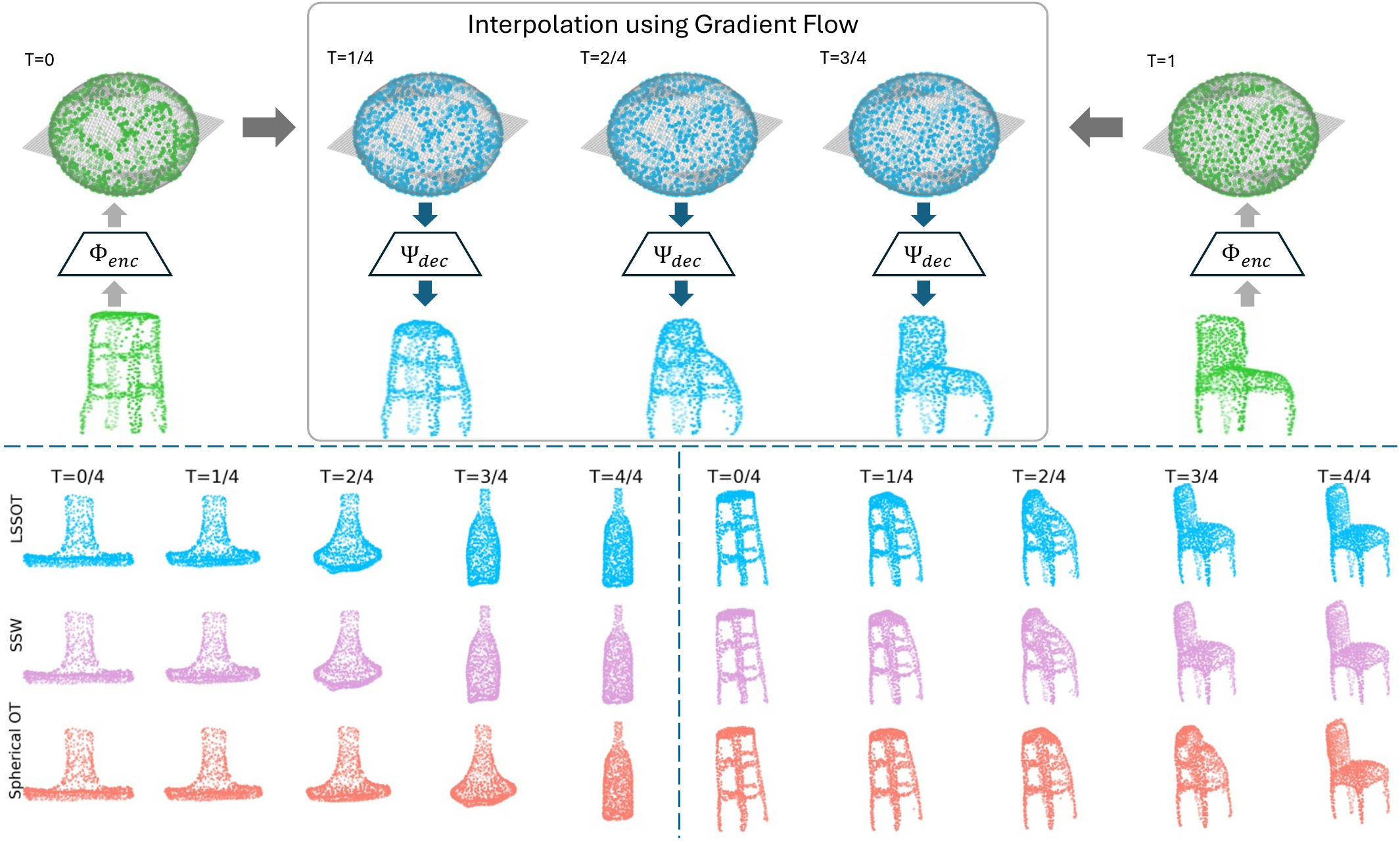}
    \caption{Top panel: the interpolation process between a pair of point clouds using gradient flow in the latent space $\mathbb{S}^2$. Bottom panel: gradient flow interpolations from a range hood to a bottle (left), and from a stool to a chair (right) using three metrics LSSOT, SSW and spherical OT. See Appendix \ref{subsec: pc_interp} for more interpolated pairs.}
    \label{fig:modelnet_interpolation}
\end{figure}

\section{Limitations and Future Work}
\vspace{-0.15in}
To the best of our knowledge, this work is the first to adapt the Linear Optimal Transport methodology onto hyperspheres, enabling rapid comparisons among a group of spherical distributions. However, spherical data/signals may not come in the form of distributions in practice, and normalizing data into probability distributions may result in the loss of information, thus one limitation of our work is that LSSOT is defined only for probability measures. Immediate future work includes a partial extension of LSSOT for positive signals with unequal mass, and transport $L^p$ \cite{thorpe2017transportation, liu2023ptlp,crook2024linear} extension for non-positive spherical signals, both of which will be studied on $\mathbb{S}^1$ as the first step, and extended to $\mathbb{S}^2$ by slicing.

\section{Conclusion}
\vspace{-0.15in}
In this work, we propose Linear Spherical Sliced Optimal Transport (LSSOT) as a novel and fast metric for defining the similarities between spherical data. We provide rigorous mathematical proofs establishing LSSOT as a true metric. We also demonstrate the computation efficiency of the LSSOT, superior to other baselines in the setting of pairwise comparisons of spherical data. We verify that LSSOT captures the rotation and scaling geometry on spheres through experiments on the Von Mises-Fisher distributions. In the task of atlas-based cortical surface registration on various datasets, LSSOT delivers top registration performance across multiple evaluation metrics while reducing the training time significantly. Furthermore, LSSOT is shown to yield meaningful interpolations between point cloud pairs. The metric's effectiveness in both point cloud analysis and practical neuroimaging applications demonstrates its broad potential in 3D data analysis.

\section{Acknowledgement}
SK acknowledges support from NSF CAREER Award \#2339898, CC acknowledges funding from NIH Grant P50MH109429, and BAL acknowledges funding from NIH 1R01EB017230. The authors also thank Shiyu Wang and Karthik Ramdas for processing the NKI and ADNI datasets and providing the processed files to the research team.

Data collection and sharing for this project was funded (in part) by the Alzheimer's Disease Neuroimaging Initiative (ADNI) (National Institutes of Health Grant U01 AG024904),
DOD ADNI (Department of Defense award number W81XWH-12-2-0012). ADNI is funded by the National Institute on Aging, the National Institute of Biomedical Imaging and Bioengineering, and through generous contributions from the following: AbbVie, Alzheimer’s Association; Alzheimer’s Drug Discovery Foundation; Araclon Biotech; BioClinica, Inc.; Biogen; Bristol-Myers Squibb Company; CereSpir, Inc.; Cogstate; Eisai Inc.; Elan Pharmaceuticals, Inc.; Eli Lilly and Company; EuroImmun; F. Hoffmann-La Roche Ltd and its affiliated company Genentech, Inc.; Fujirebio; GE Healthcare; IXICO Ltd.; Janssen Alzheimer Immunotherapy Research \& Development, LLC.; Johnson \& Johnson Pharmaceutical Research \& Development LLC.; Lumosity; Lundbeck; Merck \& Co., Inc.;
Meso Scale Diagnostics, LLC.; NeuroRx Research; Neurotrack Technologies; Novartis Pharmaceuticals Corporation; Pfizer Inc.; Piramal Imaging; Servier; Takeda Pharmaceutical Company; and Transition Therapeutics. The Canadian Institutes of Health Research is providing funds to support ADNI clinical sites in Canada. Private sector contributions are facilitated by the Foundation for the National Institutes of Health (\url{www.fnih.org}). The grantee organization is the Northern California Institute for Research and Education, and the study is coordinated by the Alzheimer’s Therapeutic Research Institute at the University of Southern California. ADNI data are disseminated by the Laboratory for Neuro Imaging at the University of Southern California.

\bibliography{iclr2025_conference.bib}
\bibliographystyle{iclr2025_conference}

\newpage
\appendix

\section{Spherical Radon transform in $\mathbb{S}^2$}\label{app: Radon S2}

In \cite{bonet2022spherical}, a \textbf{Spherical Radon transform} is introduced and it is used to give an equivalent formula of the SSW discrepancy \eqref{eq: SSW}.   
 In this section, we will discuss such Spherical Radon transform and its relation with the \textbf{Semicircle Transform}. 

\subsection{Notations on Sphere $\mathbb{S}^2$}
Let $U\in V_2(\mathbb{R}^3)$, where 
$$V_2(\mathbb{R}^3)=\{U\in \mathbb{R}^{3\times 2}: \,  U^TU=I_2\}.$$
Each $U$ determines a $2$-dimensional plane spanned by the columns of $U$: 
$$H(U):=\text{span}\{U[:,1], U[:,2]\}\subset \mathbb{R}^3.$$
The set $\{U[:,1], U[:,2]\}$ is an orthonormal basis for $H(U)$. The subspace $H(U)$ can be equivalently defined by a normal unit vector
$n_U$, which is the uniquely  determined by $U$ (up to scalar $\pm1$): 
\begin{align}
  n_U\in \mathbb{S}^2, \qquad  U^Tn_U=0_2. 
  \label{eq:n_U}  
\end{align}
Thus, the plane $H(U)$ can be defined as 
$$H(U)=n_U^\perp:=\{x\in \mathbb{R}^3: \,  x^Tn_U=0\}.$$ 

Note that $\mathcal{S}(U):=H(U)\cap \mathbb{S}^2$ 
defines a great circle. 

\subsection{Projection onto the great circle $\mathcal{S}(U)$}
Given $U\in V_2(\mathbb{R}^3)$, the operator $UU^T:\mathbb{R}^3\to\mathbb{R}^3$ defines a projection. Indeed, for each point $s\in \mathbb{S}^2\setminus\{\pm n_U\}$, the mapping 
\begin{align}
  s\mapsto \mathcal{P}^1_{\mathcal{S}(U)}(s):= \frac{UU^Ts}{\|UU^Ts\|}=\frac{UU^Ts}{\|U^Ts\|}=\frac{U[:,1]^Ts}{\|U^Ts\|}U[:,1]+\frac{U[:,2]^Ts}{\|U^Ts\|}U[:,2] \in\mathcal{S}(U)\subset\mathbb{R}^3\label{eq:proj_1}.
\end{align}
defines a projection onto the great circle $\mathcal{S}(U)$, where in \eqref{eq:proj_1} we used that
$$\left\|U U^T s\right\|=s^T U U^T U U^T s=s^T U U^T s=\left\|U^T s\right\|$$
(see \cite[Lemma 1]{bonet2022spherical}). 
Note that in \cite{bonet2022spherical}, the above projection is written as 
\begin{align}
    s\mapsto P^Us:=\frac{U^Ts}{\|U^Ts\|} \label{eq:P^u}, 
\end{align}
which is a map from $\mathbb{S}^2\setminus\{\pm n_U\}$ to $\mathbb{R}^2$ that saves the coordinates $[U[:,1]^Ts,U[:,2]^Ts]$.
This function maps $s$ into a 1-dimensional circle $\approx\mathbb{S}^1$, which can be regarded as a representation/parametrization of the great circle $\mathcal{S}(U)$.

An equivalent method that allows us to define such a projection is introduced by the following parameterization system \cite{quellmalz2023sliced}:

Each $s\in \mathbb{S}^2$ can be written in spherical coordinates: 
$$s=[\cos \alpha\sin\beta,\sin\alpha\sin\beta,\cos\beta]^T.$$
where $\alpha\in[0,2\pi)$ is called the azimuth angle of $s$, and  $\beta\in [0,\pi]$ the zenith angle of $s$. We will use the notation $\alpha=azi(s)$, $\beta=zen(s)$.
\begin{remark}
When $s=\pm [0,0,1]$, the above parametrization is not unique. As in \cite{quellmalz2023sliced}, we set $\alpha=0,\beta=0$ for $s=[0,0,1]$ and 
$\alpha=0,\beta=\pi$ for $s=[0,0,-1]$. 
Thus, 
$$s\mapsto (\alpha,\beta)$$ becomes a bijection (see \cite[page 6]{quellmalz2023sliced}).
\end{remark}
Let us define the following mappings: 
\begin{align}
[0,2\pi)\times[0,\pi]\ni (\alpha,\beta)&\mapsto \Phi(\alpha,\beta):=[\cos \alpha\sin\beta,\sin\alpha\sin\beta,\cos\beta]^T=:s
\in \mathbb{S}^2\label{eq:phi}  \\
\mathbb{S}^2\ni s &\mapsto azi(s)=:\alpha 
\in [0,2\pi) \label{eq:azi}\\
\mathbb{S}^2\ni s &\mapsto zen(s)=:\beta 
\in [0,\pi] \label{eq:zen}
\end{align}

In addition, 
we define the following matrix: 
\begin{align}
\Psi(\alpha,\beta,\gamma):=\begin{bmatrix}
\cos \alpha & -\sin \alpha & 0 \\
\sin \alpha & \cos \alpha & 0 \\ 
0           & 0           & 1 
\end{bmatrix}\begin{bmatrix}
\cos \beta & 0 & \sin\beta \\
0          & 1 & 0 \\ 
-\sin\beta & 0 & \cos\beta 
\end{bmatrix}\begin{bmatrix}
\cos \gamma & -\sin \gamma & 0 \\
\sin \gamma & \cos \gamma & 0 \\ 
0           & 0           & 1 
\end{bmatrix}\nonumber
\end{align}
By \cite[page 7]{quellmalz2023sliced} and \cite{graf2009sampling}, 
\begin{align}
  \mathbb{S}^2\times [0,2\pi)\ni (s,\gamma)\mapsto \Psi({azi}(s),{zen}(s),\gamma)\in SO(3)  \label{eq:bijection_Psi}  
\end{align}
is a bijection, where $SO(3)$ is the special orthogonal group in 
$\mathbb{R}^3$:
\begin{align}
  SO(3)=\{G\in \mathbb{R}^{3\times 3}: \,  GG^T=G^TG=I_3, \, \text{det}(G)=1\}.\nonumber 
\end{align}

In the special case $\gamma=0$, we adopt the following notation:
\begin{align}
\Psi(s):=\Psi(\alpha,\beta,0)&=\begin{bmatrix}
\cos\alpha\cos\beta & -\sin\alpha & \cos\alpha\sin\beta \\
\sin\alpha\cos\beta &\cos\alpha  & \sin\alpha\sin\beta\\
-\sin\beta & 0 & \cos\beta 
\end{bmatrix}    \nonumber 
\end{align}
where, given $s\in\mathbb{S}^2$, $\alpha$ and $\beta$ are as in \eqref{eq:azi} and \eqref{eq:zen} (or equivalently, given $(\alpha,\beta)$, $s\in\mathbb{S}^2$ is defined by \eqref{eq:phi}).

It is straightforward to verify 
$$\Psi(s)^Ts=\mathfrak{n}:=[0,0,1]^T.$$
That is, we can regard $\Psi(s)^T$ as a rotation matrix that can rotate $s$ back to the North Pole $\mathfrak{n}=[0,0,1]^T$. Equivalently, $\Psi(s)$ is a rotation matrix, which can rotate $[0,0,1]^T$ to $s$.

Define the Equator by 
$$\mathcal{E}:=\{[\cos\gamma,\sin\gamma,0]: \, \gamma\in[0,2\pi)\}.$$
Given $U\in V_2(\mathbb{R}^3)$, consider the following set
$$\{s\in \mathbb{S}^2: \Psi(n_U)^Ts\in \mathcal{E}\}$$
to obtain the following characterization of great circles for $\mathbb{S}^2\subset\mathbb{R}^3$:
\begin{lemma}\label{lem:S(U)}
Given $U\in V_2(\mathbb{R}^3)$, it defines a great circle $\mathcal{S}(U)=H(U)\cap \mathbb{S}^2$ which is characterized by the following identity: 
\begin{equation}\label{eq:SU charact}
  \mathcal{S}(U)=\{s\in \mathbb{S}^2: \Psi(n_U)^Ts\in \mathcal{E}\}.
\end{equation}
\end{lemma}
\begin{proof}
First, since $\Psi(n_U)$ is a rotation matrix, $\Psi(n_U)^Ts\in \mathbb{S}^2$ for all $s\in\mathbb{S}^2$, as $\|\Psi(n_U)^Ts\|=\|s\|=1$. 

Now, write $n_U=[\cos\alpha^*\sin\beta^*,\sin\alpha^*\sin\beta^*,\cos\beta^*]^T$, where $\alpha^*:=azi(n_U)$ and $\beta^*=zen(n_U)$. 

For any $s\in \mathcal{S}(U)$, we have $(n_U)^Ts=0$. 

We have to prove that the third coordinate of $\Psi(n_U)^Ts\in\mathbb{S}^2$ is $0$:
\begin{align}
[\Psi(n_U)^Ts][3] &=[\cos\alpha^*\sin\beta^*, \sin\alpha^*\sin\beta^*, \cos \beta^*]s=(n_U)^Ts=0.\nonumber 
\end{align}
That is, for every  $s\in \mathcal{S}(U)$, we have obtained that $s\in \{s\in \mathbb{S}^2: \Psi(n_U)^Ts\in \mathcal{E}\}$.

Similarly, if we consider $s\in \{s\in \mathbb{S}^2: \Psi(n_U)^Ts\in \mathcal{E}\}$, we have 
$0=[\Psi(n_U)^Ts][3]=(n_U)^Ts$. Thus, $s\in \mathcal{S}(U)$. 
 
 Therefore, we have the equality in \eqref{eq:SU charact}.
\end{proof}

\begin{lemma}\label{lem:gamma-S(U)}
The mapping 
$$[0,2\pi)\ni\gamma\mapsto \Phi(\gamma,\pi/2)=[\cos\gamma,\sin\gamma,0]^T\in \mathcal{E} $$
and the mapping 
$$\mathcal{E}\ni s\mapsto \Psi(n_U)s \in \mathcal{S}(U) $$
are bijections. 
Thus, the mapping 
$$[0,2\pi)\ni \gamma\mapsto s:=\Psi(n_U)\Phi(\gamma,\pi/2)\in \mathcal{S}(U)$$
is also a bijection. 
\end{lemma}
\begin{proof}
It is trivial to verify that the first mapping, $\gamma\mapsto [\cos\gamma,\sin\gamma,0]^T$,
is a bijection. 

By the Lemma \ref{lem:S(U)}, the mapping $s\mapsto \Psi(n_U)^Ts$ is well-defined from $\mathcal{S}(U)$ to $\mathcal{E}$. 
Since $\Psi(n_U)^T$ is a rotation matrix, it is bijection $\mathbb{S}^2\to \mathbb{S}^2$ whose inverse is $\Psi(n_U)$, and so  
the mapping $s\mapsto \Psi(n_U)s$ is a well-defined bijection from $\mathcal{E}$ to $\mathcal{S}(U)$: that is, if we restrict $\Psi(n_U)$ to $\mathcal{E}$, by Lemma \ref{lem:S(U)}, the matrix $\Psi(n_U)$ defines a bijection from $\mathcal{E}$ to $\text{Range}(\Psi(n_U)\mathcal{E})=\mathcal{S}(U)$. 
\end{proof}

Now, we define the following composition map $\mathcal{P}^2_{\mathcal{S}(U)}(s)$:
\begin{align}
\mathbb{S}^2\setminus\{\pm n_U\}\ni s\mapsto \gamma:={azi}(\Psi(n_U)^Ts)\mapsto \mathcal{P}^2_{\mathcal{S}(U)}(s):=  \Psi(n_U)\Phi(\gamma,\pi/2)\in \mathcal{S}(U)\label{eq:proj_2} 
\end{align}
\begin{proposition}\label{pro:proj_12}
The mappings $\mathcal{P}^1_{\mathcal{S}(U)}$ (defined by \ref{eq:proj_1}) and $\mathcal{P}^2_{\mathcal{S}(U)}$ (defined by \ref{eq:proj_2}) coincide on $\mathbb{S}^2\setminus \{\pm n_U\}$. 
\end{proposition}
\begin{proof}
Case 1: Suppose $n_{U}=[0,0,1]^T$. Consider $s=[\cos\alpha\sin\beta,\sin\alpha\sin\beta,\cos\beta]^T\in \mathbb{S}^2\setminus\{\pm n_U\}$, i.e., where $\alpha=azi(s)$ and $\beta=zen(s)$ with $\beta\in(0, \pi)$.
In this 
case, from \eqref{eq:n_U}, $U[3,:]=0_2$. Thus, $\{U[1:2,1],U[1:2,2]\}$ forms an orthornomal basis for $\mathbb{R}^2$. 
From \eqref{eq:proj_1}, we have
\begin{align}
UU^Ts=\begin{bmatrix}
    I_2 & 0 \\
    0   & 0 
\end{bmatrix} s=[\cos\alpha\sin\beta,\sin\alpha\sin\beta,0]^T\nonumber
\end{align}
Thus, 
$$\mathcal{P}^1_{\mathcal{S}(U)}(s)=\frac{UU^Ts}{\|U^Ts\|}=[\cos\alpha,\sin\alpha,0]^T.$$
In addition, in this case, $\Psi(n_U)=I_3$. Thus, $\gamma:=azi(\Psi(n_U)^Ts)=azi(s)=\alpha$. Therefore,
$$\mathcal{P}^2_{\mathcal{S}(U)}(s)=\Psi(n_U)\Phi(\gamma,\pi/2)=\Phi(\gamma,0)=[\cos\alpha,\sin\alpha,\pi/2]^T.$$
Thus, \ref{eq:proj_1} and \ref{eq:proj_2} coincide in this case. 

Case 2: Suppose $n_U\neq [0,0,1]$.
Each point $s\in \mathbb{S}^2$ can be written as $s=\Psi(n_U)s'$ where $s'=\Psi(n_U)^Ts$. In addition, when $s\neq \pm n_U$, we can parametrize $s'$ via 
$s'=[\cos\alpha\sin\beta,\sin\alpha\sin\beta,\cos\beta]^T$, where $\beta\in(0,\pi)$.  
We also have that $\Psi(n_U)^{T}n_U=[0,0,1]^T$. 

Let $U'=\Psi(n_U)^TU$. We have $$(U')^T[0,0,1]^T=U^T\Psi(n_U)[0,0,1]^T=U^Tn_U=0_2.$$ Thus $n_{U'}=[0,0,1]^T$.
Besides, we have that
\begin{align}
UU^Ts=\Psi(n_U)U' (U')^T\Psi(n_U)^T \Psi(n_U)s'=\Psi(n_U)U'(U')^Ts'. \nonumber 
\end{align} 
Thus, 
$$\mathcal{P}^1_{\mathcal{S}(U)}(s)=\Psi(n_U) \mathcal{P}^1_{\mathcal{S}(U')}(s'),\qquad \forall s\neq \pm n_U. $$
In addition, by using the parametrization of $s'$ in spherical coordinates, we have 
\begin{align}
\gamma:=azi(\Psi(n_U)^Ts)=azi( s')=\alpha .\nonumber  
\end{align}
Thus, \begin{align}
\mathcal{P}^2_{\mathcal{S}(U)}(s)=\Psi(n_U)[\cos\alpha,\sin\alpha,0]^T=\Psi(n_U)\mathcal{P}^2_{\mathcal{S}(U')}(s').\nonumber   
\end{align}
From Case 1, we have 
$$\mathcal{P}^1_{\mathcal{S}(U')}(s')=\mathcal{P}^2_{\mathcal{S}(U')}(s'),$$ and from the the fact that $s'\mapsto s=\Psi(n_U)s'$ is a bijection on $\mathbb{S}^2$, we get 
$$\mathcal{P}^1_{\mathcal{S}(U)}(s)=\mathcal{P}^2_{\mathcal{S}(U)}(s).$$
\end{proof}

\begin{lemma}\label{lem:trans_param_1}
Given $z\in \mathbb{S}^1,U\in V_2(\mathbb{R}^3)$, there exists 
$\alpha^*\in [0,2\pi),\beta^*\in[0,\pi],\gamma\in[0,2\pi)$ such that 
\begin{align}
n_U&=\Phi(\alpha^*,\beta^*)=[\cos\alpha^*\sin\beta^*,\sin\alpha^*\sin\beta^*,\cos\beta^*]^T \label{eq:params_1}\\
Uz&=\Psi(n_U)[\cos \gamma,\sin\gamma,0]^T\label{eq:params_2} 
\end{align}
Vice-versa:  Given $\alpha^*\in [0,2\pi),\beta^*\in[0,\pi],\gamma\in[0,2\pi)$, there exist $z\in \mathbb{S}^2,U\in V_2(\mathbb{R}^3)$ for which formulas \ref{eq:params_1} and \ref{eq:params_2} hold. 
\end{lemma}

\begin{proof}
Given $U\in V_2(\mathbb{R}^3)$, one can find  $n_U\in \mathbb{S}^2$ via \eqref{eq:n_U}. (Note that if $n_U$ satisfies \eqref{eq:n_U}, then  $-n_U$ also satisfies \eqref{eq:n_U}, and so $n_U$ is unique up to sign.)
Then, once $n_U$ is chosen, 
we can uniquely find $(\alpha^*,\beta^*)$ as \eqref{eq:azi}, \eqref{eq:zen}.

In addition, for $z\in \mathbb{S}^1\subset\mathbb{R}^2$, since
$$\mathcal{P}^1_{\mathcal{S}(U)}(Uz)=\frac{UU^TUz}{\|U^TUz\|}=\frac{Uz}{\|z\|}=Uz,$$
we obtain that $Uz\in \mathcal{S}(U).$
By Lemma \ref{lem:gamma-S(U)},  
$\Psi(n_U)^TUz\in \mathcal{E}$. Then, we can uniquely determine $\gamma$ (see Lemma \ref{lem:gamma-S(U)}).
Thus, \eqref{eq:params_1}, and \eqref{eq:params_2} are satisfied. 

For the other direction, consider angles $\alpha^*,\beta^*,\gamma$ and let us separate the analysis into the following cases:

Case 1: $\alpha^*=0,\beta^*\in \{0,\pi\}$. In this case, we have $n_U=[0,0,\pm 1]$ and thus we can set 
$$U:=\begin{bmatrix}
    I_2  \\
    0   
\end{bmatrix}.$$ 
In addition, set $z:=[\cos \gamma,\sin\gamma]^T\in\mathbb{S}^1$ and so the identities \ref{eq:params_1}, and \ref{eq:params_2} are satisfied. 

Case 2: If $\alpha^*\neq 0$, then we set $n_U$ by \ref{eq:params_1}. Choose an orthornormal basis $\{u_1,u_2\}$ for $n_U^\perp=\{x\in \mathbb{R}^3: \,  x^Tn_U=0\}$, and set $U:=[u_1,u_2]$. Then $(U,n_U)$ satisfies \ref{eq:n_U}. 

By Lemma \ref{lem:gamma-S(U)}, we have $\Psi(n_U)[\cos\gamma,\sin\gamma,0]^T\in \mathcal{S}(U)$. Then, $$UU^T\Psi(n_U)[\cos\gamma,\sin\gamma,0]^T\in \mathcal{S}(U),$$ and  by setting 
$$z:=U^T\Psi(n_U)[\cos\gamma,\sin\gamma,0]^T,$$ 
we have 
\begin{align}
Uz&=UU^T\Psi(n_U)[\cos\gamma,\sin\gamma,0]^T=\frac{UU^T\Psi(n_U)[\cos\gamma,\sin\gamma,0]^T}{\|U^T\Psi(n_U)[\cos\gamma,\sin\gamma,0]^T\|}\nonumber\\
&=\mathcal{P}^1_{\mathcal{S}(U)}(\Psi(n_U)[\cos\gamma,\sin\gamma,0]^T)=\Psi(n_U)[\cos\gamma,\sin\gamma,0]^T\nonumber 
\end{align}
Thus, \ref{eq:params_2} is satisfied. 
\end{proof}

Now, given $\gamma\in[0,2\pi)$, and $n_U=\Phi(\alpha^*,\beta^*)\in \mathbb{S}^2$, we can define the following \textbf{semicircle}: 
\begin{align}
\mathcal{SC}^\gamma_{n_U}:=\mathcal{SC}^\gamma_{\alpha^*,\beta^*}
:=\{s\in\mathbb{S}^2: \, {azi}(\Psi(n_U)^Ts)=\gamma\}\label{eq:SC(U,gamma)}. 
\end{align}
\begin{remark}\label{rm:SC}
When $\gamma\neq 0$, we can replace the above definition by 
\begin{align}
\mathcal{SC}^\gamma_{n_U}:=\mathcal{SC}^\gamma_{\alpha^*,\beta^*}
:=\{s\in\mathbb{S}^2\setminus\{\pm n_U\}: \, {azi}(\Psi(n_U)^Ts)=\gamma\}\label{eq:SC(U,gamma),2}. 
\end{align}
\end{remark}

\begin{proposition}\label{pro:SC}
Let $z\in \mathbb{S}^1,U\in V_2(\mathbb{R}^3)$. Consider 
$\alpha^*\in[0,2\pi),\beta^*\in[0,\pi],\gamma\in[0,2\pi)$ as in Lemma \ref{lem:trans_param_1} (or equivalently, first, fix $(\alpha^*,\beta^*)\in ((0,2\pi)\times [0,\pi])\cup \{0\}\times\{0,\pi\}),
\gamma\in[0,2\pi)$, and then, find $z,U$ according to Lemma \ref{lem:trans_param_1}). We have: 
$$\mathcal{SC}_{n_U}^{\gamma}\setminus\{\pm n_U\}=\mathcal{SC}_{\alpha^*,\beta^*}^\gamma\setminus\{\pm n_U\}=\{s\in \mathbb{S}^2\setminus\{\pm n_U\}: \, P^Us=z\}.$$
\end{proposition}
\begin{proof}
We have that
\begin{align}
&\{s\in \mathbb{S}^2\setminus\{\pm n_U\}: \, P^Us=z\}=\left\{s\in \mathbb{S}^2\setminus\{\pm n_U\}: \,  \frac{UU^Ts}{\|U^Ts\|}=Uz\right\} \notag\\
&=\left\{s\in \mathbb{S}^2\setminus\{\pm n_U\}: \,  \mathcal{P}^1_{\mathcal{S}(U)}(s)=Uz\right\}\nonumber\\
&=\left\{s\in \mathbb{S}^2\setminus\{\pm n_U\}: \mathcal{P}^2_{\mathcal{S}(U)}(s)=\Psi(n_U)[\cos\gamma,\sin\gamma,0]^T
\right\}\quad \text{(by Proposition \ref{pro:proj_12} and \eqref{eq:params_2})}\nonumber \\
&=\left\{s\in \mathbb{S}^2\setminus\{\pm n_U\}:\Psi(n_U)\Phi({azi}(\Psi(n_U)^Ts),\pi/2)=\Psi(n_U)[\cos\gamma,\sin\gamma,0]^T\right\}\quad \text{(by  definition \ref{eq:proj_2})}\nonumber \\
&=\left\{s\in \mathbb{S}^2\setminus\{\pm n_U\}:\Phi({azi}(\Psi(n_U)^Ts),\pi/2)=[\cos\gamma,\sin\gamma,0]^T\right\}\quad \text{(applying $\Psi^T(n_U)$ on both sides)}\nonumber \\
&=\left\{s\in \mathbb{S}^2\setminus\{\pm n_U\}: {azi}(\Psi(n_U)^Ts)=\gamma\right\}\qquad\text{(by Lemma \ref{lem:gamma-S(U)})}\nonumber\\
&=\mathcal{SC}_{n_U}^{\gamma}\setminus\{\pm n_U\}=\mathcal{SC}_{\alpha^*,\beta^*}^\gamma\setminus\{\pm n_U\}.\nonumber 
\end{align}
\end{proof}

\begin{remark}
With the notation of the above proposition. For each fixed pair $(z,U)\in\mathbb{S}^2\times V_2(\mathbb{R}^3)$, if we choose distinct $(\alpha^*,\beta^*,\gamma)$ and $(\tilde\alpha^*,\tilde\beta^*,\tilde\gamma)$ such that the conditions in Lemma \ref{lem:trans_param_1} are satisfied, we obtain 
$$\mathcal{SC}_{\alpha^*,\beta^*}^\gamma=\mathcal{SC}_{\tilde\alpha^*,\tilde\beta^*}^{\tilde\gamma}.$$
\end{remark}

\subsection{Semicircle Transform and Spherical Radon transform}
Let $f\in L^1(\mathbb{S}^2)$. The \textbf{Normalized Semicircle Transform} of $f$ is a function $\mathbb{S}^2\to \mathbb{R}$ defined as: 
\begin{align}
\mathcal{W}f(\alpha^*,\beta^*,\gamma):=\frac{1}{4\pi}\int_{\mathcal{SC}^\gamma_{n_U}}f(s)\sin({zen}(s))d \mathcal{SC}_{n_U}^\gamma\label{eq:Wf_n}
\end{align}
where $d \mathcal{SC}^\gamma_{n_U}$ is the curve differential on the semicircle $\mathcal{SC}^\gamma_{n_U}$. 

Its \textbf{un-normalized version} is defined as 
\begin{align}
\mathcal{UW}f(\alpha^*,\beta^*,\gamma):=\frac{1}{4\pi}\int_{\mathcal{SC}^\gamma_{n_U}}f(s)d\mathcal{SC}_{n_U}^\gamma\label{eq:Wf_un}
\end{align}

The \textbf{Spherical Radon Transform} of $f\in L^1(\mathbb{S}^{d-1})$, for $d\geq 2$, defined in \cite{bonet2022spherical} is the mapping $\mathcal{R}f: V_2(\mathbb{R}^d)\times \mathbb{S}^1\mapsto \mathbb{R}$ given by: 
\begin{align}
\mathcal{R}f(U,z)=\int_{\mathbb{S}^{d-1}}f(s)\delta(z=P^U(s))ds\label{eq:Rf}
\end{align}

When considering $d=3$, we obtain the following relation between the above-defined transforms.

\begin{proposition}
The Unbounded Semicircle Transform \ref{eq:Wf_un} and the Spherical Radon Transform \ref{eq:Rf} are equivalent. In particular, given $z\in \mathbb{S}^1,U\in V_2(\mathbb{R}^3)$, consider $\alpha^*\in [0,2\pi),\beta^*\in[0,\pi],\gamma\in[0,2\pi)$ defined as in Lemma \ref{lem:trans_param_1}, then
$$\mathcal{UW}f(\alpha^*,\beta^*,\gamma)=\frac{1}{4\pi}\mathcal{R}f(U,z).$$
\end{proposition}
\begin{proof}
From Proposition \ref{pro:SC}, we have 
$\mathcal{SC}_{\alpha^*,\beta^*}^\gamma=\{s:P^Us=z\},$ thus
\begin{align}
\mathcal{R}f(U,z)=\int_{\mathcal{SC}_{n_U}^\gamma}f(s) d\mathcal{SC}_{n_U}^\gamma=|4\pi| \mathcal{UW}(f)(\alpha^*,\beta^*,\gamma)\nonumber .
\end{align}
\end{proof}

\begin{remark}
    We notice that there are other versions of the Spherical Radon Transform. For example, in \cite{groemer1998spherical}
    the spherical Radon transformation
    is defined by integrating over the full great circles $\mathcal{S}(U)$
    $$\widetilde{\mathcal{R}}f(n_U)=\int_{\mathcal{S}(U)} f \, d\mathcal{S}(U).$$
    In this paper we are using the variant given in \cite{bonet2022spherical}, which we have proven coincides with the integral transformation over semicircles. That is, this version coincides with the integral transformation denoted by $\mathcal{B}$ in \cite{groemer1998spherical}. It holds that
    $$\widetilde{\mathcal{R}}(n_U)=\mathcal{R}f(U,z)+\mathcal{R}f(U,-z)$$
    We will review this in Appendix \ref{app: Radon general d} when generalizing from $d=3$ to higher dimensions.
\end{remark}

We end this section with a discussion about some alternative formulations for $\mathcal{W}f$ and $\mathcal{UW}f$.

For each $s\in \mathcal{SC}^{\gamma}_{\alpha^*,\beta^*}\setminus\{\pm[0,0,1]^T\}$, we have 
$${azi}(\Psi(\alpha^*,\beta^*,0)^Ts)=\gamma.$$ 
Thus, we can write $s$ as  
$$s=\Psi(\alpha^*,\beta^*,0)\Phi(\gamma,\xi)=\Psi(\alpha^*,\beta^*,\gamma)\Phi(0,\xi),$$ for some $\xi\in (0,\pi)$ (recall that $\Phi(0,\xi)=[\sin(\xi),0,\cos(\xi)]^T$, and see \cite{groemer1998spherical} for more details).

Similar to Remark \ref{rm:SC}, we can redefine the semicircle as 
\begin{align}
\mathcal{SC}_{\alpha,\beta}^\gamma=\{\Psi(\alpha,\beta,\gamma)\Phi(0,\xi):\xi\in (0,\pi)\}, \nonumber 
\end{align}
and $\Psi(\alpha^*,\beta^*,\gamma)$ is called \textbf{primal median} of  $\mathcal{SC}_{\alpha^*,\beta^*}^\gamma$. 

Combining this with the fact that $(\alpha^*,\beta^*,\gamma)\mapsto \Psi(\alpha^*,\beta^*,\gamma)$ is a bijection, we can rewrite \eqref{eq:Wf_n}, and \eqref{eq:Wf_un} as 
\begin{align}
\mathcal{W}f(Q)&=\frac{1}{4\pi}\int_0^\pi f(Q\Phi(0,\xi))\sin(\xi)d\xi \nonumber\\
\mathcal{UW}f(Q)&=\frac{1}{4\pi}\int_0^\pi f(Q\Phi(0,\xi))d\xi \nonumber
\end{align}
where $Q\in SO(3)$ (by using the identity $s=Q\Phi(0,\xi)$, for an appropriate $Q\in SO(3)$). 

In \cite{hielscher2018svd}, the \textbf{Un-normalized Semicircle Transform} is defined as 
\begin{align}
\widetilde{\mathcal{UW}}(f)(Q):=\int_{-\pi/2}^{\pi/2}f(Q^T\Phi(\xi,\pi/2))d\xi, \qquad Q\in SO(3). 
\end{align}
The equivalence relation between $\mathcal{UW}(f)$ and $\widetilde{\mathcal{UW}}(f)$ reads as follows: 
\begin{lemma}
For any $Q\in SO(3)$, we have 
\begin{align}
\mathcal{UW}f(Q)=\widetilde{\mathcal{UW}}f\left(\begin{bmatrix}
1 & 0 & 0 \\
0 & 0 & -1 \\
0 & 1 & 0 
\end{bmatrix}Q^T\right)\nonumber. 
\end{align}
\end{lemma}
\begin{proof}
Given $Q\in SO(3),\xi \in (0,\pi)$, let $Q'=\begin{bmatrix}
1 & 0 & 0 \\
0 & 0 & -1 \\
0 & 1 & 0 
\end{bmatrix}Q^T,\xi'=\xi-\beta$. 
Then, 
\begin{align}
Q'[1,:]&=Q^T[1:]=Q[:,1] \nonumber \\
Q'[2,:]&=-Q^T[3,:]=-Q[:,3] \nonumber\\
Q'[3,:]&=Q^T[2,:]=Q[:,2]\nonumber 
\end{align}
and so
\begin{align}
Q\Phi(0,\xi)&=Q[\sin\xi,0,\cos\xi]^T\nonumber\\ 
&=Q[:,1]\sin\xi+Q[:,3]\cos\xi\nonumber\\ 
&=Q[:,1]\cos(\xi')-Q[:,3]\sin(\xi')\nonumber\\
&=Q'[1,:]\cos\xi'+Q'[2,:]\sin\xi' \nonumber\\
&=(Q')^T[\cos\xi',\sin\xi',0]^T\nonumber\\
&=(Q')^T\Phi(\xi',\pi/2)\nonumber
\end{align}
Thus, $\mathcal{UW}f(Q)=\widetilde{\mathcal{UW}}f(Q)$. 
\end{proof}

Another formulation has been introduced in \cite{groemer1998spherical}. Since it allows us to define $\mathcal{UW}f$ in higher dimensions ($d\ge 3$), we will discuss it in the next section.

\section{Spherical Radon Transform in $\mathbb{S}^{d-1}$}\label{app: Radon general d}
\subsection{Great circle and semicircle / $(d-2)$-dimensional hemisphere in $\mathbb{S}^{d-1}$}
We set $d\ge 3$ and 
consider 
$$V_2(\mathbb{R}^d):=\{U\in \mathbb{R}^{d\times 2}: \,  U^TU=I_2\}.$$

Similar to the previous section, we can define 2D \textbf{hyperplane} 
$$H(U):=\text{span}\{U[:,1],U[:,2]\},$$
and $U^\perp:=H(U)^\perp$. 

In addition, we define the corresponding \textbf{great circle}: 
$\mathcal{S}(U)=H(U)\cap \mathbb{S}^{d-1}$. 
Note that the Equator is a specific great circle: 
$$\mathcal{E}=\{s\in \mathbb{S}^{d-1}: s^T\mathfrak{n}=0\},$$
where $\mathfrak{n}=[0,0,\ldots,0,1]^T\in \mathbb{S}^{d-1}\subset\mathbb{R}^d$.  
Similar to the previous section, the corresponding projection onto the great circle is 
\begin{align}
\mathcal{P}^1_{\mathcal{S}(U)}(s)=\frac{UU^Ts}{\|U^Ts\|}, \qquad s\in \mathbb{S}^{d-1}\setminus U^T\label{eq:proj_1_d}
\end{align}
Similar to Proposition \ref{pro:SC} and \eqref{eq:SC(U,gamma)}, we define the $(d-2)$-dimensional \textbf{hemisphere}: 
\begin{align}
\mathcal{HS}(U,z):=\{s:P^Us=z\}\cup (U^\perp\cap \mathbb{S}^{d-1}) \label{eq:HS(U,z)}
\end{align}
\begin{remark}
When $d=3$, $U^\perp=\{\pm n_U\}$ and if we ignore these two points, \eqref{eq:HS(U,z)} is consistent to \eqref{eq:SC(U,gamma)}. 

By introducing the hemisphere, the projection \ref{eq:proj_1_d} can be regarded as the following: 
Let $s\in \mathbb{S}^d\setminus U^\perp$, we have: $s\in \mathcal{HS}(U,z)$ where $z=\frac{U^Ts}{\|U^Ts\|}$.
Then the function $\mathcal{P}^1_{\mathcal{S}(U)}$ maps $s$ to the intersection point $\mathcal{HS}(U,z)\cap \mathcal{SC}(U)$.

\end{remark}
\begin{proposition}\label{pro:S(U)}
The following identity holds for great circles: 
\begin{align}
 \mathcal{S}(U)=\mathcal{P}^1_{\mathcal{S}(U)}(\mathbb{S}^d\setminus U^\perp)=\{Uz: z\in \mathbb{S}^1\} \label{eq:S(U),d}
\end{align}
\end{proposition}
\begin{proof}
Since $\mathcal{P}^1_{\mathcal{S}(U)}$ is the projection, we have 
$\mathcal{P}^1_{\mathcal{S}(U)}(\mathbb{S}^d\setminus U^\perp)\subset \mathcal{S}(U)$. 

For the other direction, choose $s\in \mathcal{S}(U)$ and we decompose $s$ as 
$$s=s_U+s_{U^\perp},$$
where $s_U\in H(U), s_{U^\perp}\in U^\perp$. 

By definition of $\mathcal{S}(U)$, we have $s_{U^\perp}=0$. Thus $s=s_U$. From the fact $H(U)\cap U^\perp=\{0_d\}$, we have $s\notin U^\perp$. Therefore, $s\in \mathbb{S}^{d-1}\setminus U^\perp$. Thus 
$s=\mathcal{P}^1_{\mathcal{S}(U)}(s)$. Therefore, we have 
$$\mathcal{S}(U)\subset \mathcal{P}^1_{\mathcal{S}(U)}(\mathbb{S}^d\setminus U^\perp).$$

It remains to verify the second equality. Pick $z\in \mathbb{S}^1$, we have 
\begin{align}
\mathcal{P}^1_{\mathcal{S}(U)}(Uz)=\frac{UU^TUz}{\|U^TUz\|}=Uz\nonumber
\end{align}
Thus $Uz\in \mathcal{S}(U)$. 
For the other direction, pick $s\in \mathcal{S}(U)$, thus 
$$s=s_U=UU^Ts$$ 
In addition, $1=\|s\|=\|UU^Ts\|=\|U^Ts\|$. 
Thus, $s\in \{Uz:z\in \mathbb{S}^1\}$ and the second equality holds.
\end{proof}

\subsection{Great circles and hemispheres in \cite{groemer1998spherical}}
In this subsection, we re-parametrize the great circles and semicircles by using the notation in \cite{groemer1998spherical}. 

Let $\mathcal{B}^d=\{(u,v): \, u,v\in \mathbb{S}^{d-1},u^Tv=0\}.$
We define, as in \cite{groemer1998spherical}, the hyperplane $H(u,v)=\text{span}\{u,v\}$ and the corresponding ``great circle'' 
$$\mathcal{S}(u,v)=\mathbb{S}^{d-1}\cap H(u,v).$$

The $(d-2)$-dimensional hemisphere is defined by 
\begin{align}
\mathcal{HS}(u,v)=\{s\in \mathbb{S}^{d-1}:s^Tu=0,s^Tv\ge 0\} \label{eq:HS(u,v)}
\end{align}   

\begin{lemma}\label{lem:trans_param_2}
Pick $U\in V_2(\mathbb{R}^d)$, $z\in \mathbb{S}^1$, there exists $(u,v)\in \mathcal{B}^d$ such that \begin{align}
    \text{span}(U)&=\text{span}(\{u,v\})\label{eq:params_cond3}\\
    U^Tv&=z\label{eq:params_cond4}
\end{align}
and vice-versa. 

In addition, if $U,z,u,v$ satisfy \eqref{eq:params_cond3}, and \eqref{eq:params_cond4}, we have: 
\begin{align}
\mathcal{S}(U)&=\mathcal{S}(u,v)\nonumber\\
\mathcal{HS}(U,z)&=\mathcal{HS}(u,v)\nonumber 
\end{align}
\end{lemma}

\begin{proof}
First, we pick $U,z$. We set $:v=Uz$. There exists a uniquely determined (up to scalar $\pm 1$) $u\in H(U)\cap v^\perp$ with $\|u\|=1$.
Since $U^Tv=U^TUz=z$,
we have that \eqref{eq:params_cond3}, and \eqref{eq:params_cond4} hold. 

Now consider $(u,v)\in \mathcal{B}^d$. We set $U=[u,v]$. (Note that the choice of $U$ is not unique, for example, we can also set $U=[-u,v]$ or $U=[u,-v]$.) 
Let $z=U^Tv=[0,1]^T$. We have that \eqref{eq:params_cond3} and \eqref{eq:params_cond4} hold. 

Suppose $U,z,u,v$ satisfy \eqref{eq:params_cond3} and \eqref{eq:params_cond4}. We have 
$ H(U,z)= H(u,v)$, and thus $$\mathcal{S}(U)=\mathcal{S}(u,v).$$

Consider $s\in \mathcal{HS}(U,z)$, we decompose $s$ as 
$s=s_U+s_{U^\perp}$, then (by Proposition \ref{pro:S(U)})
$$\mathcal{P}^1_{\mathcal{S}(U)}(s)=
Uz=UU^Tv=v.$$
The last equation holds from the fact $v\in H(U)$. 
Since $u,v\in H(U)$, we have 
\begin{align}
s^Tv&=s_U^Tv=(UU^Ts)^Tv=\|U^Ts\|v^Tv\ge 0\nonumber\\
s^Tu&=s_U^Tu=(UU^Ts)^Tu=\|U^Ts\|v^Tu=0\nonumber
\end{align}
So, $s\in \mathcal{HS}(u,v)$ and thus 
$$\mathcal{HS}(U,z)\subset \mathcal{HS}(u,v).$$

For the other direction, we pick $s\in \mathcal{HS}(u,v)$. We have $s^Tu=0$. 

Since $u\in H(U)$, we obtain: 
\begin{align}
0=s^Tu=s_U^Tu+s_{U^\perp}^Tu=s_U^Tu\nonumber
\end{align}
By combining this with the fact that $s_U\in H(U)=\text{span}(\{u,v\})$, we have $s_U=\alpha v$ for some $\alpha\in \mathbb{R}$. 
In addition, from $s^Tv\ge 0$, we have $\alpha\ge0$. 

Case 1: If $\alpha=0$, we have $s\in U^\perp$, thus $s\in \mathcal{HS}(U,z)$. 

Case 2: If $\alpha>0$, we have: 
$$P^Us=P^Us_U=\frac{U^Ts}{\|U^Ts\|}=\frac{U^T\alpha v}{\alpha}=U^TUz=z.$$
Thus, $s\in \mathcal{HS}(U,z)$. 
Therefore $\mathcal{HS}(u,v)\subset \mathcal{HS}(U,z)$. 
\end{proof}

\begin{remark}
    When $d=3$, then $u=U[:,1]$, $v=U[:,2]$, and $n_U=\frac{1}{\|u\times v\|}u\times v$ (cross product, normalized).

    Moreover, these $(d-2)$-dimensional hemispheres $\mathcal{HS}$ generalize the semicircles $\mathcal{SC}$  (see \eqref{eq:SC(U,gamma)}) given for $d=3$.
\end{remark}

\subsection{Spherical Radon transform and a Hemispherical transform}
We call the \textit{integral transformation} defined in \cite[Section 2]{groemer1998spherical} as the \textbf{``$(d-2)$-Hemispherical Transform on $\mathbb{S}^{d-1}$''}: it is the $d-$dimensional extension of the Un-normalized \textbf{Semicircle Transform} given by \eqref{eq:Wf_un} for $d=3$. 
For $f\in L^1(\mathbb{S}^{d-1})$, we define 
\begin{equation}
\mathcal{H}f(u,v):=\int_{\mathcal{HS}(u,v)}f(s)d \mathcal{HS}(u,v)\nonumber 
\end{equation}
where $d\mathcal{HS}(u,v)$ is the surface area differential on $\mathcal{HS}(u,v)$. (We note that in the article \cite{groemer1998spherical}, the notation for the above transformation is $\mathcal{B}$.)

\begin{remark}
    In the literature, the so-called \textbf{Hemispherical Transform} integrates over a full $(d-1)$-dimensional hemisphere of $\mathbb{S}^{d-1}$, that is,
    \begin{equation}\label{eq: classical hemispherical transform}
      f(z):=\int_{\{x\in\mathbb{S}^{d-1}: \, x^Tz\geq 0\}} f(x)dx \qquad \quad (z\in\mathbb{S}^{d-1}).  
    \end{equation}
    See, e.g., \cite{rubin1999inversion,groemer1998spherical}.
\end{remark}

The \textbf{Spherical Radon Transform} in $\mathbb{S}^{d-1}$ is defined in \cite{bonet2022spherical} as the linear and bounded integral  operator $\mathcal{R}:L^1(\mathbb{S}^{d-1})\to L^1(V_2(\mathbb{R}^d)\times \mathbb S^1)$ with the formula:
\begin{equation}
\mathcal{R}f(U,z):=\int_{\mathbb{S}^{d-1}}\delta(P^Us=z)f(s)ds. \label{eq:R(f)_d}
\end{equation}
Thus, we can re-write it as
\begin{equation}\label{eq: Rf}
\mathcal{R}f(U,z)=\int_{\mathbb{S}^{d-1}}\delta(P^Us=z)f(s)ds =\int_{\mathcal{HS}(U,z)}f(s) d\mathcal{HS}(U,z)
\end{equation}
We define the following \textit{variant} of such Spherical Radon Transform: 
\begin{align}
    \widetilde{\mathcal{H}}(f)(U,z):=
    \int_{\mathbb{S}^{d-1}\setminus U^\perp}\delta(P^Us=z)f(s)ds=\int_{\mathcal{HS}(U,z)\setminus U^\perp}f(s)d\mathcal{HS}(U,z).
    \label{eq:R(f)_2}
\end{align}

\begin{proposition}\label{pro:R(f)-H(f)}
For $U\in V_2(\mathbb{R}^d)$ and $z\in \mathbb{S}^1$, consider $(u,v)$ as in Lemma \ref{lem:trans_param_2}. 
Then 
\begin{align}
    \mathcal{R}f(U,z)=\mathcal{H}f(u,v)=\widetilde{\mathcal{H}}f(U,z), \qquad \forall f\in L^1(\mathbb{S}^{d-1})\nonumber
\end{align}
\end{proposition}
\begin{proof}
    By Lemma \ref{lem:trans_param_2}, we have 
    $\mathcal{HS}(U,z)=\mathcal{HS}(u,v)$. 
    Then, by \eqref{eq: Rf}, $${\mathcal{R}}f(U,z)=\mathcal{H}f(u,v).$$

    Thus, it remains to show $\widetilde{\mathcal{H}}f=\mathcal{R}f$. 
    
    For each $(U,z)$, $U^\perp$ is a $(d-2)$-dimensional subspace. Thus, $\mathbb{S}^{d-1}\cap U^\perp$ is a $(d-3)$-dimensional sub-sphere.

    Since $\{s\in \mathbb{S}^{d-1}\setminus U^\perp\},U^\perp \cap \mathbb{S}^{d-1}$ are disjoint, we have: 
    \begin{align}
    \mathcal{HS}(U,z)\cap U^\perp&=\left[\{s\in \mathbb{S}^{d-1}\setminus U^\perp: P^Us=z\} \cup (U^\perp\cap \mathbb{S}^{d-1}))\right] \cap  U^\perp \\
    &=U^\perp \cap \mathbb{S}^{d-1}\nonumber 
    \end{align}
In addition, since $U^\perp\cap \mathbb{S}^{d-1}$ is a $({d-3})$-dimensional sub-sphere, while $\mathcal{HS}(U,z)=\mathcal{HS}(u,z)$ is a $(d-3)$-dimensional hemisphere, 
we have 
\begin{align}
\mathcal{R}f(U,z)&=\widetilde{\mathcal{H}}f(U,z)+\int_{U^\perp \cap \mathbb{S}^{d-1}}f(x) d\mathcal{HS}(U,z)\nonumber\\
&=\widetilde{\mathcal{H}}f(U,z)+0=\widetilde{\mathcal{H}}f(U,z) \nonumber 
\end{align}
and we have completed the proof. 
\end{proof}

\begin{corollary}\label{eq: 1-1 for cont}
Given continuous function $f,g\in C(\mathbb{S}^{d-1})$, if $\mathcal{R}f=\mathcal{R}g$, then $f=g$. 
\end{corollary}
\begin{proof}
For every $(U,z)$ there exists $u,v\in \mathbb S^{d-1}$ such that, by Lemma \ref{pro:R(f)-H(f)}, 
$$\mathcal{R}f(U,z)=\mathcal{H}f(u,v),\quad \mathcal{R}g(U,z)=\mathcal{H}g(u,v).$$
Thus, $\mathcal{H}f(u,v)=\mathcal{H}g(u,v)$. 
From \cite[Theorem 1]{groemer1998spherical}, we have $f=g$. 
\end{proof}

    

\begin{remark}\label{remark: our difference with Bonet et al}
    In the article \cite{bonet2022spherical} the authors relate the Spherical Radon Transform $\mathcal{R}$ with the Hemispherical Transform $\mathcal{F}$ given in \eqref{eq: classical hemispherical transform}.
    Indeed,     $$\mathcal{R}f(U,z)=\mathcal{F}(f_{|_{\mathcal{S}(U)}})(z),$$
    where $f_{|_{\mathcal{S}(U)}}$ denotes the restriction of $f\in L^1(\mathbb{S}^{d-1})$ to the $(d-2)$-dimensional ``great circle'' ${\mathcal{S}(U)}\simeq\mathbb S^{d-2}$, and the above operator $\mathcal{F}$ is the Hemispherical Transform in $L^1(\mathcal{S}^{d-2})$.
    By doing so, the authors can use the characterization of $Ker(\mathcal{F})$ provided in \cite{rubin1999inversion} to prove their characterization of $Ker(\mathcal{R})$ (see \cite[Proposition 4]{bonet2022spherical}). Also, in \cite{bonet2022spherical} the authors claim that they leave for future works checking whether the set $Ker(\mathcal{R})$ is null or not. In this paper we show that it is more natural to relate $\mathcal{R}$ with the integral transformation given in \cite[Section 2]{groemer1998spherical} that we have named as the \textit{``$(d-2)$-Hemispherical Transform on $\mathbb{S}^{d-1}$''} and denoted by $\mathcal{H}$. Thus,  \cite[Theorem 1]{groemer1998spherical} we can show the injectivity of $\mathcal{R}$ for absolutely continuous measures on the sphere with continuous density functions.  (Since $\mathbb{S}^{d-1}$ is compact, the space of continuous functions $C(\mathbb S^{d-1})$ is dense in $L^1(\mathbb{S}^{d-1})$ under closure with respect to the $L^1$-norm. However, we still can not guarantee that the injectivity property of $\mathcal{R}$ extends from $C(\mathbb R^{d-1})$ to $L^1(\mathbb{S}^{d-1})$, which represents the densities of absolutely continuous measures).
\end{remark}

\section{Spherical Radon transform for measures on $\mathbb{S}^{d-1}$}\label{app: Radon for measures}

Similarly to the classical Radon Transform in Euclidean spaces, the dual-operator of $\mathcal{R}$, called the \textit{back-projection} operator and denoted as $\mathcal{R}^*: C_0(V_2(\mathbb{R}^d)\times\mathbb S^1)\to C(\mathbb{S}^{d-1})$, is given in \cite{bonet2022spherical} by the following formula:
\begin{equation}\label{eq:R*(g)_2}
    \mathcal{R}^*(\psi)(x)=\int_{V_2(\mathbb{R}^d)}\psi(U,P^U(x))d\sigma(U), \qquad x\in\mathbb{S}^{d-1}, \quad \psi\in C_0(V_2(\mathbb{R}^d)\times\mathbb S^1).
\end{equation}
That is, $\mathcal{R}^*$ is such that the following identity holds:
\begin{align}
\int_{\mathbb{S}^1\times V_2(\mathbb{R}^d)}\mathcal{R}f(U,z)\psi(U,z) d\sigma(U) dz=\int_{\mathbb{S}^{d-1}}f(s) \, \mathcal{R}^*\psi(s)ds\label{eq:R*(g)},  
\end{align}
for $f\in L^1(\mathbb{S}^{d-1}),\psi\in C_0(V_2(\mathbb{R}^d)\times\mathbb S^1)$ (see \cite[Proposition 2]{bonet2022spherical}).
Then, the Spherical Radon transform  $\mathcal{R}$ can be extended from $L^1(\mathbb{S}^{d-1})$ to the space of measures supported on $\mathbb{S}^{d-1}$ by defining $\mathcal{R}\mu$ as a new measure on $V_2(\mathbb{R}^d)\times \mathbb S^1$ by the ``duality expression'':
\begin{align}
\int_{\mathbb{S}^1\times V_2(\mathbb{R}^d)}\psi(U,z) d\mathcal{R}\mu(U,z)=\int_{\mathbb{S}^{d-1}}\mathcal{R}^*\psi(s)d\mu(s), \qquad \forall\psi\in C_0(V_2(\mathbb{R}^d)\times\mathbb S^1).
\end{align}

\begin{remark}
For each $s\in \mathbb{S}^{d-1}$, if $s\in U^\perp$, then $U[:,1],U[:,2]\in s^\perp$. Thus 
\begin{align}
&\sigma(\{U\in V_2(\mathbb{R}^d): s\in U^\perp\})\nonumber\\
&=\sigma(\{U\in V_2(\mathbb{R}^d): U[:,1],U[:,2]\in s^\perp\})\nonumber\\
&=0\nonumber 
\end{align}
In \cite{bonet2022spherical}, the authors mention the above that definition is considered a.e.. Via the above property, we have that \eqref{eq:R*(g)_2} is well-defined for all $s\in \mathbb{S}^{d-1}$. 
\end{remark}

Now, consider probability measures $\mu\in\mathcal{P}(\mathbb{S}^{d-1})$. In analogy to the classical Radon Transform in Euclidean spaces, by the \textbf{disintegration} theorem in classic measure theory, there exists  a 
$\sigma$-almost everywhere uniquely determined family of $1$-dimensional probability measures $\{(\mathcal{R}\mu)^U\}_{U\in V_2(\mathbb R^d)}$ on the unit circle $\mathbb{S}^1$ such that for any $\psi \in C_0(V_2(\mathbb R^d)\times\mathbb{S}^{1})$ we have 
\begin{align}
    \int_{V_2(\mathbb R^d)\times\mathbb{S}^{1}}\psi(U,z) \, d\mathcal{R}\mu(U,z)=\int_{V_2(\mathbb{R}^d)}\int_{\mathbb{S}^1}\psi(t,\theta) \, d(\mathcal{R}\mu)^U(z) \, d\sigma(U) \label{eq: nu_theta}.
\end{align}
 In \cite[Proposition 3]{bonet2022spherical} the authors prove that if $\mu_f$ is an absolutely continuous probability measure on the sphere $\mathbb{S}^{d-1}$, that is, it has a density $f\in L^1(\mathbb{S}^{d-1})$, then
$$(\mathcal{R}\mu_f)^U=P^U_\#(\mu_f)$$
The same computation can be extended to any measure $\mu$ supported on $\mathbb{S}^{d-1}$ to get $(\mathcal{R}\mu)^U=P^U_\#(\mu)$. 

\begin{remark}
In the particular case of $\mu_f$, that is, for functions $f\in L^1(\mathbb S^{d-1})$ we have that the measure $\mathcal{R}\mu_f$ has density $\mathcal{R}f$. This implies that the function $\mathcal{R}f(U,\cdot)$ (i.e., the function $z\mapsto \mathcal{R}f(U,z)$ defined for $z\in\mathbb S^1$) is the density of $(\mathcal{R}\mu_f)^U$, and therefore $\mathcal{R}f(U,\cdot)$ is also the density of
$P^U_\#(\mu_f)$.    
\end{remark}

\subsection{Linear Spherical Sliced Optimal Transport in terms of the Spherical Radon Transform}\label{sec: LSSOT in terms of Radon}

By using the previous expressions, given $\nu^1,\nu^2\in\mathcal{P}(\mathbb S^{d-1})$ we can re-write $LSSOT_2(\cdot,\cdot)$ as follows:
\begin{align}\label{eq: LSSOT in terms of Radon}
 \left(LSSOT_2(\nu_1,\nu_2)\right)^2&:=\int_{V_2(\mathbb{R}^d)}\left(LCOT_2(P_{\#}^U\nu_1,P_{\#}^U\nu_2)\right)^2 \, d\sigma(U) \notag\\
 &=\int_{V_2(\mathbb{R}^d)}\left(LCOT_2((\mathcal{R}\nu_1)^U,(\mathcal{R}\nu_2)^U)\right)^2  \, d\sigma(U).
\end{align}






\section{Metric property of Linear Spherical Sliced Optimal Transport}\label{app: metric property}

\begin{proposition}\label{thm: pseudo metric}
$LSSOT_2(\cdot,\cdot)$ is a well-defined \textbf{pseudo-metric} in $\mathcal{P}(\mathbb{S}^{d-1})$, that is, $LSSOT_2(\nu,\nu)=0$ for all $\nu\in\mathcal{P}(\mathbb{S}^1)$, it is non-negative,  symmetric, and satisfies the triangle inequality.     
\end{proposition}
\begin{proof}
Let $\nu_1,\nu_2,\nu_3\in \mathcal{P}(\mathbb{S}^{d-1})$. 

It is straightforward to verify positivity
(i.e., $LSSOT_2(\nu_1,\nu_2)\ge 0$), and symmetry (i.e., $LSSOT_2(\nu_1,\nu_2)=LSSOT_2(\nu_2,\nu_1)$).

If $\nu_1=\nu_2$, for each $U\in V_2(\mathbb{R}^d)$, we have $P_\#^U\nu_1=P_\#^U\nu_2$, thus 
$\hat{\nu}_1^S(\cdot,U)=\hat{\nu}_2^S(\cdot,U)$.
It implies that
$\|\hat{\nu}_1^S(\cdot,U)-\hat{\nu}_2^S(\cdot,U)\|^2_{L^2(\mathbb{S}^{d-1})}=0$, and so
$LSSOT_2(\nu_1,\nu_2)=0$. 

Finally, we verify the triangle inequality. 
Let $U\in V_2(\mathbb{R}^d)$. 
By \cite[Theorem 3.6]{martin2023lcot}, we have $LCOT_2(\cdot,\cdot)$ defines a metric. Thus, the following triangle inequality holds: 
$$LCOT_2(P_\#^U\nu_1,P_\#^U\nu_2)
\leq LCOT_2(P_\#^U\nu_1,P_\#^U\nu_3)+LCOT_2(P_\#^U\nu_2,P_\#^U\nu_3).$$
Thus, 
\begin{align}
&LSSOT_2(\nu_1,\nu_2)\nonumber\\
&=\left(\int_{V_2(\mathbb{R}^d)}(LCOT_2(P_\#^U\nu_1,P_\#^U\nu_2))d\sigma(U)\right)^{1/2}\nonumber\\
&\leq\left(\int_{V_2(\mathbb{R}^d)}(LCOT_2(P_\#^U\nu_1,P_\#^U\nu_3)+LCOT_2(P_\#^U\nu_2,P_\#^U\nu_3))d\sigma(U)\right)^{1/2}\nonumber\\
&\leq\left(\int_{V_2(\mathbb{R}^d)}(LCOT_2(P_\#^U\nu_1,P_\#^U\nu_3))d\sigma(U)\right)^{1/2}+\left(\int_{V_2(\mathbb{R}^d)}(LCOT_2(P_\#^U\nu_2,P_\#^U\nu_3))d\sigma(U)\right)^{1/2}\label{pf:tri_inq}\\
&=LSSOT_2(\nu_1,\nu_3)+LSSOT_2(\nu_2,\nu_3)\nonumber
\end{align}
where \eqref{pf:tri_inq} holds by Minkowski inequality in space $L^2(V_2(\mathbb{R}^d),d\sigma)$. 
\end{proof}

\begin{theorem}[Theorem \ref{thm: metric LSSOT}]
    $LSSOT_2(\cdot,\cdot)$ is a well-defined \textbf{metric} in the set of absolutely continuous measures in $\mathcal{P}(\mathbb{S}^{d-1})$ with continuous density function, that is, it is pseudo-metric satisfying also the identity of indiscernibles. 
\end{theorem}

\begin{proof}
    From Proposition \ref{thm: pseudo metric}, we only have to prove that if $LSSOT_2(\nu_1,\nu_2)=0$, then $\nu_1=\nu_2$. To do so, we will the expression of the $LSSOT_2(\cdot,\cdot)$ in terms of a Spherical Radon Transform given in given in \eqref{eq: LSSOT in terms of Radon}. Let $\nu_1,\nu_2\in\mathcal{P}(\mathbb{S}^{d-1})$ such that 
    $$0=\left(LSSOT_2(\nu_1,\nu_2)\right)^2=\int_{V_2(\mathbb{R}^d)}\left(LCOT_2((\mathcal{R}\nu_1)^U,(\mathcal{R}\nu_2)^U\right)^2  \, d\sigma(U).$$
    Thus, $$LCOT_2((\mathcal{R}\nu_1)^U,(\mathcal{R}\nu_2)^U)=0 \qquad \sigma(U)-a.s.$$ 
    Since $LCOT_2(\cdot,\cdot)$ defines a metric in $\mathcal{P}(\mathbb{S}^1)$ (see again \cite[Theorem 3.6]{martin2023lcot}), we have that
    $$(\mathcal{R}\nu_1)^U=(\mathcal{R}\nu_2)^U \qquad \sigma(U)-a.s.$$
    As the disintegration is a family $\sigma$-almost everywhere uniquely determined, we have that $$\mathcal{R}\nu_1=\mathcal{R}\nu_2.$$
    Finally, if $\nu_1,\nu_2$ have continuous density functions, by injectivity of the operator $\mathcal{R}$ (Corollary \ref{eq: 1-1 for cont}), we have that $\nu_1=\nu_2$.
\end{proof}

As a byproduct of our deductions, we obtain that the SSW in \cite{bonet2022spherical} and defined in \eqref{eq: SSW}, is also a metric in $\mathcal{P}(\mathbb{S}^{d-1})$:

\begin{theorem}\label{thm: SSW metric}
    The Spherical Sliced Wasserstein discrepancy $SSW_p(\cdot,\cdot)$, for $p\geq 1$, defines a metric in the set of absolutely continuous measures in $\mathcal{P}(\mathbb{S}^{d-1})$ with continuous density function.
\end{theorem}
\begin{proof}
    First, notice that since the sphere is compact, there is no need to ask that  $\nu\in\mathcal{P}(\mathbb{S}^{d-1})$ has to have finite $p$-moment, as it automatically has such a property. In \cite[Proposition 5]{bonet2022spherical}, the authors show that $SSW_p(\cdot,\cdot)$ is a pseudo-metric. Finally,  
    by using the expression
    \begin{align}
    (SSW_{p}(\mu,\nu))^p&:=\int_{V_2(\mathbb{R}^d)}\left(COT_p(P_{\#}^U\mu,P_{\#}^U\nu)\right)^p \, d\sigma(U)\\
    &=\int_{V_2(\mathbb{R}^d)}\left(COT_p((\mathcal{R}\mu)^U,(\mathcal{R}\nu)^U)\right)^p \, d\sigma(U)
    \end{align}
    and repeating the arguments in the proof of Theorem \ref{thm: metric LSSOT} with the Wasserstein distance $COTp(\cdot,\cdot)$ in the unit circle in place of the LCOT distance, we get that $SSW_p(\cdot,\cdot)$ satisfies the identity of indiscernibles when restricted to probability measures with continuous densities. 
\end{proof}

\section{Runtime Analysis with respect to Number of Slices}
Here we provide further running time comparisons of the slice-based methods for pairwise distance calculations under different numbers of slices $L=500, 1000, 5000$. All methods are run on NVIDIA RTX A6000 GPU.
\label{app: runtime-slices}
\begin{figure}[H]
    \centering
    \includegraphics[width=\linewidth]{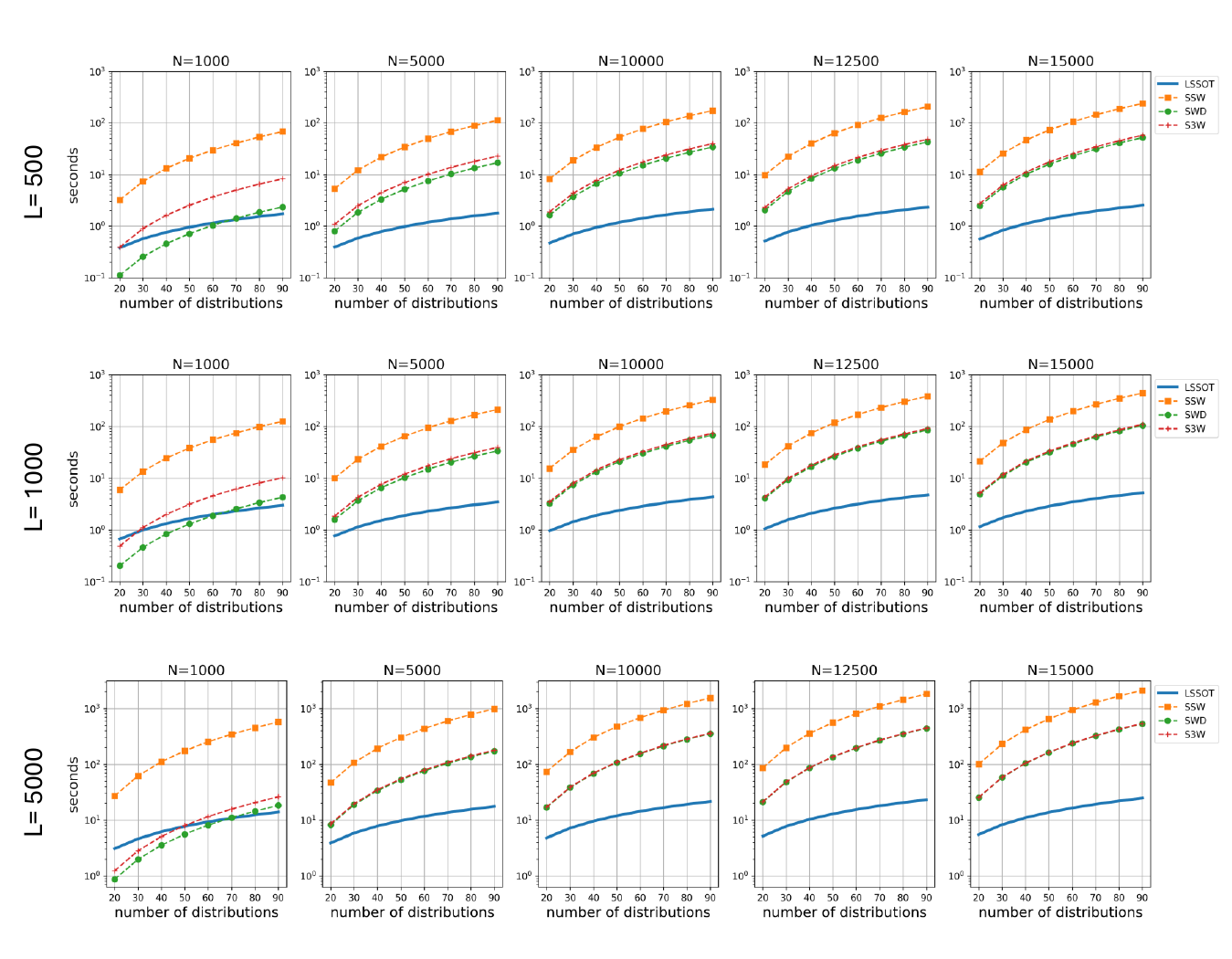}
    \caption{Running time comparison under varying number of slices. $L$ denotes the number of slices used in all the slice-based methods in each row. $N$ is the number of samples in each distribution for each column.}
    \label{fig:runtime-slices}
\end{figure}

\section{Implementation Details of Cortical Surface Registration}
\label{app: cortical-reg}
\subsection{Icospheres and Resampling}
S3Reg relies on Spherical-Unet \citep{Zhao2019SphericalUO} as the backbone, which capitalizes on the regular structure of icospheres to perform convolution and pooling on spheres. Spherical surfaces are resampled on icospheres in order to be input to the S3Reg network.

Icospheres are generated from a regular convex icosahedron with 12 vertices, then by recursively adding a new vertex to the center of each edge in each triangle. The number of vertices on the icospheres is increased from 12 to 42 (level 1), 162 (level 2), 642 (level 3), 2562 (level 4), 10,242 (level 5), 40,962 (level 6), 163,842 (level 7), etc. In this work, we focus on the level-6 icosphere with 40,962 vertices on the cortical surfaces.

For resampling in the preprocessing step, we employ the simple but fast Nearest Neighbor method and implement it using sklearn.neighbors.KDTree function from Scikit Learn \citep{scikit-learn}. The feature value on each output point is represented as a weighted sum of the feature values on the nearest three vertices.

\subsection{S3Reg Framework and Loss Components}
In a nutshell, the S3Reg framework receives a concatenation of fixed and moving spherical surface maps defined on icospheres as input, processes them through an encoder-decoder architecture, and finally outputs the velocity field represented by tangent vectors on $\mathbb{S}^2$. In particular, as the convolution kernel in Spherical U-Net could cause distortions around the poles \citep{Zhao2019SphericalUO}, 3 complementary orthogonal Spherical U-Nets are combined to predict the deformation field so that the polar areas are also regularly processed at least two times. As the 3 predicted velocity fields should be orthogonal accordingly, a loss term of inconsistency is added to the total loss. We refer the readers to \citet{9389746} (Section III-A) for details of the consistency loss.

As the similarity losses are thoroughly discussed in the main text, we will now describe other loss components in more detail. Let $F$ denote the fixed image and $M$ be the moving image. Suppose $p_F$, $p_M$ are the corresponding parcellation maps of $F$ and $M$, respectively. Let $\phi$ denote the velocity field predicted by the network.

\textbf{Dice Score and Dice Loss}
Parcellation maps can be integrated into the training process to enforce biological validity. The Dice Score is a measure of similarity/overlap between two parcellations. Assume both $p_F$ and $p_M$ have $K$ regions of interest (ROIs), then the Dice score is defined as
\begin{align}
    \label{eq: Dice}
    \text{Dice}(p_F, p_M\circ \phi) = \frac{1}{K}\sum_{k=1}^K2\frac{|p_F^k\cap(p_M^k\circ\phi)|}{|p_F^k|+|p_M^k\circ\phi|},
\end{align}
where $p_M^k\circ\phi$ is the warped moving image by $\phi$, the superscript $k$ denotes the $k$-th ROI. Since the Dice Score is between $0$ and $1$, the Dice Loss can be defined as $L_{\text{Dice}}(p_F, p_M\circ \phi)=1-\text{Dice}(p_F, p_M\circ \phi)$.

\textbf{Smoothness Loss}
The local spatial non-smoothness of the output velocity field $\phi$ is penalized by $L_{\text{smooth}}(\phi) =\sum_v|\nabla_{\mathbb{S}^2}C(\phi(v))|$, where the sum is over all the moving vertices $v$, $C$ is the 1-ring convolution operator \citep{Zhao2019SphericalUO}, and $\nabla_{\mathbb{S}^2}$ denotes the spherical gradients.


\section{More visualizations of Cortical Surface Registration}
\label{app: visual-cortical-reg}
Below, we provide additional visualizations of our cortical surface registration comparison. The figures are presented in landscape mode to enhance visual inspection.

\begin{sidewaysfigure}[ht!]
    \centering
    \includegraphics[width=.9\textwidth]{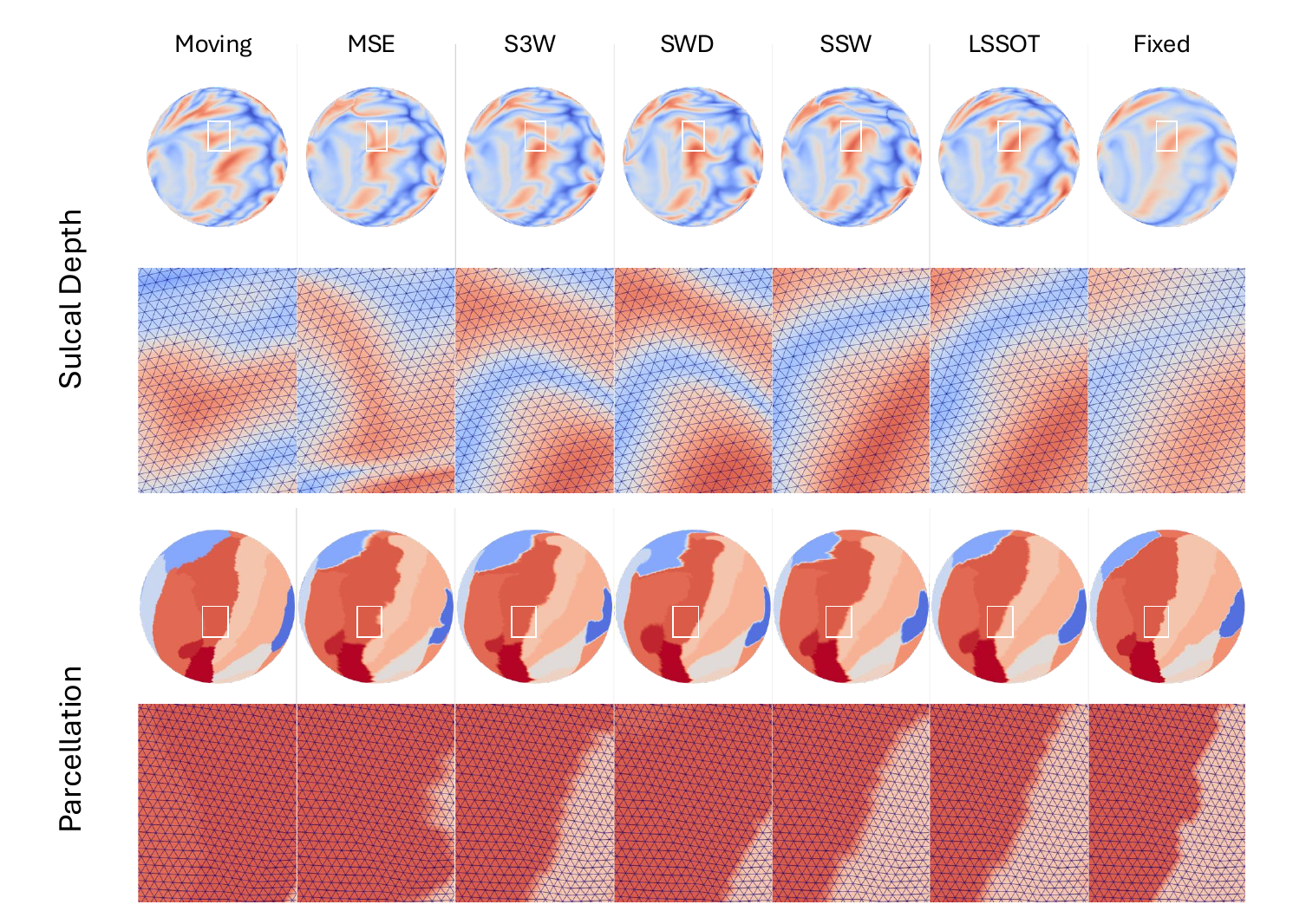}
    \caption{Qualitative registration results from Subject A00060516 in the NKI dataset.}
    \label{fig:nki-A00060516}
\end{sidewaysfigure}

\begin{sidewaysfigure}[ht!]
    \centering
    \includegraphics[width=.9\textwidth]{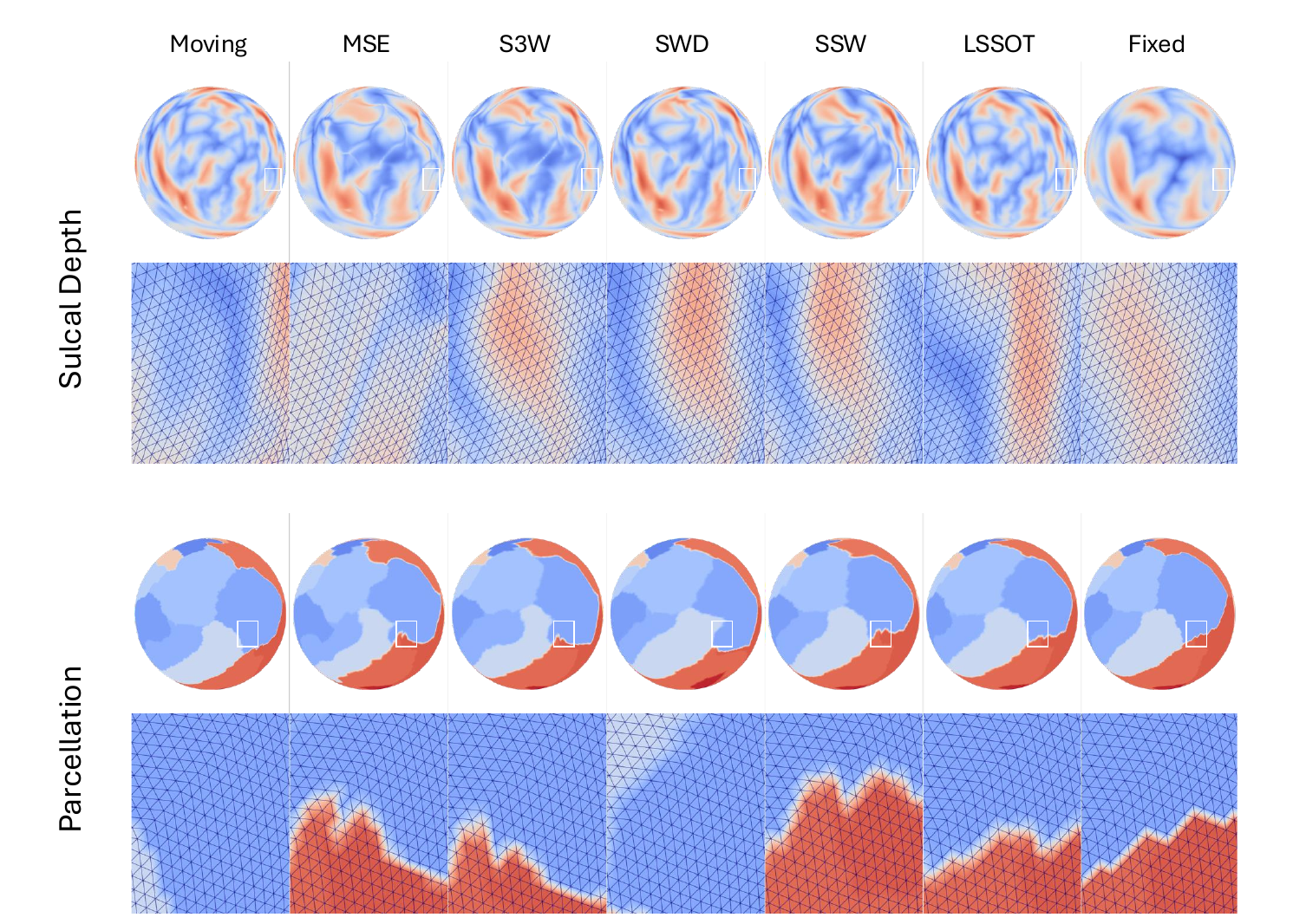}
    \caption{Qualitative registration results from Subject A00076798 in the NKI dataset.}
    \label{fig:nki-A00076798}
\end{sidewaysfigure}

\begin{sidewaysfigure}[ht!]
    \centering
    \includegraphics[width=.9\textwidth]{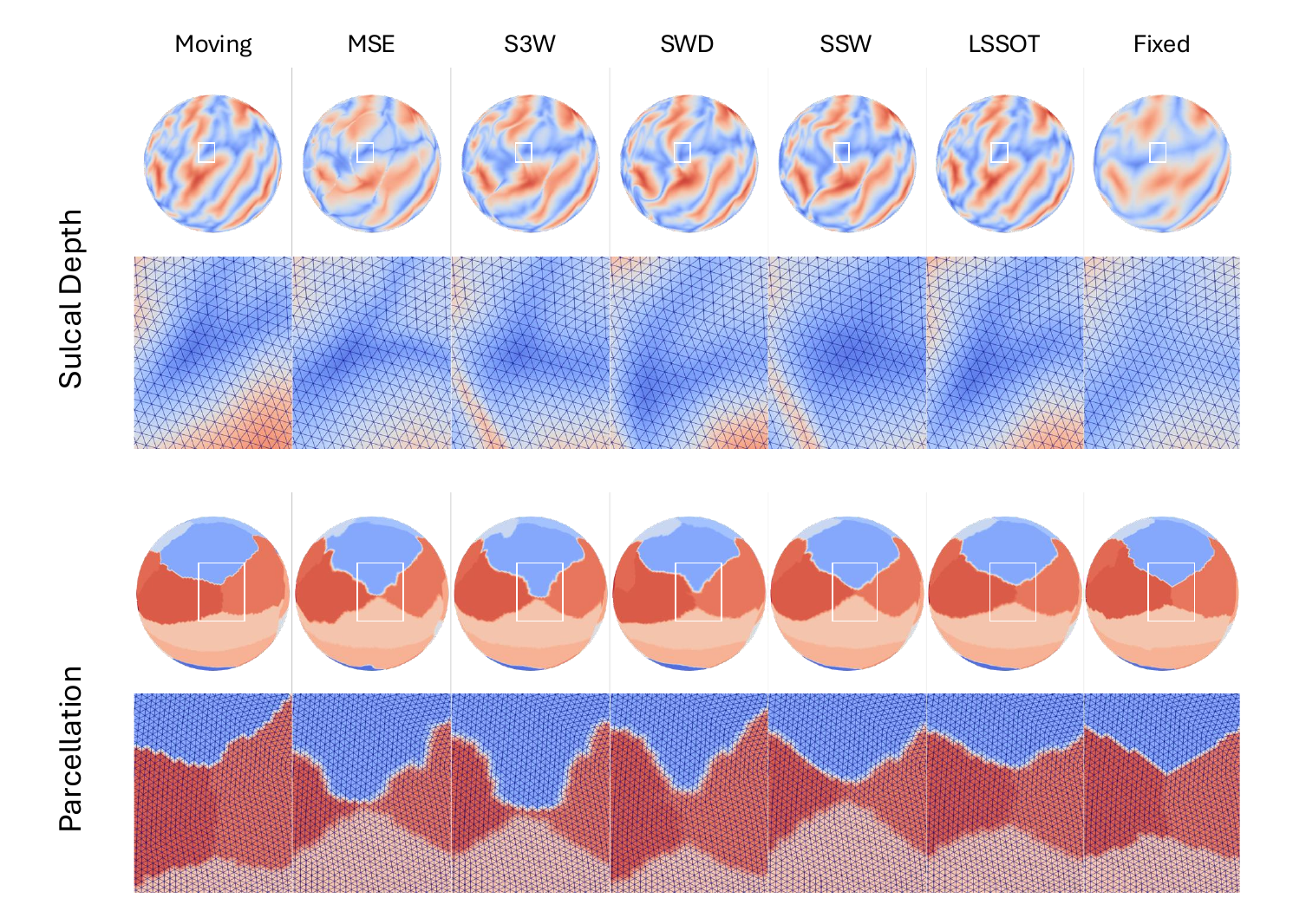}
    \caption{Qualitative registration results from Subject 0356 in the ADNI dataset.}
    \label{fig:adni-0356}
\end{sidewaysfigure}

\begin{sidewaysfigure}[ht!]
    \centering
    \includegraphics[width=.9\textwidth]{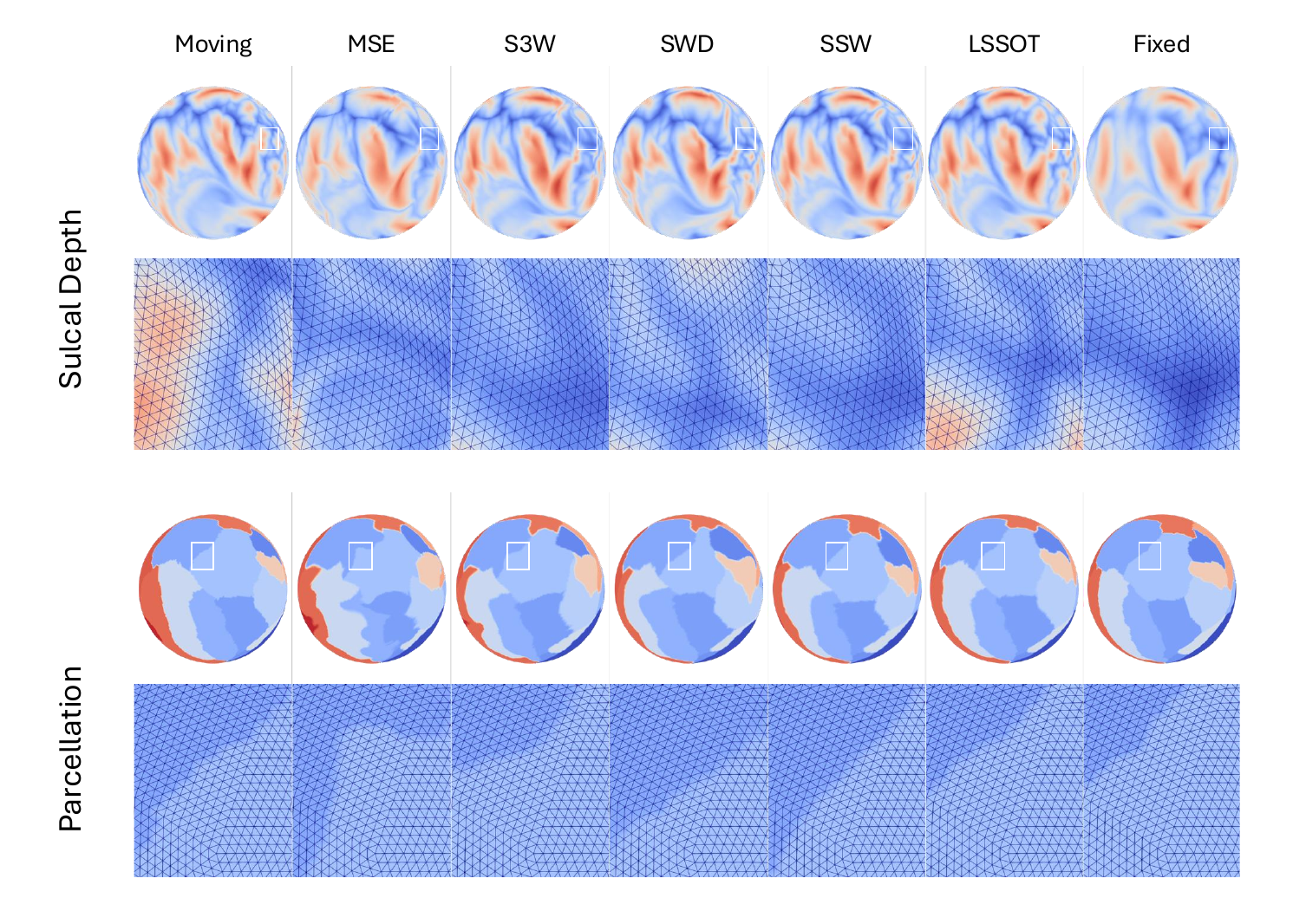}
    \caption{Qualitative registration from Subject 6650 in the ADNI dataset.}
    \label{fig:adni-6650}
\end{sidewaysfigure}

\begin{sidewaysfigure}[ht!]
    \centering
    \includegraphics[width=.9\textwidth]{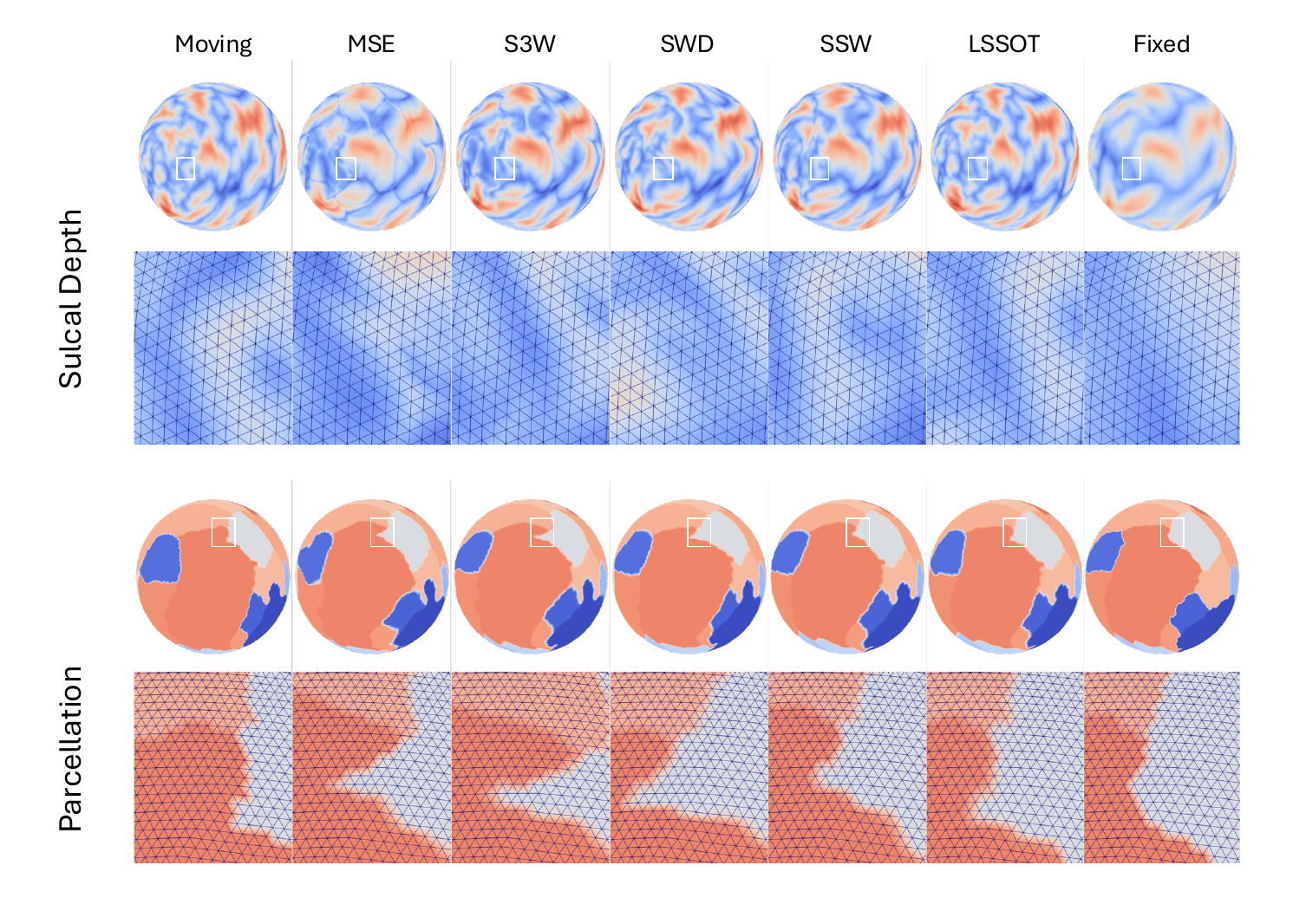}
    \caption{Qualitative registration results from Subject 7060 in the ADNI dataset.}
    \label{fig:adni-7060}
\end{sidewaysfigure}

\newpage
\clearpage

\section{3D Shape Clustering and Retrieval}
To further validate the effectiveness of LSSOT in evaluating the similarity of spherical distributions, we conduct additional experiments using a variety of genus-0 3D shapes. Our methodology involves a two-step process: 1) mapping the 3D shapes onto spherical surfaces; and 2) performing clustering and nearest neighbor retrieval tasks on these spherical representations. For the spherical mapping process, we employ two distinct techniques, Inflation and Projection (IAP) and Deep Unsupervised Learning using PointNet++ \cite{qi2017pointnet++}.

\subsection{Dataset}
We conduct experiments on mesh models from the Watertight Track of SHREC 2007 dataset \citep{Giorgi2007SHapeRC}. This dataset contains surfaces belonging to 19 categories, among them, we select 4 genus-0 closed surface models, human, octopus, teddy bear, and Fish. Moreover, as explored in the previous experiment, LSSOT is sensitive to rotations, therefore we remove the meshes with different orientations than the rest.

\subsection{Spherical Mapping}
\textbf{Inflation and Projection.} 
A direct projection from a genus-0 closed surface to the sphere by fixing an inside point source and normalizing the vertices may cause undesirable topologically incorrect foldings. One workaround is iteratively smoothing and inflating the surface before the final projection. This inflation and projection is a common spherical mapping strategy used in neuroimaging data processing tools such as FreeSurfer \citep{fischl2012freesurfer}, Connectome Workbench \citep{marcus2013human}.

\textbf{Deep Unsupervised Learning using PointNet++.}
To accelerate the forward pass of spherical mapping, we capitalize on deep unsupervised learning and train a PointNet++ network to map a cortical surface on a sphere with minimal distortion. The trained network then serves as a deep learning-based spherical mapping tool.

\textit{Loss function.}
\citet{zhao2022fast} proposed an unsupervised approach of mapping cortical surfaces to a sphere using Spherical U-Net \cite{Zhao2019SphericalUO} to learn the spherical diffeomorphic deformation field for minimizing the metric (distance), area, or angle distortions. We follow their definitions of metric ($J_d$) and area ($J_a$) distortions as local relative differences, and conformal/angle distortion ($J_c$) after market share normalization, but adapt them to general meshes instead of restricting to icosahedron-reparameterized surfaces. For a pair of original input surface $R$ and the predicted spherical mesh $S$, the loss function has three components:
\begin{align}
    &J_d(R, S) = \min_{r\in\mathbb{R}^+}\frac{1}{N_e}\sum_{i=1}^{N_e}\frac{|rd_i^S-d_i^R|}{d_i^R} \label{eq: edge-distort}\\
    &J_a(R, S) = \min_{r\in\mathbb{R}^+}\frac{1}{N_f}\sum_{j=1}^{N_f}\frac{|ra_j^S-a_j^R|}{a_j^R} \label{eq: area-distort}\\
    &J_c(R, S) = \frac{1}{3N_f}\sum_{k=1}^{N_f} \frac{1}{\pi}(|\alpha^R_k-\alpha^S_k| + |\beta^R_k-\beta^S_k| + |\theta^R_k-\theta^S_k|)
\end{align}
where $N_e$ is the total number of edges, $N_f$ is the total number of faces. $d_i^R$ and $d_i^S$ are the length of $i$-th edge on original surface and predicted sphere. Similarly, $a_j^R$ and $a_j^S$ are the area of $j$-th face on original surface and predicted sphere; $(\alpha^R_k, \beta_k^R, \theta^R_k)$ and $(\alpha^S_k, \beta_k^S, \theta^S_k)$ are the inner angles of the $k$-th face on original surface and predicted sphere.

\begin{figure}[H]
    \centering
    \includegraphics[width=.7\textwidth]{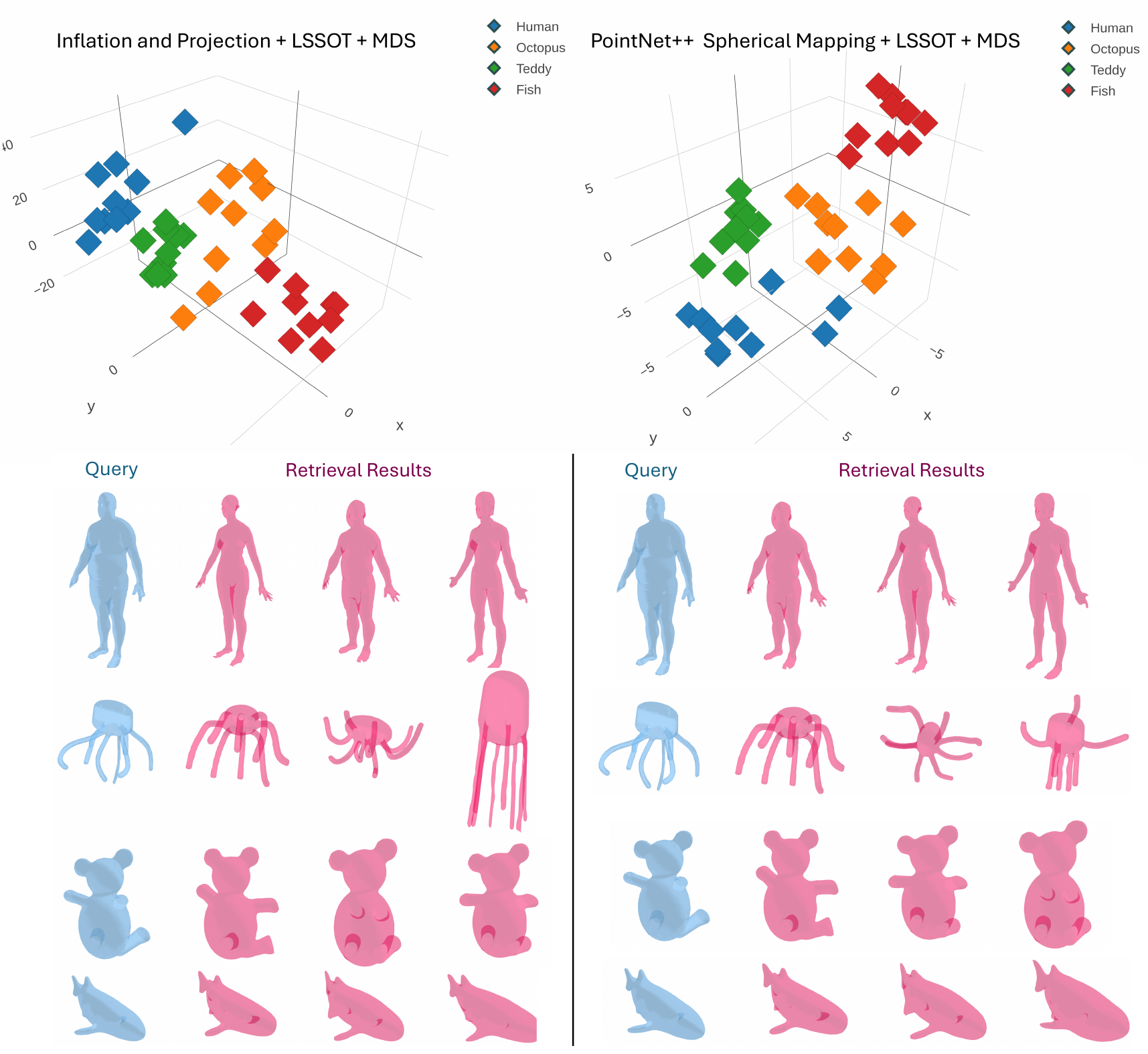}
    \caption{Manifold learning of clustering results (top) and nearest neighbor retrieval results (bottom) using Inflation and Projection (left) and PointNet++ (right) as spherical mappings. }
    \label{fig:shrec-results}
\end{figure}
\subsection{Results}
Figure \ref{fig:shrec-results} illustrates the effectiveness of the LSSOT metric under both spherical mapping techniques. The results demonstrate that LSSOT successfully identifies inherent similarities and dissimilarities among the spherical surface representations, regardless of the mapping method employed. This capability translates into good performance in both clustering and nearest neighbor retrieval tasks. The consistency of LSSOT's effectiveness under different mapping approaches underscores its robustness as a metric for analyzing spherical distributions and comparing surface geometries.

\section{Gradient flow with respect to LSSOT in the spherical space}
\subsection{background: introduction to Stiefel manifold}
In this section, we briefly introduce the gradient flow and related concepts in spherical space. 
First, we introduce the Riemannian structure of the Stiefel manifold, which is a general case of spherical space. 
Next, we specify the related concepts in the spherical space. 

Given positive integers $1\leq k\leq d$, the Stiefel manifold is given by 
$$\mathbb{V}_{d,k}:=\mathbb{V}_{k}(\mathbb{R}^d)=\{V\in \mathbb{R}^{d\times k}:V^\top V=I_k\}.$$
It is a $\left(dk-(k+\binom{k}{2})\right)$-dimensional manifold. 

Given $X\in \mathbb{V}_{d,k}$, the tangent space is given by: 
\begin{align}
\mathcal{T}_{X}(\mathbb{V}_{d,k}):=\{Z\in \mathbb{R}^{d\times k}: ZX+XZ=0\}\label{eq:tangent_V}. 
\end{align}

Given matrix $G\in \mathbb{R}^{d\times k}$, the projection of $G$ onto $\mathcal{T}_X(\mathbb{V}_{d,k})$ is given by, 
\begin{align}
  P_{T_X(\mathbb{V}_{d,k})}(G):=G-X\text{sym}(G^\top X)\label{eq:project_V},
\end{align}
where $\text{sym}(G^\top X)=\frac{1}{2}(G^\top X+X^\top G)$ is the symmetric matrix obtained from $G^\top X$. 

Given a differentiable function $F: V_{d,k}\to \mathbb{R}$,  the following steps can describe the gradient of $F$ in the manifold. In particular, choose $X\in \mathbb{V}_{d,k}$:
\begin{itemize}
    \item Step 1: compute Euclidean gradient $\nabla_E(F(X))$.
    \item Step 2: the Stiefel manifold gradient is defined as:
    \begin{align}
      \nabla_S F(X):=P_{\mathcal{T}_X(\mathbb{V}_{d,k})}(\nabla_E (F(X)))  \label{eq:Gradient_S}.
    \end{align}
\end{itemize}
We consider the gradient flow problem: 
$$\partial_t X_t=-\nabla_S F(X_t),\forall X\in \mathbb{V}_{d,k}.$$
Let $\eta$ denote the step size (i.e., the learning rate for the machine learning community), the discrete gradient descent step of the above problem becomes: 
\begin{align}
X\mapsto X-\eta \nabla_SF(X)=X-\eta (\nabla_EF(X)-X\text{sym}(\nabla_E F(X)^\perp X))\label{eq:gradient_descent_St}.
\end{align}

\subsection{Gradient flow in spherical space}
In spherical space, we have $k=1$ (in this case, we use convention $\binom{1}{2}=0$). Thus, the tangent space \eqref{eq:tangent_V} becomes: 
\begin{align}
\mathcal{T}_{x}(\mathbb{S}^{d-1}):=\text{span}(x)^\perp:=\{z\in \mathbb{R}^{d}: z^\top x=0\}\label{eq:tangent_Sph}, 
\end{align}
where $\text{span}(x)$ is the 1D line spaned by $x$. The corresponding projection \eqref{eq:project_V} becomes
\begin{align}
\mathbb{R}^d\ni g\mapsto P_{\mathcal{T}_{x}(\mathbb{V}_{d,k})}(g):=g-x(g^\top x)\in \mathcal{T}_x(\mathbb{S}^{d-1})\label{eq:proj_Sph}. 
\end{align}
Note, the above projection is essentially the Euclidean projection of $x$ into the hyperplane $x^\perp$. 

Given $X=[x_1,\ldots, x_n]\in \mathbb{R}^{d\times n},Y=[y_1,\ldots, y_m]\in \mathbb{R}^{d\times m}$, where $Y$ is fixed.
Let $\nu_1=\sum_{i=1}^n p_i \delta_{x_i}, \nu_2=\sum_{j=1}^m q_j\delta_{y_j}$. We consider the following function 
: 
\begin{align}
C(X)&:=LSSOT(\mu,\nu)\nonumber\\
    &\approx \frac{1}{T}\sum_{t=1}^TLCOT((\theta_t)_\#\mu,(\theta_t)_\#\nu) \label{eq:C(X)}.
\end{align}
where each $\theta_t\in\mathbb{S}^{d-1}$, \eqref{eq:C(X)} is the Monte Carlo approximation.

We select one $\theta$ and consider the 1-dimensional LCOT problem 
\begin{align}
&LCOT((\theta)_\#\nu_1,(\theta)_\#\nu_2)\nonumber\\
&=\int_0^1 h\left(\left|F_{\theta_\#\nu_1}^{-1}(s-\mathbb{E}[\theta_\#\nu_1]+\frac{1}{2})-|F_{\theta_\#\nu_2}^{-1}(s-\mathbb{E}[\theta_\#\nu_2]+\frac{1}{2})\right|^2\right) ds\nonumber  
\end{align}
where $h(r)=\min(r,1-r)$.

Note, the above quantity implicitly defines a transportation plan between $\mu$ and $\nu$. In particular, suppose $\gamma_\theta\in \mathbb{R}^{n\times m}$, with 
\begin{align}
    \gamma_{\theta}[i,j]=|\{s: F_{\theta_\#\nu_1}^{-1}(s-\mathbb{E}[\theta_\#\nu_1])\equiv x_i, F_{\theta_\#\nu_1}^{-1}(s-\mathbb{E}[\theta_\#\nu_2])\equiv y_j\}| \label{eq:Gamma_theta}, 
\end{align}
where $|\cdot|$ is the uniform measure (Lebesgue measure) in $[0,1)$, $a\equiv_1 b$ means $a \mod 1= b$. 
It is straightforward to verify $\gamma_\theta\in\Gamma(\theta_\#\nu_1,\theta_\#\nu_2)$, and thus is a transportation plan between $\theta_\#\nu_1,\theta_\#\nu_2$. 

Therefore, we can explicitly write \eqref{eq:C(X)} in term of $X=[x_1,\ldots x_n]$: 
\begin{align}
\eqref{eq:C(X)}=\frac{1}{T}\sum_{t=1}^T \sum_{i,j}h(\theta_t^\top x_{i}-\theta_t^\top y_j)\gamma_{\theta_t}[i,j]\label{eq:C(X)_2}.
\end{align}
where each $\gamma_{\theta_t}$ is defined via \eqref{eq:Gamma_theta}. 
For each $x_i$, the Euclidean gradient is given by 
\begin{align}
\nabla_{E}C(x_i)
&=\frac{1}{T}\sum_{t=1}^T \theta_t\left(\sum_{j: |\theta_t^\top x_i-\theta_t^\top y_j|\leq 1/2}2(\theta_t^\top x_{i}-\theta_t^\top y_j)\gamma_{\theta_t}[i,j]\right.\nonumber\\
&\qquad-\left.\sum_{j: |\theta_t^\top x_i-\theta_t^\top y_j|>1/2}2(\theta_t^\top x_{i}-\theta_t^\top y_j)\gamma_{\theta_t}[i,j]\right)\label{eq:Gradient_C(X)}. 
\end{align}

Based on this, we derive the gradient descent step for the following gradient flow problem.  
$$-\nabla_S C(X)=-V(X)$$
That is,
\begin{align}
    X\mapsto X- \eta \nabla_S(C(X))=X-(X-X(\nabla_EC(X)^\top X))\label{eq:GD_LSSOT}
\end{align}
where $\nabla_E(C(X))[:,i]$ is given by \eqref{eq:Gradient_C(X)} for each $i$.

\subsection{Implementation of Gradient Flow}
\label{subsec: gradient-flow algo}
\textbf{Riemannian Gradient Descent.} Let $\mu_\mathcal{X} = \frac{1}{N}\sum_{n=1}^N\delta_{x_n}$ be the source distribution with mass uniformly distributed at $\mathcal{X}=\{x_n\}_{n=1}^N\subset S^{d-1}$, and $\nu_\mathcal{Y} = \frac{1}{M}\sum_{m=1}^M\delta_{y_m}$ be the target distribution with $\mathcal{Y}=\{y_m\}_{m=1}^M\subset S^{d-1}$. For each $x\in \mathcal{X}$, let $x^{(k)}$ be the location of $x$ at $k$-th iteration, then at $(k+1)$-th iteration, the update should be the following \citep{absil2008optimization, boumal2019global}:
\begin{equation}
    x^{(k+1)} = \exp_{x^{(k)}}(-\gamma_k\nabla_{S^{d-1}}\mathcal{F}(x^{(k)})),
\end{equation}
where $\gamma_k$ is the learning rate at $k$-th iteration, the exponential map $\exp_x(v)=x\cos{\|v\|}+\frac{v}{\|v\|}\sin{\|v\|}$, and $\nabla_{S^{d-1}}\mathcal{F}(x^{(k)})$ is the projection of the Euclidean gradient $\nabla_{\mathbb{R}^d} \mathcal{F}(x^{(k)})$ onto the tangent space at $x^{(k)}$. We implement this gradient descent method using the exponential function and projection function in the Geoopt library \citep{geoopt2020kochurov} \footnote{https://github.com/geoopt/geoopt}.

\textbf{Spherical Coordinate Gradient Descent.} A more straightforward gradient descent method is to work in the spherical coordinate system. By the following correspondences (in $\mathbb{S}^2$ for example) between $(x, y, z)\in \mathbb{S}^2$ and $(\theta, \phi)$
\begin{align*}
    &x=\cos{\theta}\sin{\phi}\\
    &y=\sin{\theta}\sin{\phi}\\
    &z=\cos{\phi}
\end{align*}
and 
\begin{align*}
    &\theta=\arctan(\frac{y}{x})\\
    &\phi=\arccos{z}
\end{align*}
we can calculate the gradient of the loss function with respect to $(\theta, \phi)$. Note here $\arctan$ must be suitably defined according to the quadrant.
\subsection{Additional Point Cloud Interpolation Visualizations}
\label{subsec: pc_interp}
\begin{figure}[H]
    \centering
    \vspace{-.2in}
    \includegraphics[width=0.65\linewidth]{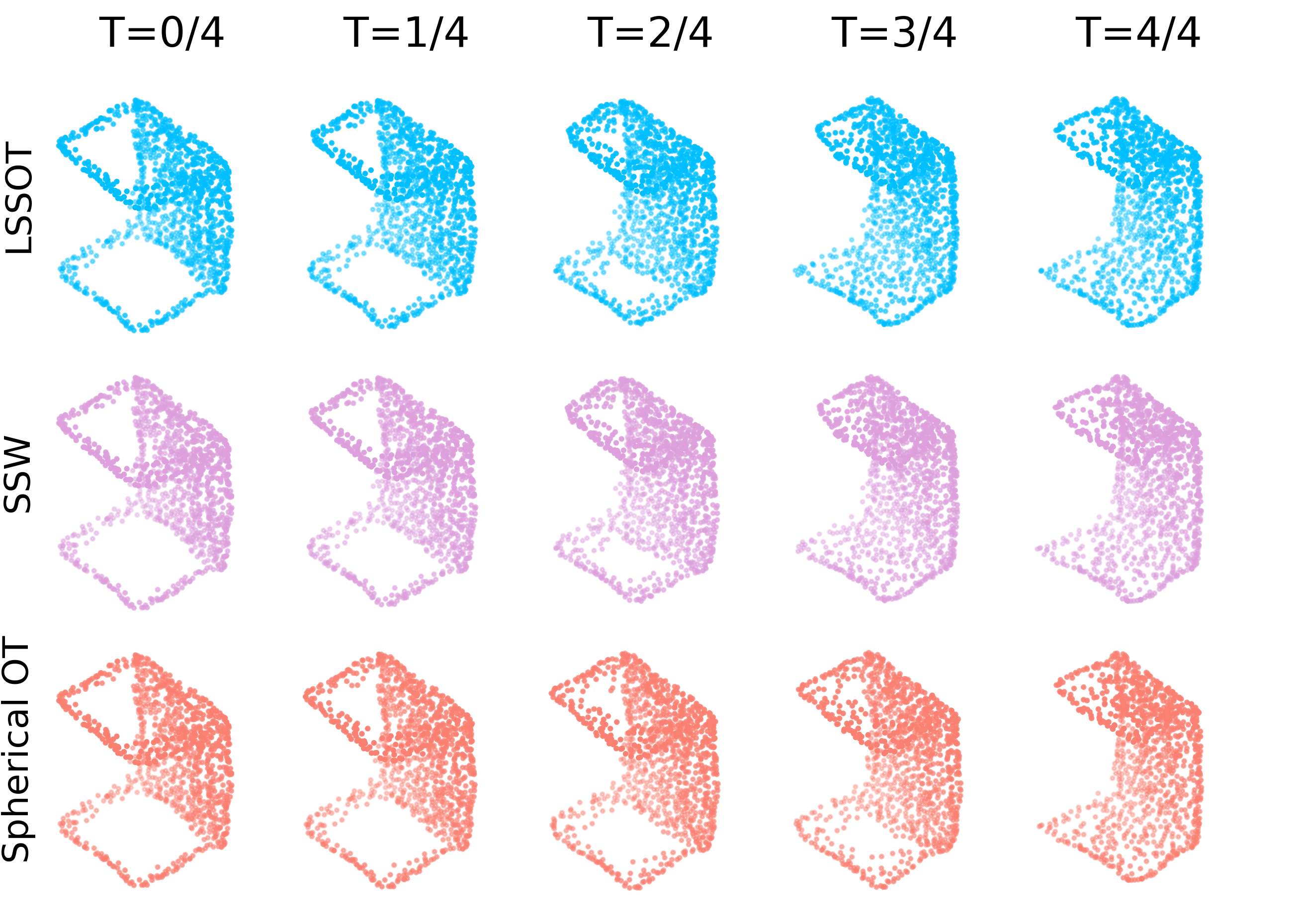}
    \vspace{-.2in}
    \caption{Gradient flow interpolations between two tables.}
    \label{fig:35_53}
\end{figure}
\begin{figure}[H]
    \centering
    \vspace{-.2in}
    \includegraphics[width=0.65\linewidth]{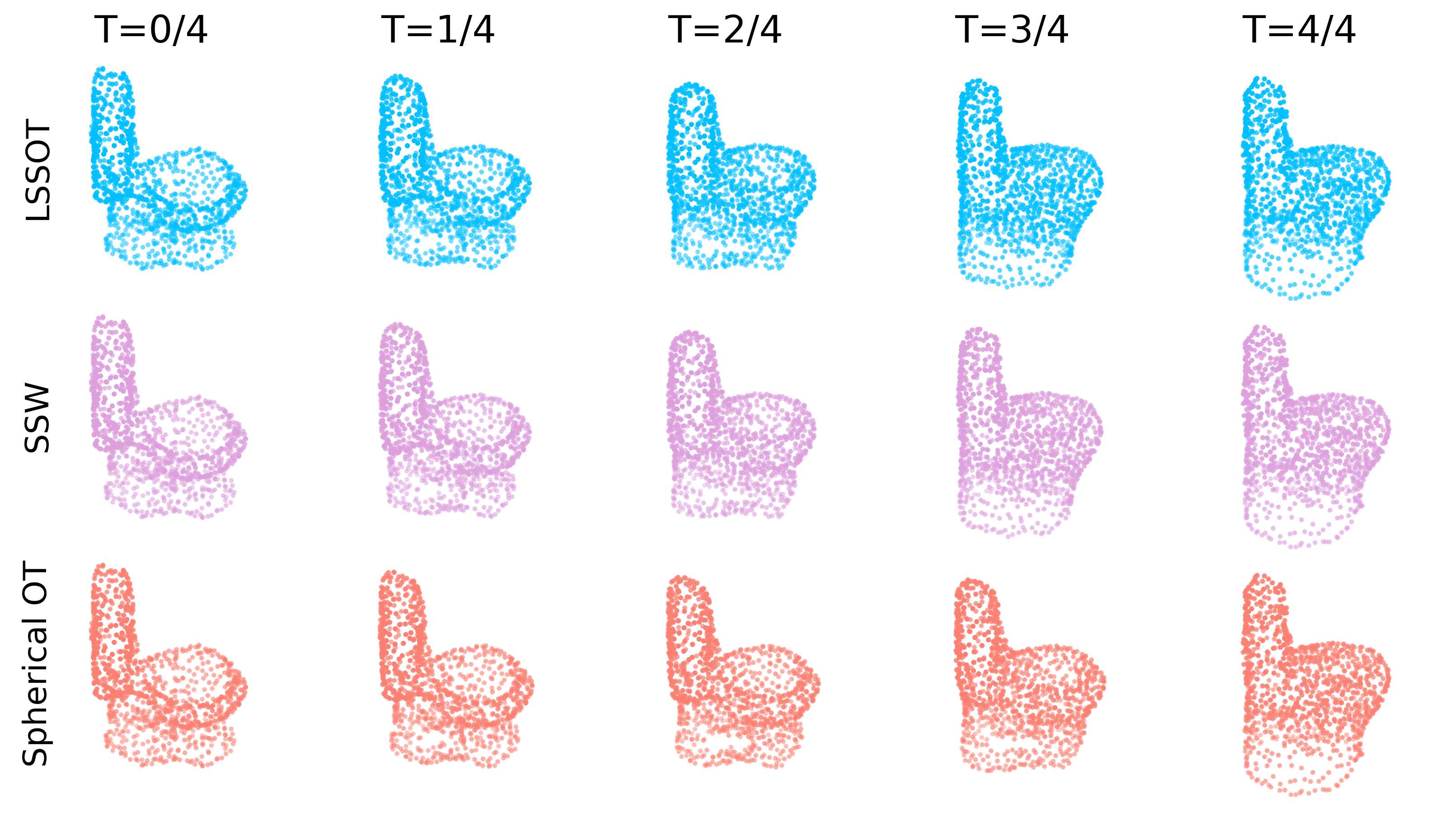}
    \vspace{-.2in}
    \caption{Gradient flow interpolations between two toilets.}
    \label{fig:92_99}
\end{figure}
\begin{figure}[H]
    \centering
    \vspace{-.25in}
    \includegraphics[width=0.65\linewidth]{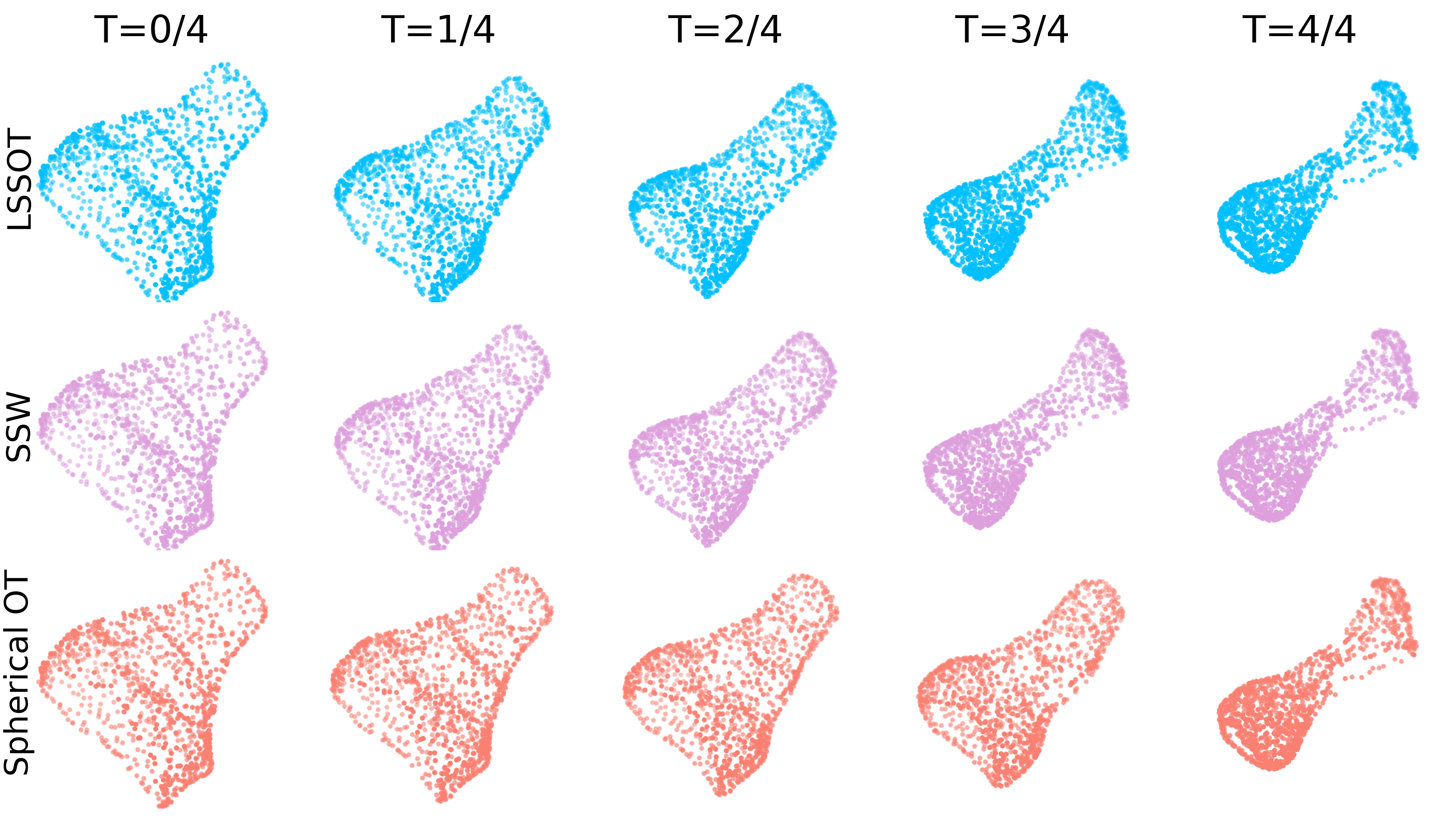}
    \vspace{-.15in}
    \caption{Gradient flow interpolations from a range hood to a plant.}
    \label{fig:37_82}
\end{figure}

\begin{figure}[H]
    \centering
    \includegraphics[width=0.7\linewidth]{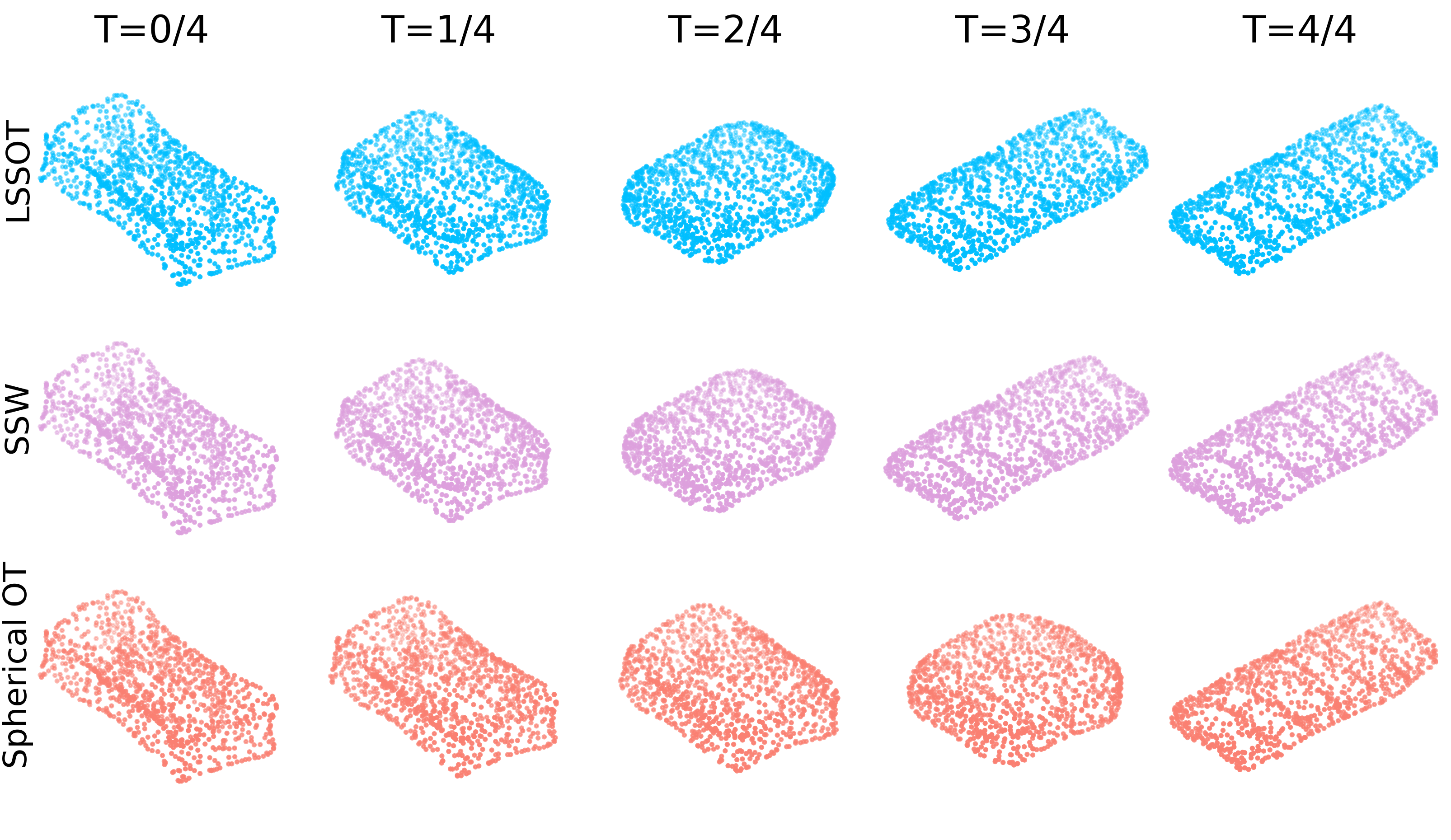}
    \caption{Gradient flow interpolations from a TV stand to a bookshelf.}
    \label{fig:74_68}
\end{figure}

\begin{figure}[H]
    \centering
    \includegraphics[width=0.7\linewidth]{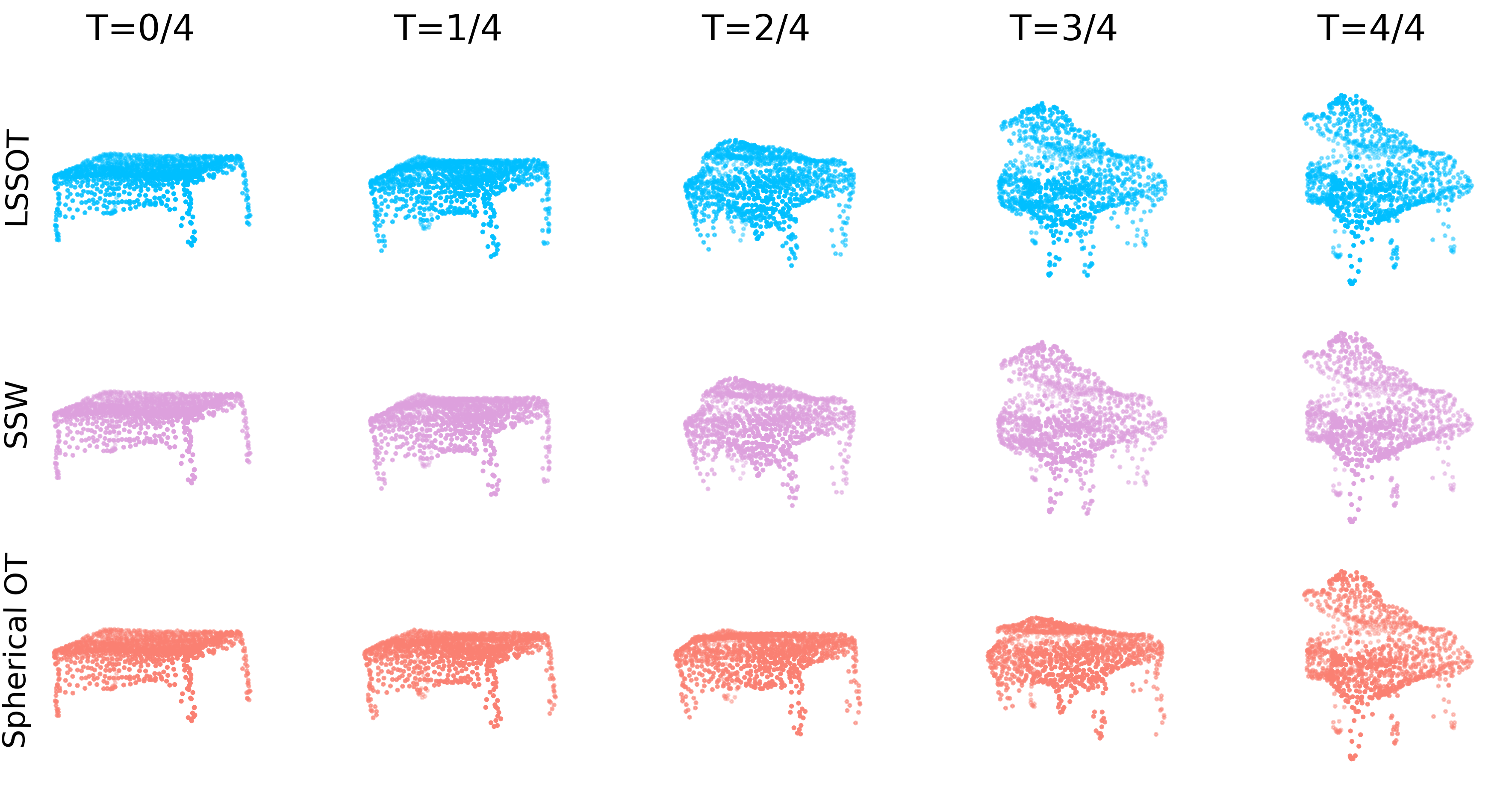}
    \caption{Gradient flow interpolations from a table to a piano.}
    \label{fig:table_piano}
\end{figure}

\begin{figure}[H]
    \centering
    \includegraphics[width=0.7\linewidth]{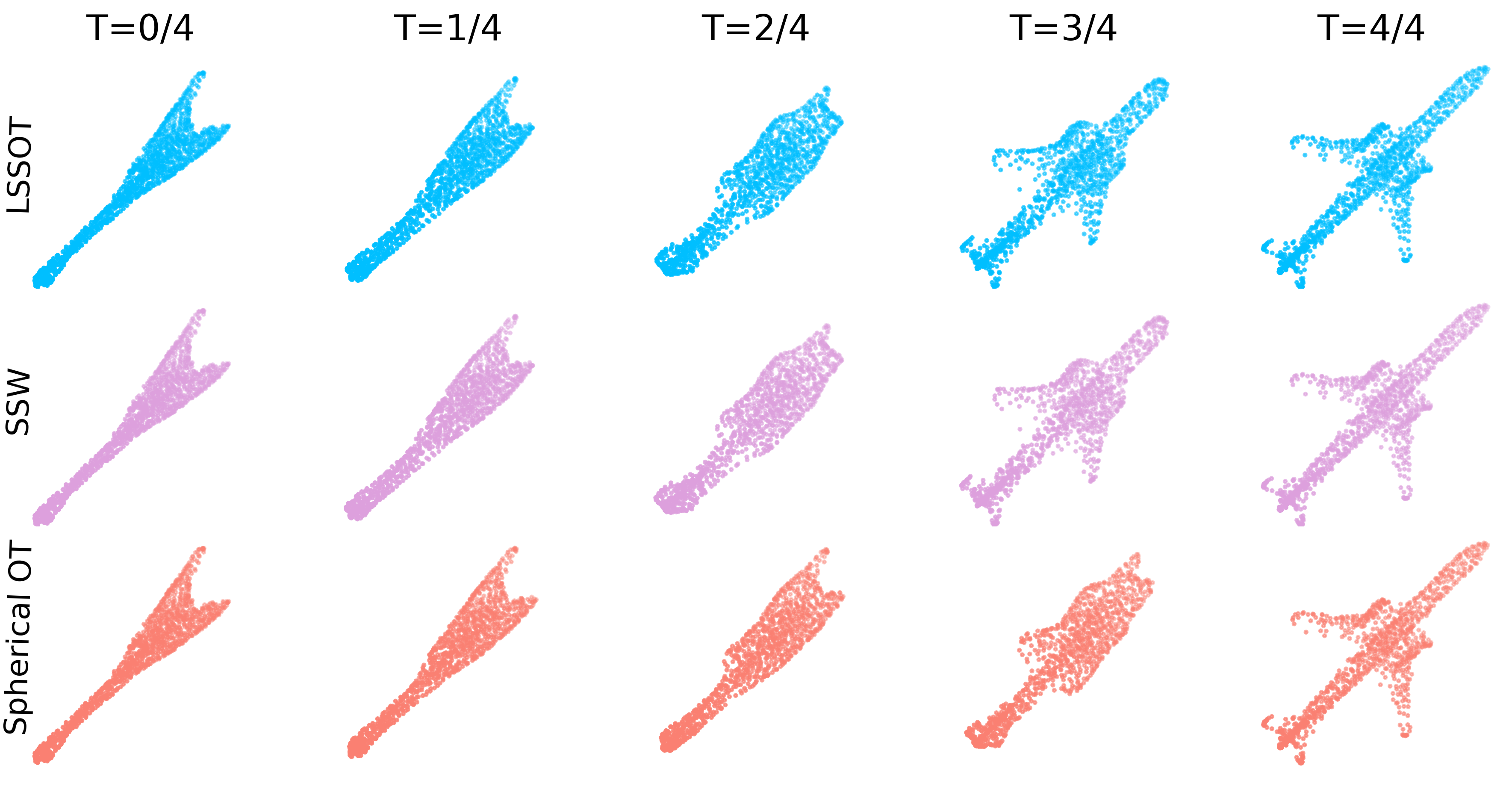}
    \caption{Gradient flow interpolations from a guitar to an airplane.}
    \label{fig:guitar_airplane}
\end{figure}

\subsection{Gradient Flow Mollweide View Visualizations}
\vspace{-.2in}
We provide more visualizations for the gradient flow of spherical distributions with respect to LSSOT metric, SSW, and spherical optimal transport distance (spherical OT). At first, a source distribution and a target distribution are generated, both of which are discrete distributions with uniform mass at all points. Then we perform gradient descent on the mass locations of the source distribution with respect to the above three objectives, driving the source toward the target distribution by minimizing these losses. We employ two methods of gradient descent on the sphere.
See Figure \ref{fig:uni_lssot_rem} to Figure \ref{fig:mixture_ot_coor} for the Mollweide views of the gradient flows between two different pairs of source-target distributions.
\begin{figure}[H]
    \centering
    \vspace{-.3in}
    \includegraphics[width=0.7\linewidth]{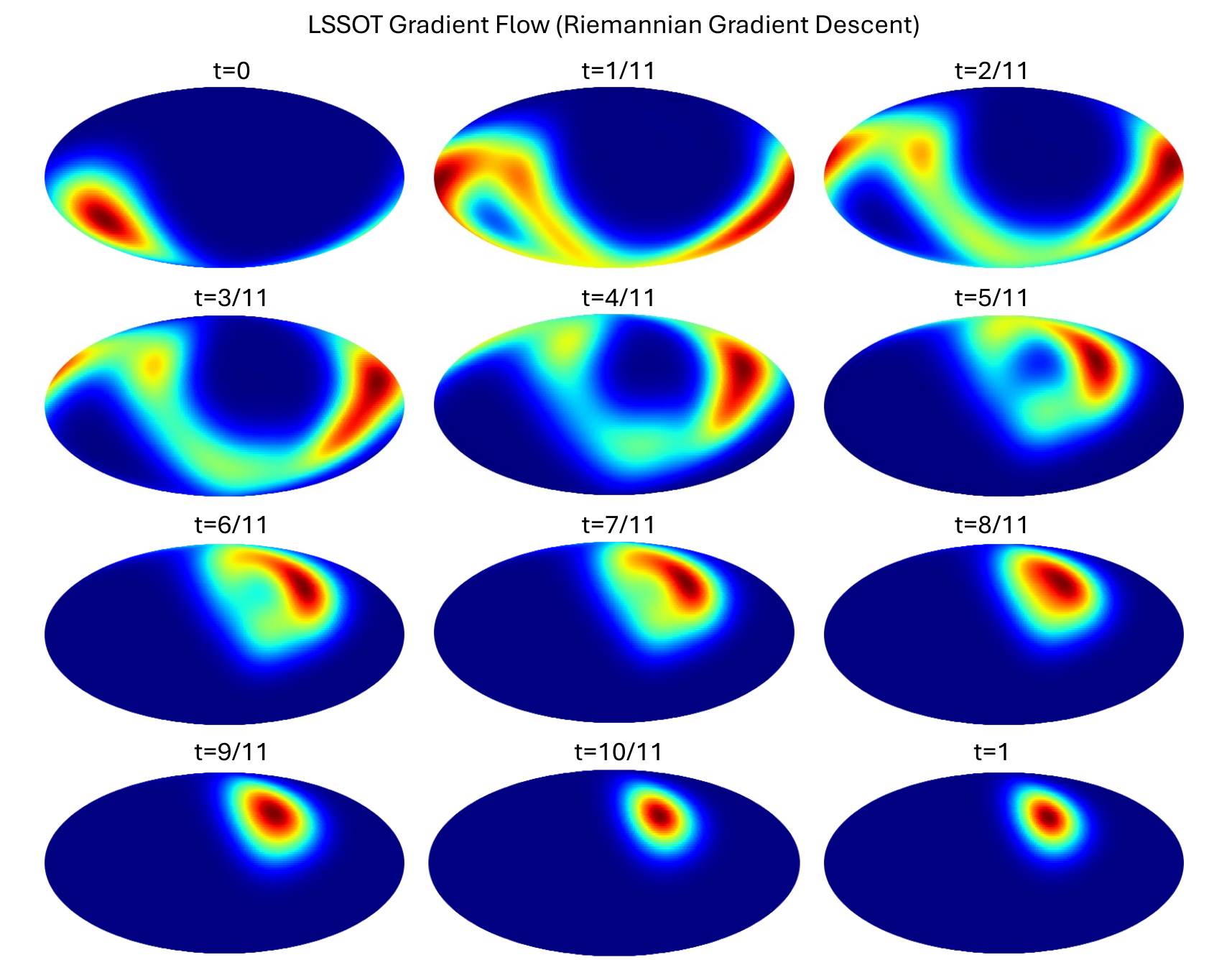}
    \caption{LSSOT gradient flow using Riemannian gradient descent.}
    \label{fig:uni_lssot_rem}
\end{figure}
\begin{figure}[H]
    \centering
    \vspace{-.4in}
    \includegraphics[width=0.7\linewidth]{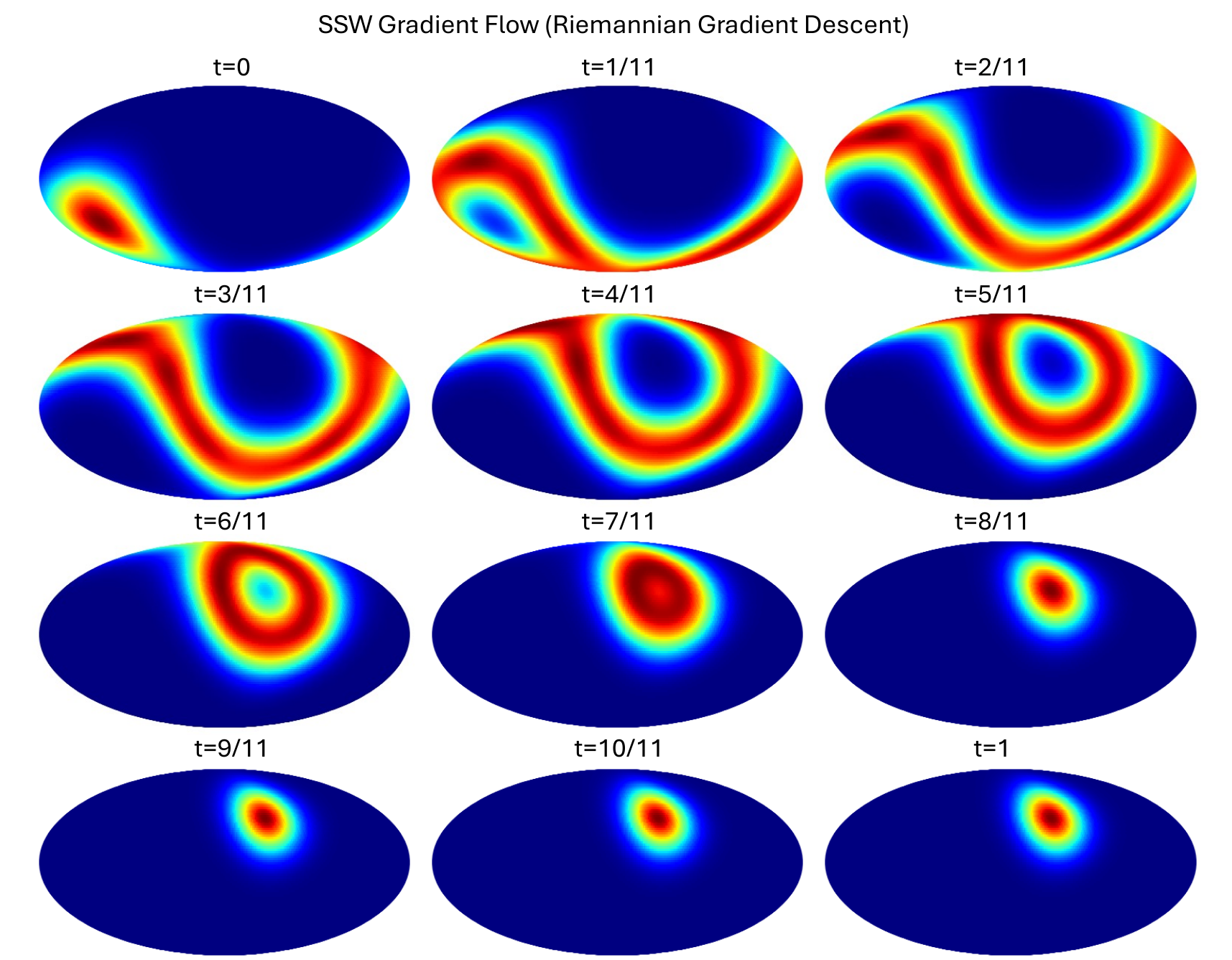}
    \caption{SSW gradient flow using Riemannian gradient descent.}
    \label{fig:uni_ssw_rem}
\end{figure}
\begin{figure}[H]
    \centering
    \includegraphics[width=0.8\linewidth]{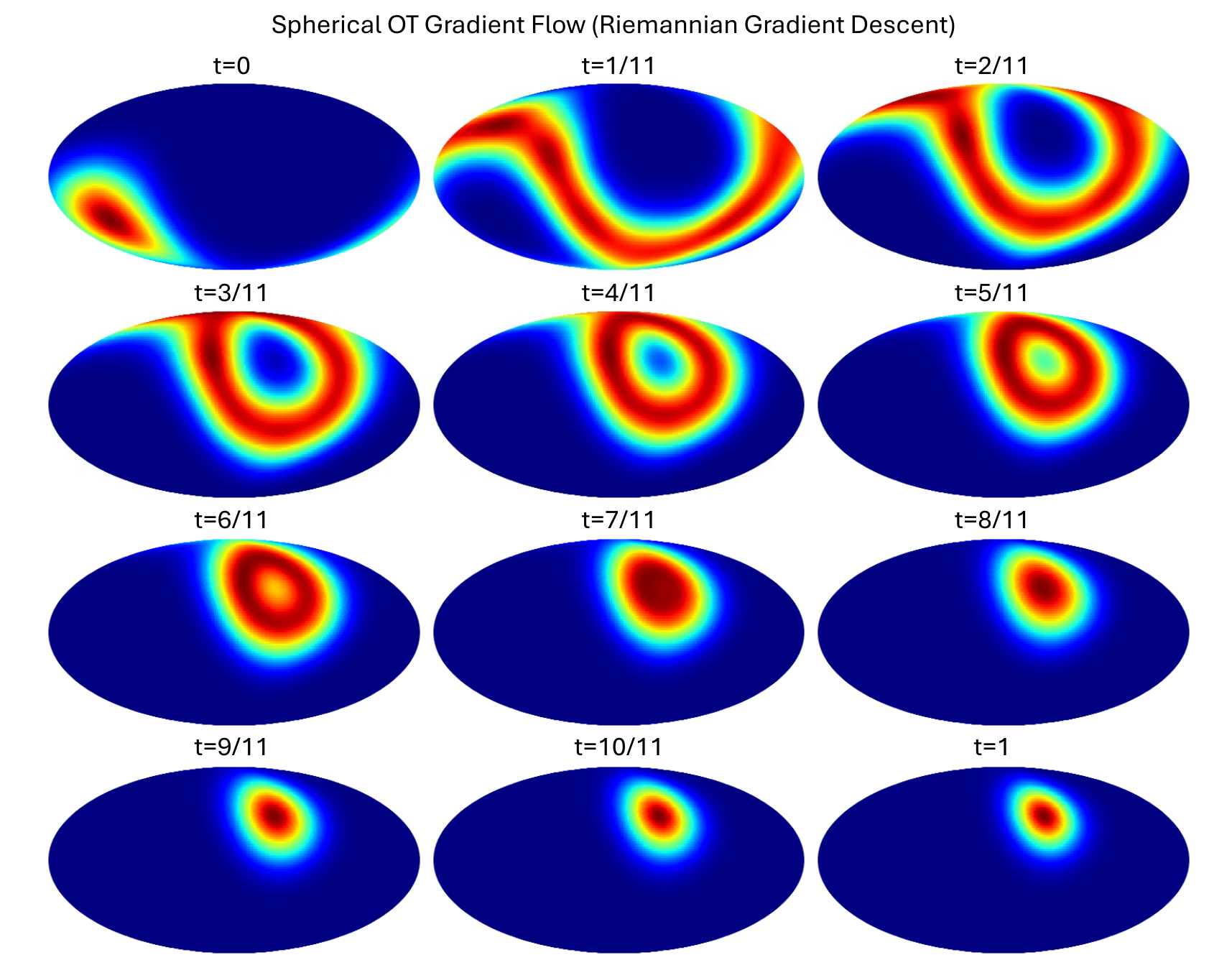}
    \caption{Spherical OT gradient flow using Riemannian gradient descent.}    
    \label{fig:uni_ot_rem}
\end{figure}
\begin{figure}[H]
    \centering
    \vspace{-.5in}
    \includegraphics[width=0.8\linewidth]{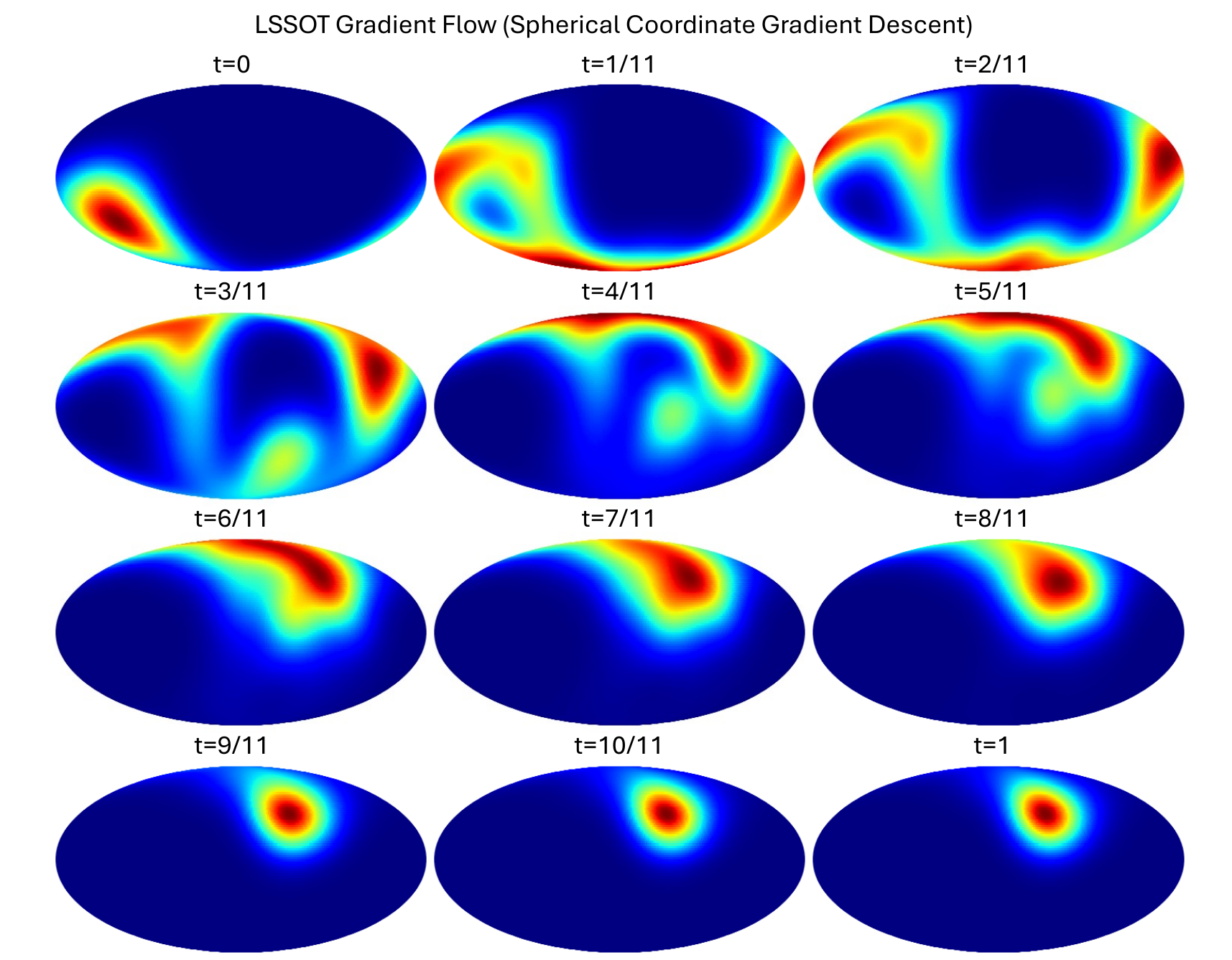}
    \caption{LSSOT gradient flow using spherical coordinate gradient descent.}    
    \label{fig:uni_lssot_coor}
\end{figure}
\begin{figure}[H]
    \centering
    \includegraphics[width=0.8\linewidth]{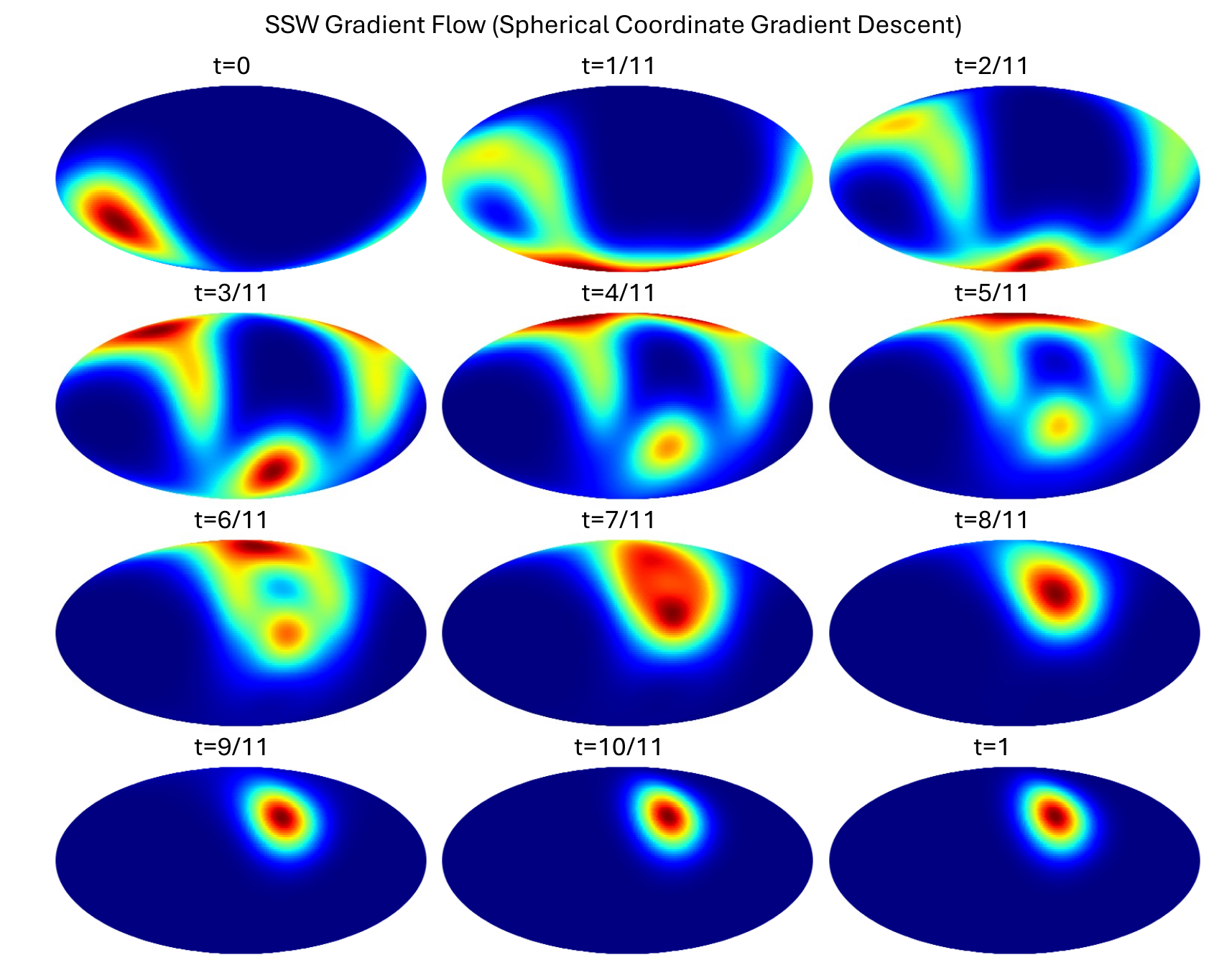}
    \caption{SSW gradient flow using spherical coordinate gradient descent.}
    \label{fig:uni_ssw_coor}
\end{figure}
\begin{figure}[H]
    \centering
    \vspace{-.5in}
    \includegraphics[width=0.8\linewidth]
    {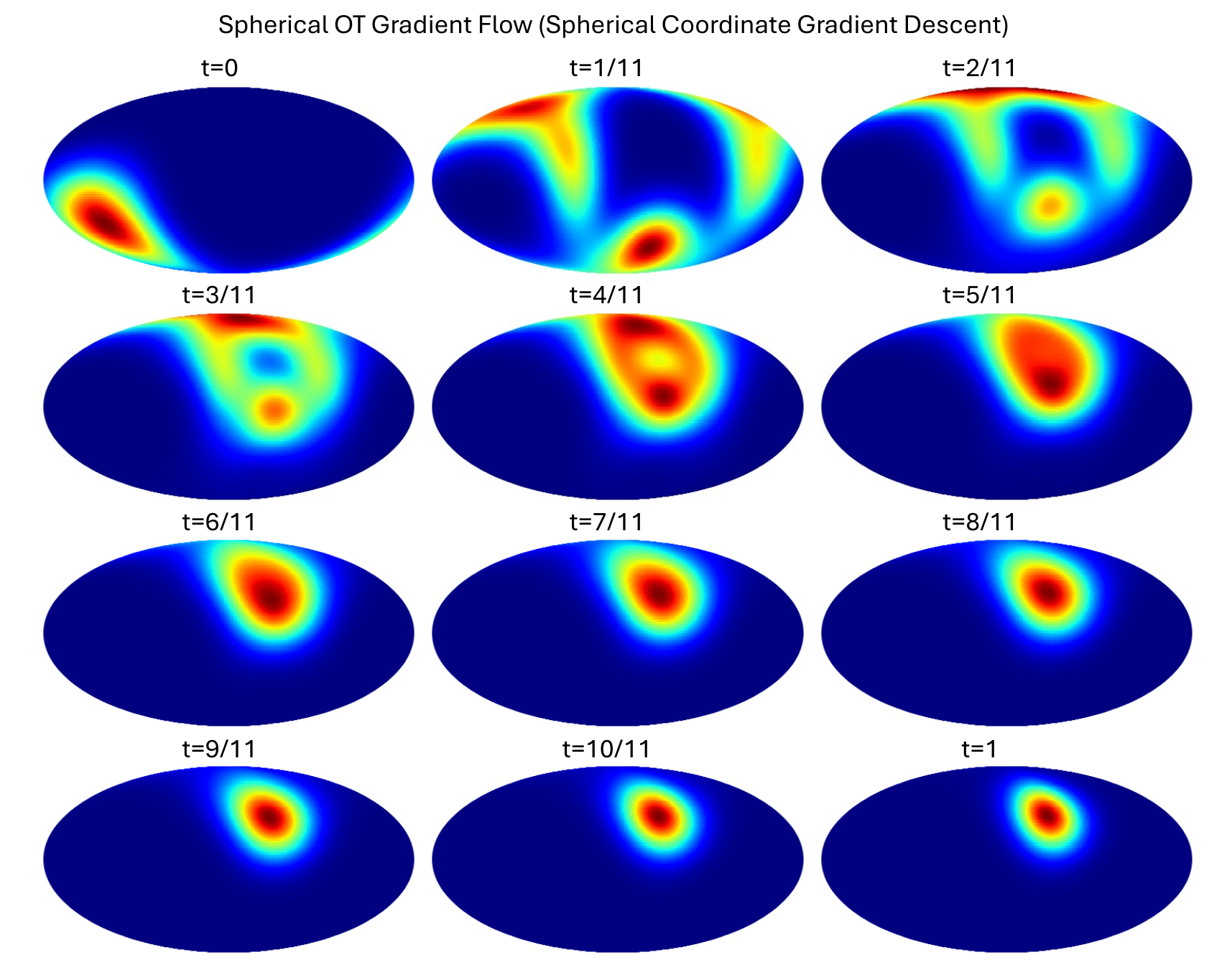}
    \caption{Spherical OT gradient flow using spherical coordinate gradient descent.}
    \label{fig:uni_ot_coor}
\end{figure}
 \vspace{-1cm}
\begin{figure}[H]
    \centering
    \includegraphics[width=0.85\linewidth]{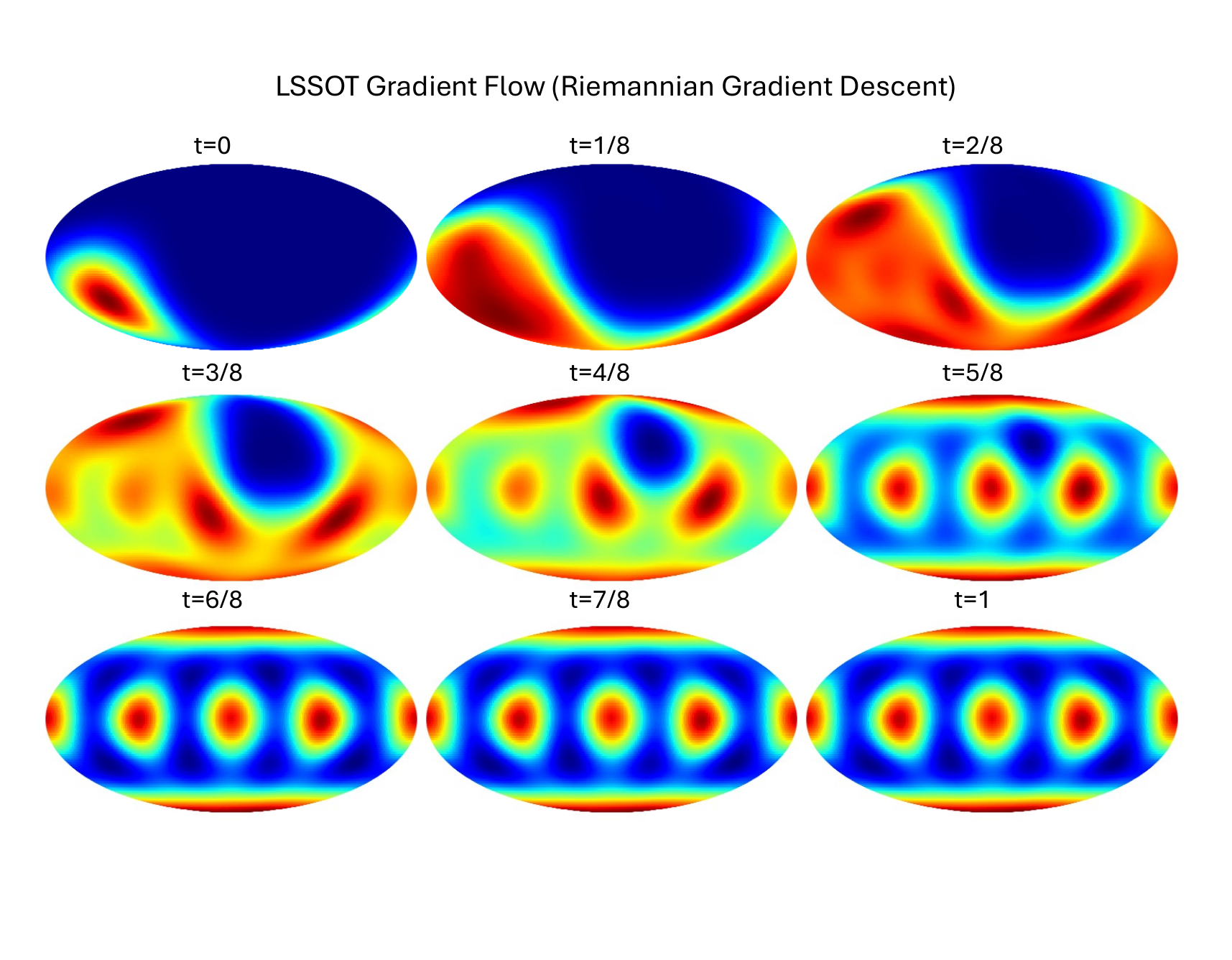}   
    \vspace{-.2in}
    \caption{LSSOT gradient flow using Riemannian gradient descent.}
    \label{fig:mixture_lssot_rem}
\end{figure}
 \vspace{-1cm}
\begin{figure}[H]
    \centering
    \vspace{-.5in}
    \includegraphics[width=0.85\linewidth]{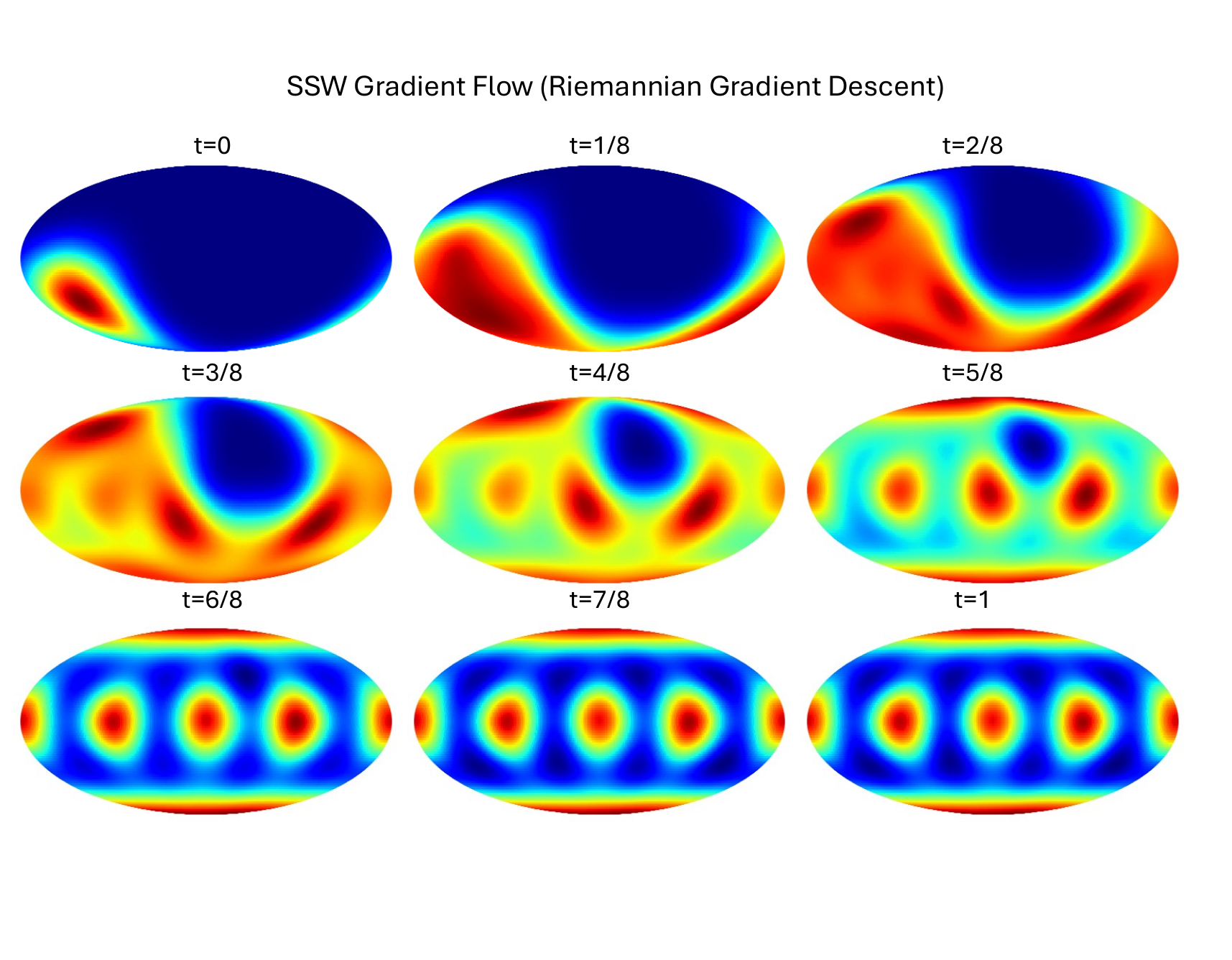}
    \vspace{-.2in}
    \caption{SSW gradient flow using Riemannian gradient descent.}
    \label{fig:mixture_ssw_rem}
\end{figure}
\begin{figure}[H]
    \centering     
    \includegraphics[width=0.85\linewidth]{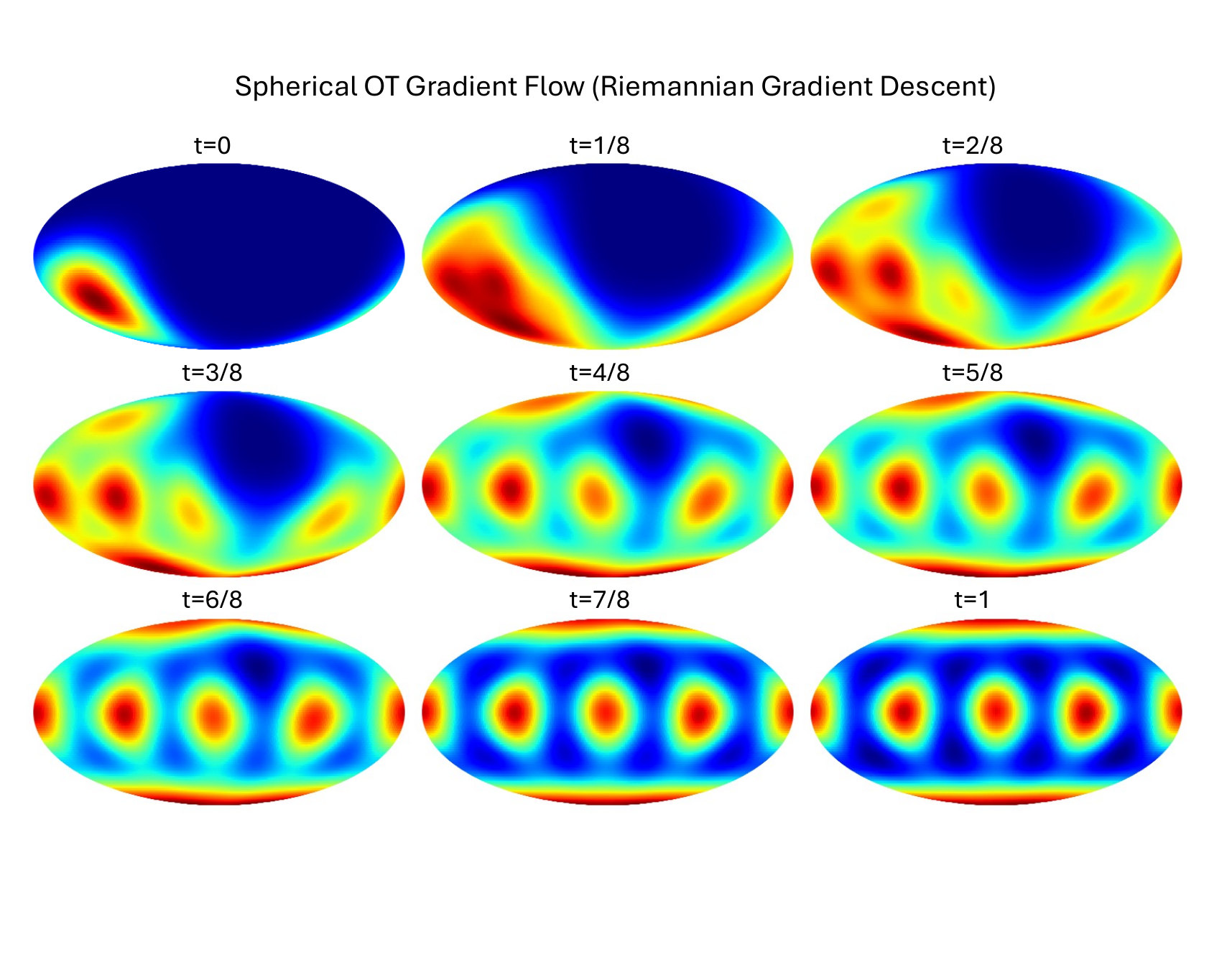}
    \vspace{-.2in}
    \caption{Spherical OT gradient flow using Riemannian gradient descent.}
    \label{fig:mixture_ot_rem}
\end{figure}
 \vspace{-1cm}
\begin{figure}[H]
    \centering
    \vspace{-.5in}
    \includegraphics[width=0.85\linewidth]{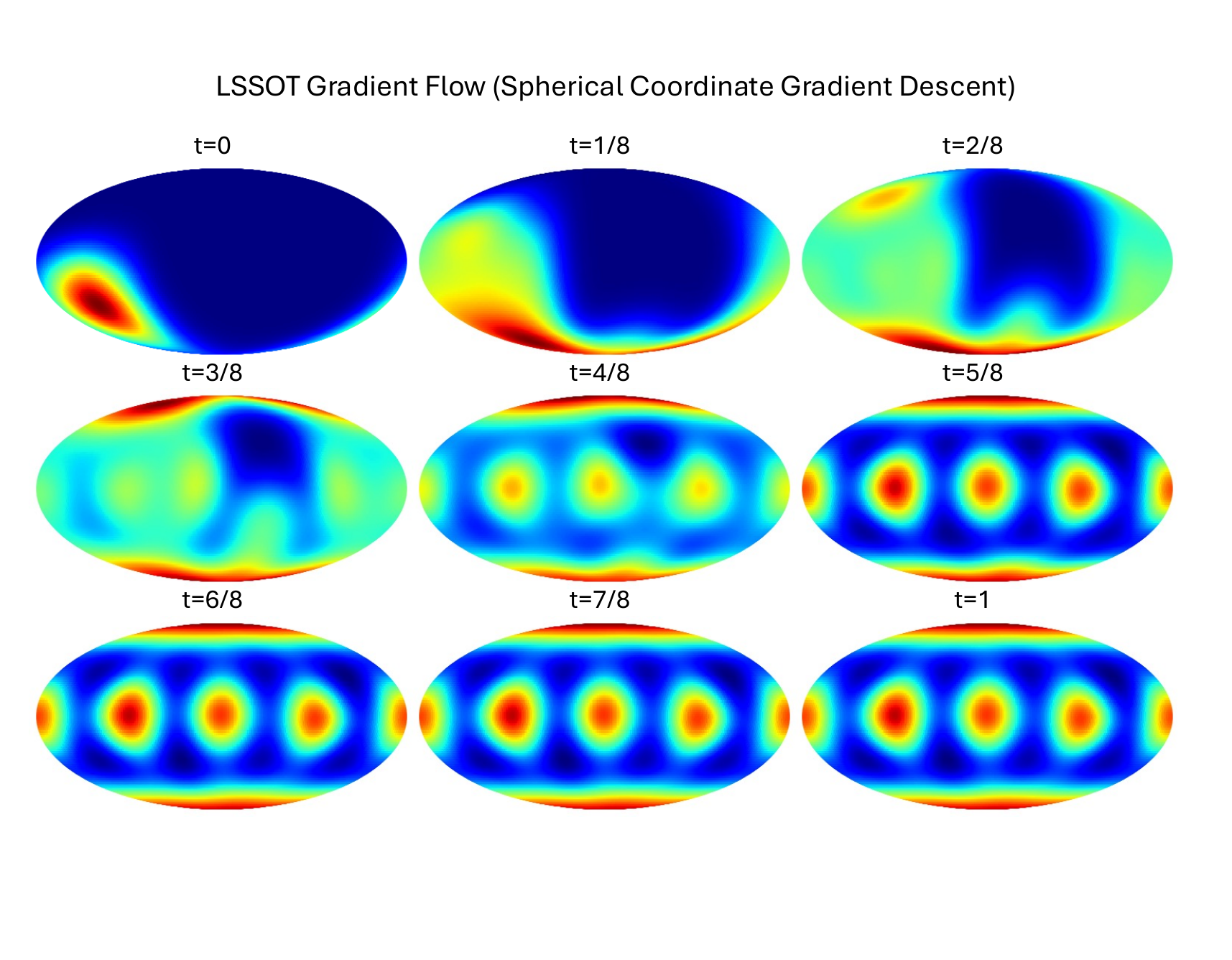}
    \vspace{-.2in}
    \caption{LSSOT gradient flow using spherical coordinate gradient descent.}

    \label{fig:mixture_lssot_coor}
\end{figure}
 \vspace{-1cm}
\begin{figure}[H]
    \centering    
    \includegraphics[width=0.85\linewidth]{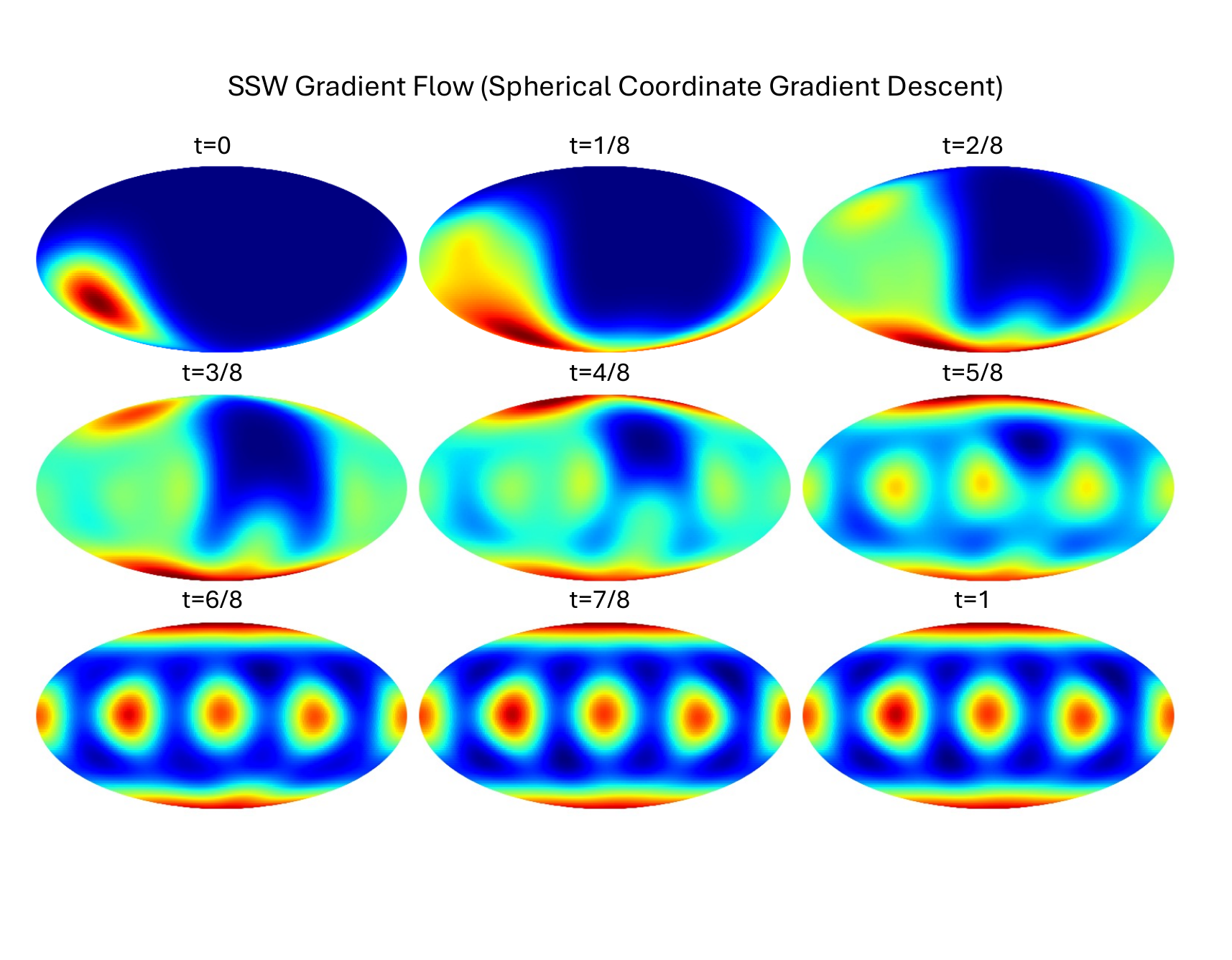}
    \vspace{-.3in}
    \caption{SSW gradient flow using spherical coordinate gradient descent.}
    \label{fig:mixture_ssw_coor}
\end{figure}
\begin{figure}[H]
    \centering
    \vspace{-.3in}
    \includegraphics[width=0.85\linewidth]{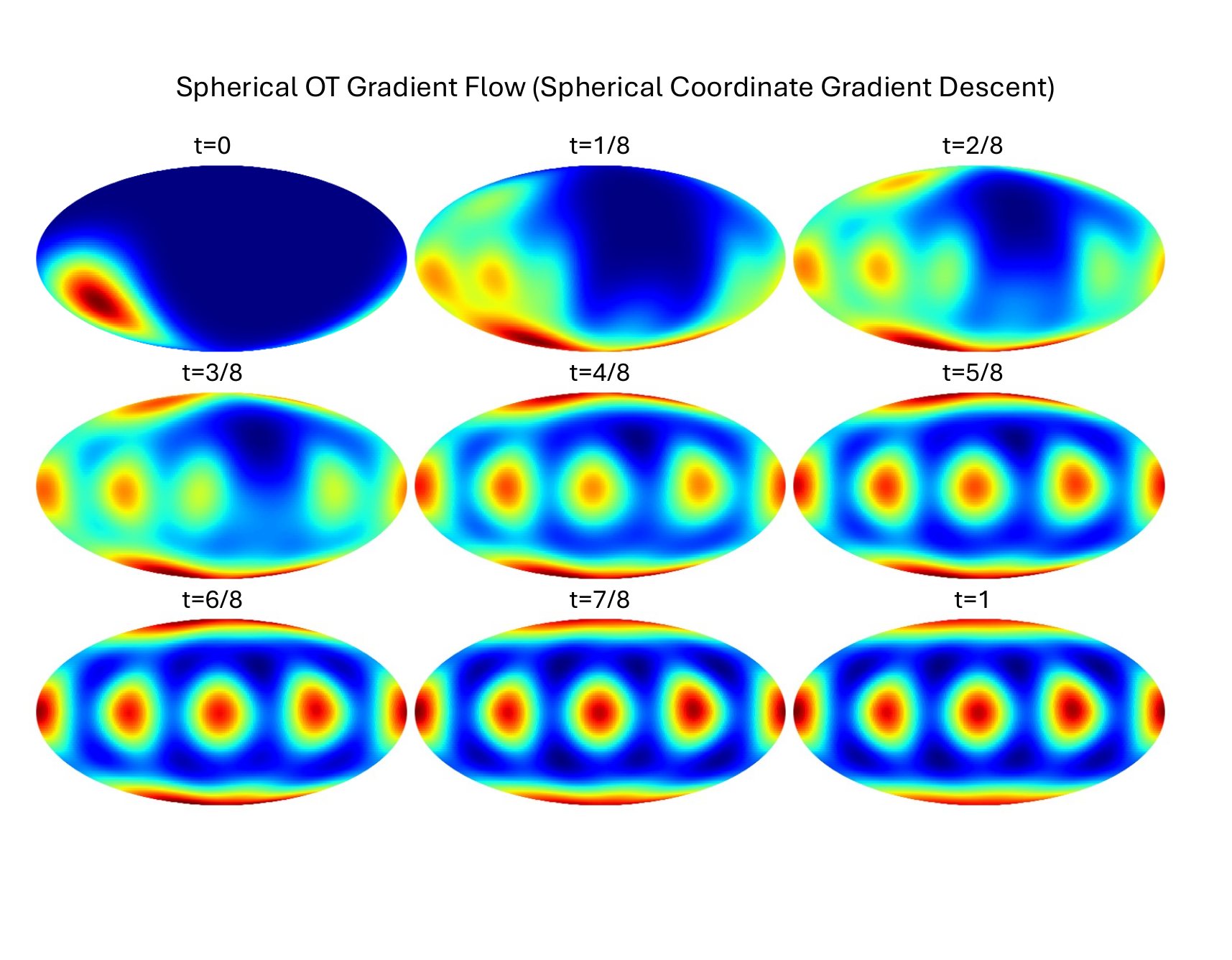}
    \vspace{-.2in}
    \caption{Spherical OT gradient flow using spherical coordinate gradient descent.}
    \label{fig:mixture_ot_coor}
\end{figure}

\end{document}